\documentclass[twoside,11pt]{article}

\usepackage[abbrvbib,preprint]{jmlr2e}

\usepackage{amsmath}
\usepackage{amsfonts}
\usepackage{bm}
\usepackage{bbm}
\usepackage[ruled,vlined]{algorithm2e}
\usepackage{float}

\usepackage{microtype}
\usepackage{subcaption}
\usepackage{booktabs}
\usepackage{multirow}

\usepackage[dvipsnames]{xcolor}

\usepackage{comment}

\DeclareMathOperator{\diag}{\textrm{\normalfont{diag}}}

\DeclareMathOperator{\tvec}{\textrm{\normalfont{vec}}}

\DeclareMathOperator{\tr}{\textrm{\normalfont{tr}}}
\DeclareMathOperator*{\argmin}{argmin}

\newcommand{\bThe}{\bm{\Theta}}
\newcommand{\bPsi}{\bm{\Psi}}
\newcommand{\bOmega}{\bm{\Omega}}
\newcommand{\bLambda}{\bm{\Lambda}}
\newcommand{\bXi}{\bm{\Xi}}
\newcommand{\bE}{\bm{E}}
\newcommand{\bW}{\bm{W}}
\newcommand{\bG}{\bm{G}}
\newcommand{\bH}{\bm{H}}
\newcommand{\bP}{\bm{P}}
\newcommand{\bI}{\bm{I}}

\newcommand{\bY}{\bm{Y}}

\newcommand{\Real}{\mathbb{R}}
\newcommand{\One}{\mathbbm{1}}

\allowdisplaybreaks

\ShortHeadings{EiGLasso}{Yoon and Kim}
\firstpageno{1}

\begin{document}

\title{EiGLasso for Scalable Sparse Kronecker-Sum Inverse Covariance Estimation}

\author{\name Jun Ho Yoon \email junhoy@cs.cmu.edu
       \AND
       \name Seyoung Kim \email sssykim@cs.cmu.edu\\
       \addr Computational Biology Department\\
       School of Computer Science\\
       Carnegie Mellon University\\
       Pittsburgh, PA 15213, USA}

\editor{}

\maketitle

\begin{abstract}
In many real-world problems, complex dependencies are present both among samples and among features.
The Kronecker sum or the Cartesian product of two graphs, each modeling dependencies across features
and across samples, has been used as an inverse covariance matrix for a matrix-variate Gaussian
distribution, as an alternative to Kronecker-product inverse covariance matrix,
due to its more intuitive sparse structure. However, the existing methods
for sparse Kronecker-sum inverse covariance estimation are limited in that they do not scale 
to more than a few hundred features and samples and that the unidentifiable parameters pose 
challenges in estimation. In this paper, we introduce EiGLasso, a highly scalable method
for sparse Kronecker-sum inverse covariance estimation, based on Newton's method combined with
eigendecomposition of the two graphs for exploiting the structure of Kronecker sum. 
EiGLasso further reduces computation time by approximating the Hessian based on the eigendecomposition
of the sample and feature graphs. EiGLasso achieves quadratic convergence 
with the exact Hessian and linear convergence with the approximate Hessian. 
We describe a simple new approach to estimating the unidentifiable parameters that generalizes
the existing methods. On simulated and real-world data, we demonstrate that EiGLasso 
achieves two to three orders-of-magnitude speed-up compared to the existing methods.
\end{abstract}


\section{Introduction}

In many real-world problems for statistical analysis, complex dependencies are found among samples
as well as among features.
For example, in tumor gene-expression data collected from cancer patients,
gene-expression levels are correlated both among genes in the same pathways and 
among patients with the same or similar cancer subtypes~\citep{breastcancersubtype}.
Other examples include multivariate time-series data such as temporally correlated stock prices 
for related
companies~\citep{stocktogether} or a sequence of images in video data~\citep{videopixel}.
Gaussian graphical models with $L_1$-regularization have been widely used to learn a 
sparse inverse covariance matrix that corresponds to a graph over 
features~\citep{glasso,quic,bigquic}.
However, they were limited in that samples were assumed to be independent. As an alternative,
a multivariate Gaussian distribution has been generalized to a matrix-variate Gaussian distribution,
where the sample and feature dependencies were modeled with two separate 
graphs.

Kronecker-product or Kronecker-sum operators have been used
to combine two matrices, each representing a graph over samples and a graph over features,
to form an inverse covariance matrix for matrix-variate Gaussian distribution.
The Kronecker product of two graphs leads to a hard-to-interpret dense 
graph and non-convex log-likelihood~\citep{lengandtang2012,YIN2012,Tsiligkaridis13,zhou2014}.
However, its bi-convexity provided a fast flip-flop optimization method.
In contrast, the Kronecker sum has the advantage of producing a sparse graph as the Cartesian
product of the two graphs and having a convex log-likelihood, but the existing optimization methods
do not scale to large datasets~\citep{biglasso, TeraLasso}. BiGLasso 
used GLasso 
as a subroutine to estimate one graph while fixing the other graph in each iteration~\citep{biglasso}.
TeraLasso significantly improved the scalability of BiGLasso with the gradient
descent method~\citep{TeraLasso}, but still did not scale to problems with more than a few hundred
samples and features.

Another main challenge with the Kronecker sum comes from
the unidentifiable parameters in the diagonals of the matrices for feature and sample graphs. 
BiGLasso did not estimate these unidentifiable parameters. 
TeraLasso employed a reparameterization to make the
parameters identifiable and projected the estimates to the reparameterized space in each iteration.
In regression, the Kronecker-sum inverse covariance has been used to model errors
in covariates, but the trace of one of the two graphs was assumed to be known~\citep{RudelsonZhou2017,Park2017,Zhang2020}.  
With the Kronecker product, the parameters are unidentifiable as well, but 
they can be made identifiable with a simple method after the optimization is complete~\citep{YIN2012}.

In this paper, we introduce eigen graphical Lasso (EiGLasso) that addresses these limiations of the
existing methods for estimating the Kronecker-sum inverse covariance matrix.
Our contribution is two fold. First,
we develop an efficient optimization method based on Newton's method
that empirically leads two to three orders-of-magnitude speed-up compared to the existing methods.
Second, we describe a simple new approach to identifying and estimating the unidentifiable parameters
that generalizes the strategy used in TeraLasso~\citep{TeraLasso}.

For an efficient optimization, we adopt the framework used in QUIC~\citep{quic,bigquic},
the state-of-the-art method for estimating sparse Gaussian graphical models,
and extend it to address additional challenges
that arise when combining two graphs with the Kronecker-sum operator.
Since the gradient and Hessian matrices are far larger with an inflated structure than in QUIC, 
to reduce the computation time,
we leverage the eigendecomposition of the parameters. 
This eigendecomposition leads to a strategy
for approximating the Hessian to further
improve the scalability. Extending the theoretical results on the convergence of QUIC, we show
EiGLasso achieves quadratic convergence with the exact Hessian and linear convergence with 
the approximate Hessian.
In our experiments, we show that 
EiGLasso with the approximate Hessian does not require significantly more iterations
than EiGLasso with the exact Hessian and achieves
two to three orders-of-magnitude speed-up compared to the state-of-the-art method, 
TeraLasso~\citep{TeraLasso}.

In addition, we introduce a new approach to determining the unidentifiable parameters
in Kronecker-sum inverse covariance estimation.
Our work is the first to point out that the unidentifiable parameters
are uniquely determined given the ratio of the traces
of the two parameter matrices for graphs. Building upon this observation, we show that all the quantities involved in 
Newton's method, including the gradient, Hessian, the descent directions, and the step size,
are uniquely determined regardless of the trace ratio. This allows us to perform the optimization
in the original space of the parameters without being concerned about the unidentifiability,
rather than in the reparameterized space as in TeraLasso.
Furthermore, we show that the parameters can be identified once with the given trace ratio after EiGLasso converges,
analogous to the simple approach used in the Kronecker-product inverse covariance 
estimation~\citep{YIN2012},
rather than in every iteration as in TeraLasso. 

Our preliminary work on EiGLasso appeared in \citet{yoon2020}. While
in our earlier work, we presented a flip-flop optimization method as in the Kronecker-product
inverse covariance estimation, in this work, we present EiGLasso with the exact Hessian,
along with EiGLasso with the approximate Hessian directly derived from the exact Hessian.
We investigate the properties of these Hessian matrices and analyze the convergence behavior
of EiGLasso. We perform a more extensive experimental evaluation of our method.

The rest of the paper is organized as follows. In Section \ref{sec:related}, we review
the previous work on statistical methods with the Kronecker-sum operator. 
We introduce our EiGLasso optimization method in Section \ref{sec:eiglasso}, study
the convergence behavior in Section \ref{sec:convg}, present experimental results
in Section \ref{sec:exp}, and conlude with future work in Section \ref{sec:conclusion}.

\section{Related Work}\label{sec:related}

A Gaussian distribution of a matrix random variable 
$\bY\in\Real^{q\times p}$ for $q$ samples and $p$ features
with a Kronecker-sum inverse covariance matrix~\citep{biglasso} is given as
\begin{equation}\label{eqn:ksmatrixnormal}
    \tvec(\bY) \sim \mathcal{N}\left(\tvec(\bm{M}), \bOmega^{-1}\right),
\end{equation}
where $\bm{M}\in\Real^{q\times p}$ is the mean and 
$\tvec(\cdot)$ is an operator that stacks the columns of a matrix into a vector. 
The $pq \times pq$ inverse covariance matrix $\bOmega$ in Eq. \eqref{eqn:ksmatrixnormal} 
is defined as the Kronecker-sum of two graphs, each given as a $p\times p$ matrix $\bThe$ 
and a $q \times q$ matrix $\bPsi$, modeling dependencies across features and across samples, respectively:
\begin{equation*}
    \bOmega = \bThe\oplus\bPsi = \bThe\otimes\bm{I}_q + \bm{I}_p\otimes\bPsi,
\end{equation*}
where $\otimes$ is the Kronecker-product operator and $\bI_a$ is an $a\times a$ identity matrix.
A non-zero value in the $(i,j)$th element $[\bThe]_{ij}$, $[\bPsi]_{ij}$, and $[\bOmega]_{ij}$
implies an edge between the $i$th and $j$th nodes in the corresponding graph.
The two graphs $\bThe$ and $\bPsi$ are constrained to form
a positive-definite Kronecker-sum space as follows:
\begin{equation}\label{def:ksspace}
    \mathbb{KS}^{p,q}=
	\{ (\bThe,\bPsi)| \bThe\in\mathbb{S}^p, \bPsi\in\mathbb{S}^q, \bOmega = \bThe\oplus\bPsi \in \mathbb{S}^{pq}_{++}\}.
\end{equation}
Then, $\bOmega$ models a graph over $pq$ nodes, where each node is associated with an observation
for the given sample and feature pair.
For the rest of this paper, we assume zero mean, i.e., $\tvec(\bm{M})=\bm{0}$ and focus on the 
inverse covariance structure.

Compared to the Kronecker product $\bThe\otimes \bPsi$, 
one main advantage of the Kronecker sum 
is that as the Cartesian product of the two graphs $\bThe$ and $\bPsi$, 
it leads to a more intuitively appealing sparse structure in the graph $\bOmega$ 
(Figure \ref{fig:collapseW}(a)). 
In $\bOmega$ with the Kronecker sum,
the sample graph $\bPsi$ is used only within the same feature
(diagonal blocks in $\bOmega$ in Figure \ref{fig:collapseW}(a)) 
and the feature graph $\bThe$ is used only within the same sample
(diagonals of off-diagonal blocks in $\bOmega$ in Figure \ref{fig:collapseW}(a)), 
whereas the Kronecker-product leads to a dense graph, 
where the sample graph $\bPsi$ and the feature graph $\bThe$ are used 
even across features and across sample samples.
Because of the sparse structure, 
the Kronecker-sum operator has been adopted in various other statistical methods to represent a 
sparse association between two sets of variables, such as
dependencies between genes and mutations in colorectal cancer \citep{Haupt2020} 
and between microRNAs and diseases \citep{Li2018}. 
It has been embedded in neural networks to model both row and 
column dependencies in a matrix~\citep{Zhang2018tensor,Gao2020} and 
in inputs for more general functions~\citep{Benzi2017}.


\begin{figure}[t!]
    \centering
    \begin{tabular}{cc}
        \begin{subfigure}{0.482\linewidth}
            \includegraphics[width=\linewidth]{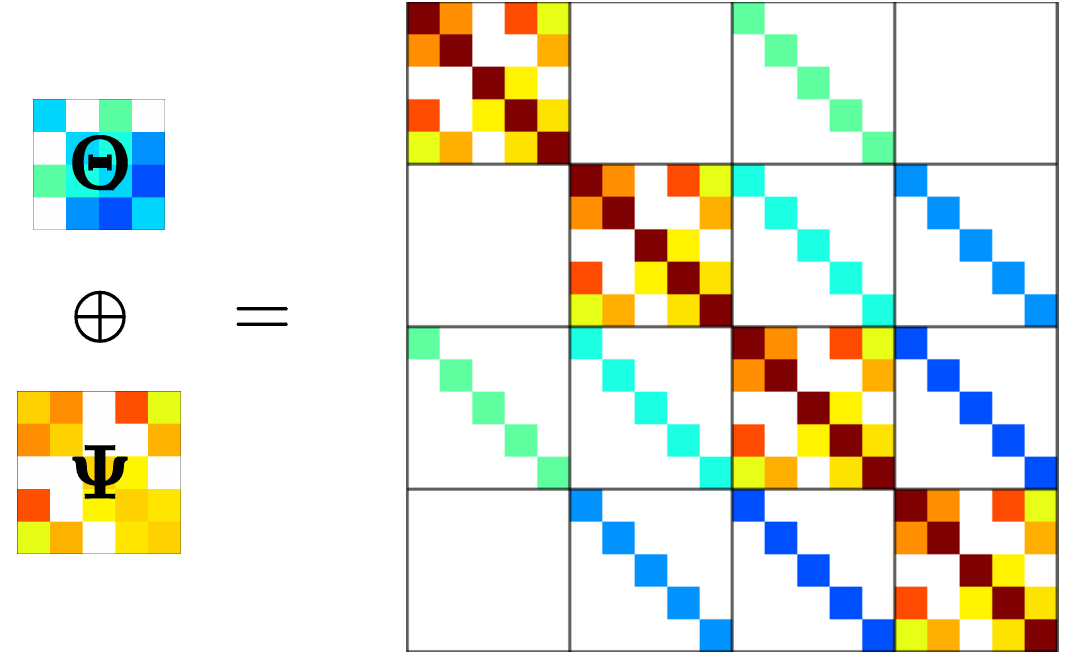}
            \caption*{$\qquad\qquad\qquad$(a)}
        \end{subfigure} &
        \begin{subfigure}{0.518\linewidth}
            \includegraphics[width=\linewidth]{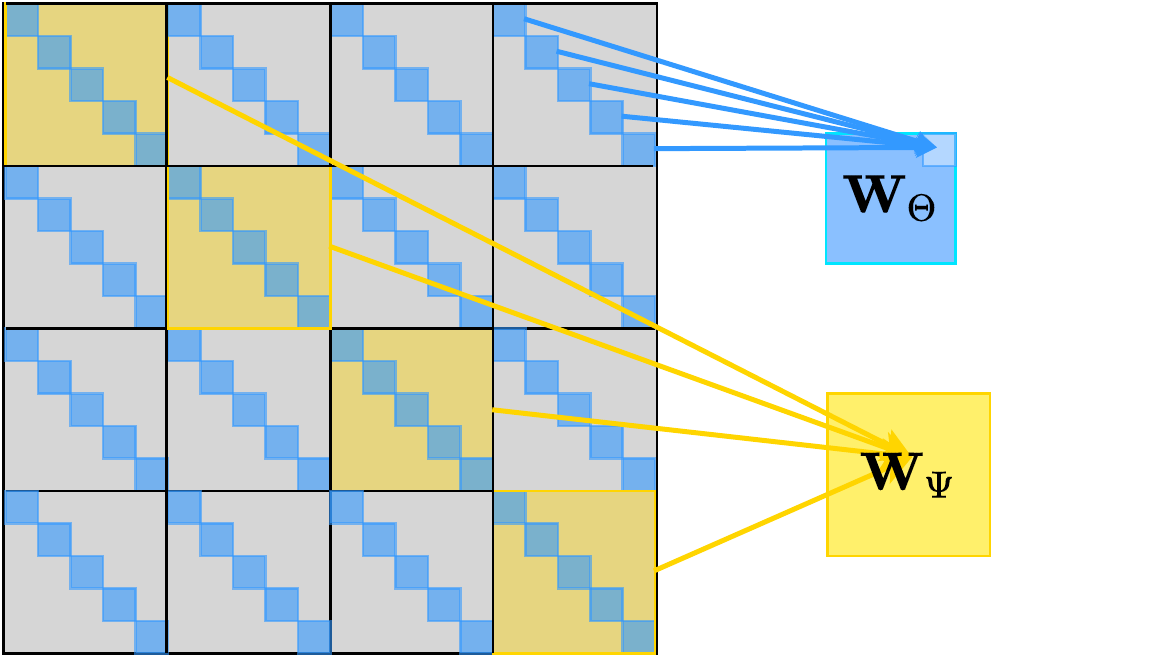}
            \caption*{(b)$\quad\qquad\qquad\qquad$}
        \end{subfigure}
    \end{tabular}
    \caption{Illustration of the Kronecker-sum inverse covariance matrix and gradient computation in EiGLasso.
    (a) $\bThe$ and $\bPsi$ (left) and their Kronecker sum $\bOmega=\bThe\oplus\bPsi$ (right).
    $\bThe$ appears in the diagonals of off-diagonal blocks of $\bOmega$, and $\bPsi$ appears 
    in the diagonal blocks of $\bOmega$. 
    (b) $\bW=\bOmega^{-1}$ (left) and $\bW_{\bThe}$ and $\bW_{\bPsi}$ in the gradient (right). 
    Computing $\bW_{\bThe}$ and $\bW_{\bPsi}$ amounts to
    collapsing $\bW$ based on the Kronecker-sum structure in Panel (a).
    The $(i,j)$th element of $\bW_{\bThe}$ is obtained by summing over the diagonals of the
    $(i,j)$th block (blue arrows), and
    $\bW_{\bPsi}$ is obtained by summing over the blocks on the diagonals of $\bW$ (yellow arrows).
    }
    \label{fig:collapseW}
\end{figure}

Given data $\{\bY^1,\,\ldots,\,\bY^n\}$, where $\bY^i \in \mathbb{R}^{q\times p}$ 
for $i=1,\,\ldots,\,n$ for $n$ observations, 
BiGLasso~\citep{biglasso} and TeraLasso~\citep{TeraLasso} obtained
a sparse estimate of $\bThe$ and $\bPsi$ 
by minimizing the $L_1$-regularized negative log-likelihood of data:
\begin{equation} \label{eq:obj}
    \argmin_{\bThe,\, \bPsi}  f(\bThe,\bPsi) \quad \textrm{subject to} \quad\bThe\oplus\bPsi\succ 0.
\end{equation}  
The objective $f(\bThe,\bPsi)$ is 
\begin{equation}
    f(\bThe,\bPsi)= g(\bThe,\bPsi) + h(\bThe,\bPsi), \nonumber
\end{equation}
with the smooth log-likelihood function 
$g(\bThe,\bPsi) = q\tr(\bm{S}\bThe) + p\tr(\bm{T}\bPsi) - \log|\bThe\oplus\bPsi|$ and
non-smooth penalty $h(\bThe,\bPsi) = q \gamma_{\bThe} \| \bThe \|_{1,\text{off}} + p \gamma_{\bPsi} \| \bPsi \|_{1,\text{off}}$,
given the sample covariances $\bm{S} = \frac{1}{nq}\sum_{i=1}^n\bY^{iT}\bY^i$ and 
$\bm{T} = \frac{1}{np}\sum_{i=1}^n\bY^i\bY^{iT}$, the regularization parameters 
$\gamma_{\bThe}$ and $\gamma_{\bPsi}$, and 
$\|\cdot\|_{1,\text{off}}$ 
for the $L_1$-regularization of the off-diagonal elements of the matrix.

One of the main challenges of solving Eq. \eqref{eq:obj} 
arises from the unidentifiable diagonals of $\bThe$ and $\bPsi$ given $\bOmega$.
Although \citet{TeraLasso} showed that the objective in 
Eq. \eqref{eq:obj} has a unique global optimum  with respect to $\bOmega=\bThe\oplus\bPsi$, 
this does not imply the uniqueness of the pair $(\bThe,\bPsi)$ given $\bOmega$,
since the set in Eq. \eqref{def:ksspace} forms equivalence classes 
\begin{equation}\label{eqn:noniddiag}
    S[\bOmega] = \{ (\bThe,\bPsi)\in \mathbb{KS}^{p,q}| 
	\left(\bThe-c\bm{I}_p\right)\oplus\left(\bPsi+c\bm{I}_q\right)=\bOmega,
		c\in \mathbb{R}\}.
\end{equation}
Thus, $\bThe$ and $\bPsi$ whose diagonals are modified by a constant $c$ and $-c$ 
lead to the same $\bOmega$. 
While BiGLasso did not estimate the diagonals of $\bThe$ and $\bPsi$, 
TeraLasso formed an identifiable re-parameterization
 $\bOmega = \bar{\bThe} \oplus \bar{\bPsi} + \tau \bm{I}_{pq}$ with
$\tau = \frac{\tr(\bThe \oplus \bPsi)}{pq}$, where
$\bar{\bThe}$ and $\bar{\bPsi}$ are forced to have zero traces as follows:
\begin{align}\label{eq:reparam}
\bar{\bThe}=\bThe -\frac{\tr(\bThe)}{p}\bI_p\qquad\text{and}\qquad\bar{\bPsi}=\bPsi -\frac{\tr(\bPsi)}{q}\bI_q,
\end{align}
Then, in each iteration of a gradient descent method,
TeraLasso projected the gradient to this re-parameterized space and distributed 
$\tau$ to $\bThe$ and $\bPsi$ as $\bar{\bThe} \gets \bar{\bThe} + \frac{1}{2}\tau\bI$ and 
$\bar{\bPsi} \gets \bar{\bPsi} + \frac{1}{2}\tau\bI$. 
For efficient computation and projection of the gradient, TeraLasso
employed an eigendecomposition of $\bThe$ and $\bPsi$:
$\bThe=\bm{Q}_{\bThe}\bm{\Lambda}_{\bThe}\bm{Q}_{\bThe}^T$ with
the eigenvector matrix $\bm{Q}_{\bThe}$ 
and the diagonal eigenvalue matrix $\bm{\Lambda}_{\bThe}$, and similarly, 
$\bPsi = \bm{Q}_{\bPsi} \bm{\Lambda}_{\bPsi} \bm{Q}_{\bPsi}^T$.
TeraLasso exploited the property that the eigendecomposition of the Kronecker sum of 
$\bThe$ and $\bPsi$ is 
\begin{equation}\label{eq:omegaeigdecomp}
    \bThe \oplus \bPsi = (\bm{Q}_{\bThe}\otimes\bm{Q}_{\bPsi})(\bm{\Lambda}_{\bThe}\oplus\bm{\Lambda}_{\bPsi})(\bm{Q}_{\bThe}\otimes\bm{Q}_{\bPsi})^T.
\end{equation}
Then, the inverse of $\bThe \oplus \bPsi$ can be obtained efficiently 
by inverting the diagonal eigenvalue matrix $\bLambda_{\bThe} \oplus \bLambda_{\bPsi}$:
\begin{equation}\label{eq:Weigdecomp}
    \bW = (\bThe \oplus \bPsi)^{-1} = (\bm{Q}_{\bThe}\otimes\bm{Q}_{\bPsi})(\bm{\Lambda}_{\bThe}\oplus\bm{\Lambda}_{\bPsi})^{-1}(\bm{Q}_{\bThe}\otimes\bm{Q}_{\bPsi})^T.
\end{equation}
While TeraLasso is significantly faster than BiGLasso, its scalability is limited to graphs with
a few hundred nodes. 

\section{EiGLasso}\label{sec:eiglasso}

We introduce EiGLasso, an efficient method for estimating a sparse Kronecker-sum inverse covariance matrix. 
We begin by describing a simple new scheme for identifying the unidentifiable parameters in Section \ref{sec:tr}. 
We introduce EiGLasso, Newton's method with the exact Hessian in Section \ref{sec:exactH} 
and its modification with approximate Hessian in Section \ref{sec:approxH}. 
In Section \ref{sec:identify}, we describe a simple strategy for identifying the 
diagonal parameters during EiGLasso estimation.
In Section \ref{sec:block}, we discuss additional strategies for improving the computational efficiency.

\subsection{Identifying Parameters with Trace Ratio}\label{sec:tr}

Theorem \ref{thm:uniquepair} below provides a simple approach to identifying the diagonal elements 
of $\bThe$ and $\bPsi$, given $\bOmega$ and the ratio of the traces of $\bThe$ and $\bPsi$.

\begin{theorem}\label{thm:uniquepair}
    Assume that the trace ratio of $\bPsi$ to $\bThe$ is given as $\rho = \frac{\tr(\bPsi)}{\tr(\bThe)}$.        
    Then, a Kronecker-sum matrix $\bOmega=\bThe\oplus\bPsi$ can be mapped to a unique pair of symmetric matrices $(\bThe,\bPsi)$ with the diagonals of $\bThe$ and $\bPsi$ identified as
    \begin{equation}\label{eq:diagmapping}
    \begin{aligned}
        \diag(\bThe) &= \frac{1}{q}\left[\diag\left( \sum_{i=1}^q(\bI_p\otimes\bm{e}_{q,i})^T\bOmega(\bI_p\otimes\bm{e}_{q,i}) \right)  - \frac{\rho}{q + \rho p}\tr(\bOmega)\bI\right],\\
        \diag(\bPsi) &= \frac{1}{p}\left[\diag\left( \sum_{j=1}^p(\bm{e}_{p,j}\otimes\bI_q)^T\bOmega(\bm{e}_{p,j}\otimes\bI_q) \right)  - \frac{1}{q + \rho p}\tr(\bOmega)\bI\right],
    \end{aligned}
    \end{equation}
    where $\bm{e}_{q,i}$ is a $q\times 1$ one-hot vector with a single $1$ at the $i$th entry and $0$'s elsewhere.
\end{theorem}
\begin{proof}
    Given $\rho=\frac{\tr(\bPsi)}{\tr(\bThe)}$, $\bThe$ and $\bPsi$ are unique, since for any $c\in\mathbb{R}\setminus\{0\}$,
        $\frac{\tr(\bPsi)}{\tr(\bThe)} \neq \frac{\tr(\bPsi+c\bI_q)}{\tr(\bThe-c\bI_p)}=\frac{\tr(\bPsi) + cq}{\tr(\bThe) - cp}$. 
    From the definition of Kronecker sum, we notice
    \begin{align}\label{eq:omegaii}
        \sum_{i=(j-1)q + 1}^{jq} [\bOmega]_{ii} &= q[\bThe]_{jj} + \tr(\bPsi),
    \end{align}
and 
    \begin{equation}\label{eq:tromega}
           \tr(\bOmega) =  q\tr(\bThe) + p\tr(\bPsi).
    \end{equation}
    From the trace ratio and Eq. \eqref{eq:tromega}, we have $\tr(\bPsi) = \frac{\rho}{q + \rho p}\tr(\bOmega)$. 
    Plugging this into Eq. \eqref{eq:omegaii}, we identify $[\bThe]_{jj}$ as 
    \begin{align*}
        [\bThe]_{jj} &= \frac{1}{q}\left(\sum_{i=(j-1)q + 1}^{jq} [\bOmega]_{ii} - \frac{\rho}{q + \rho p}\tr(\bOmega)\right).
    \end{align*}
    The case for $\bPsi$ can be shown in a similar way.
\end{proof}

Theorem \ref{thm:uniquepair} suggests that it is possible to identify the diagonals 
of $\bThe$ and $\bPsi$ by solving Eq. \eqref{eq:obj} with the linear constraint 
$\frac{\tr(\bPsi)}{\tr(\bThe)}=\rho$:
\begin{align}
    \argmin_{\bThe,\bPsi}  f(\bThe,\bPsi)  
    \quad \textrm{subject to} \quad\bThe\oplus\bPsi\succ0, \;
\bm{R}^T
    \begin{bmatrix}
        \tvec(\bThe)\\
        \tvec(\bPsi)
    \end{bmatrix} = 0,    \label{eq:objrho}
\end{align} 
where the second constraint with 
$\bm{R} = \begin{bmatrix} 
-\rho\tvec(\bI_p)\\
\tvec(\bI_q)
\end{bmatrix}$
is equivalent to $\frac{\tr(\bPsi)}{\tr(\bThe)}=\rho$.
To handle this additional constraint explicitly, the substitution method or Newton's method 
for equality-constrained optimization problem could be used \citep{cvxbook}. 
Instead, in Sections \ref{sec:exactH} and \ref{sec:identify}, 
we show that in EiGLasso, because of the special problem structure, 
it is sufficient to solve Eq. \eqref{eq:obj} with Newton's method, ignoring the equality constraint, 
and to adjust the diagonals of $\bThe$ and $\bPsi$ once after convergence to satisfy the constraint.

\subsection{EiGLasso with Exact Hessian} \label{sec:exactH}

We introduce EiGLasso for an efficient estimation of sparse $\bThe$ and $\bPsi$ that form a Kronecker-sum inverse covariance matrix. 
We adopt the framework of QUIC~\citep{quic}, Newton's method for estimating
sparse Gaussian graphical models. In each iteration, QUIC found a Newton direction 
by minimizing the $L_1$-regularized second-order approximation
of the objective, and updated the parameters given this Newton direction and the step size
found by backtracking line search \citep{bertsekas1995,Tseng2009}. 
Using the same strategy, EiGLasso finds Newton directions, $D_{\bThe}$ for $\bThe$ and $D_{\bPsi}$ for $\bPsi$, 
by minimizing the following second-order approximation of the smooth part $g(\bThe,\bPsi)$ of 
the objective in Eq. \eqref{eq:objrho} with the $L_1$-regularization:
\begin{align}
    (D_{\bThe},D_{\bPsi}) = \argmin_{\Delta_{\bThe},\Delta_{\bPsi}}\; \hat{g}(\Delta_{\bThe},\Delta_{\bPsi}) + h(\bThe + \Delta_{\bThe}, \bPsi + \Delta_{\bPsi}) \label{eq:newton}\\
    \text{subject to } \bm{R}^T
    \begin{bmatrix}
        \tvec(\Delta_{\bThe})\\
        \tvec(\Delta_{\bPsi})
    \end{bmatrix} = 0,
    \nonumber
\end{align}
where
\begin{equation*}\label{eq:g2compact}
    \hat{g}(\Delta_{\bThe}, \Delta_{\bPsi}) = \tvec(\bG)^T
    \begin{bmatrix}
        \tvec(\Delta_{\bThe})\\
        \tvec(\Delta_{\bPsi})
    \end{bmatrix} + \frac{1}{2}
    \begin{bmatrix}
        \tvec(\Delta_{\bThe})\\
        \tvec(\Delta_{\bPsi})
    \end{bmatrix}^T
    \bH
     \begin{bmatrix}
        \tvec(\Delta_{\bThe})\\
        \tvec(\Delta_{\bPsi})
    \end{bmatrix}
\end{equation*}
with the gradient $\bG$ and Hessian $\bH$ 
\begin{gather}
    \tvec(\bG) =
        \begin{bmatrix}
            \tvec\left(\bG_{\bThe} \right)\\ 
            \tvec\left(\bG_{\bPsi} \right)
        \end{bmatrix}
        =
        \begin{bmatrix}
            \tvec\left(\nabla_{\bThe} g(\bThe,\bPsi)\right)\\ 
            \tvec\left(\nabla_{\bPsi} g(\bThe,\bPsi)\right)
        \end{bmatrix},
    \label{eq:compactg}
    \\
    \bH =
        \begin{bmatrix}
            \bH_{\bThe}  & \bH_{\bThe\bPsi}\\
            \bH_{\bThe\bPsi}^T  & \bH_{\bPsi}
        \end{bmatrix}
        =
        \begin{bmatrix}
            \nabla^2_{\bThe} g(\bThe,\bPsi) & \nabla_{\bThe}\nabla_{\bPsi} g(\bThe,\bPsi)\\
            \nabla_{\bThe}\nabla_{\bPsi} g(\bThe,\bPsi)^T & \nabla^2_{\bPsi} g(\bThe,\bPsi)
        \end{bmatrix}.
    \label{eq:compacth}
\end{gather}
As we will show in Section \ref{sec:identify}, the objective above can be optimized without being
concerned about the equality constraint. Thus, 
the coordinate descent optimization can be used to solve Eq. \eqref{eq:newton} as in QUIC. 
Given the descent directions $D_{\bThe}$ and $D_{\bPsi}$, 
we update the parameters as $\bThe \gets \bThe + \alpha D_{\bThe}$ and $\bPsi \gets \bPsi + \alpha D_{\bPsi}$, 
with the step size $\alpha$ found by the line-search method in Algorithm \ref{alg:linesearch}.

Two key challenges arise in a direct application of QUIC to Kronecker-sum inverse covariance estimation.
First, as we detail in the next section, the gradient $\bG$ and Hessian $\bH$ 
involves the computation of the $pq\times pq$ matrix $\bW=(\bThe \oplus \bPsi)^{-1}$ significantly larger
than a $p\times p$ matrix required for a single graph in QUIC. 
Second, the constraint on the trace ratio in Eq. \eqref{eq:newton} needs to be considered 
to handle the unidentifiability of the diagonals of $\bThe$ and $\bPsi$ 
during the optimization process. 
In the rest of Section \ref{sec:eiglasso}, we describe how we address these challenges by leveraging 
the eigen structure of the model, and provide the details of the EiGLasso optimization outlined 
in Algorithm \ref{alg:eiglasso}.

\begin{algorithm}[t!] 
\caption{Line Search}
\label{alg:linesearch}
    \DontPrintSemicolon
    \SetKwInOut{Input}{input}\SetKwInOut{Output}{output}
    \Input{$0<\sigma<0.5$, $0<\beta<1$}
    \Output{step size $\alpha$}
    
    Initialize $\alpha = 1$\;
    
    \For{$i=0,1,\dots$}{
        Check the following conditions:\;
        \begin{enumerate}
            \item $(\bThe + \alpha D_{\bThe})\oplus(\bPsi + \alpha D_{\bPsi})\succ 0$,\;
            \item $f(\bThe + \alpha D_{\bThe},\bPsi + \alpha D_{\bPsi})\le f(\bThe,\bPsi) + \alpha\sigma\delta$, where 
    $\delta = \tvec(D)^T\tvec(\bG) + h(\bThe+D_{\bThe},\bPsi+D_{\bPsi}) - h(\bThe,\bPsi)$.
\;
        \end{enumerate}
        If satisfied, break. If not, $\alpha \leftarrow \alpha\beta$.\;
    }
\end{algorithm}

\begin{algorithm}[t] 
\caption{EiGLasso}
\label{alg:eiglasso}
    \DontPrintSemicolon
    \SetKwInOut{Input}{input}\SetKwInOut{Output}{output}
    \Input{{\bf Data} Sample covariances $\bm{S}=\frac{1}{nq}\sum_{i=1}^n\bY^{iT}\bY^i$ and $\bm{T}=\frac{1}{np}\sum_{i=1}^n\bY^{i}\bY^{iT}$\\
    {\bf Hyperparameters for} regularization $\gamma_{\bThe}, \gamma_{\bPsi}$, Hessian approximation $K$, \\\quad line search $\sigma$, $0<\beta<1$, and trace ratio $\rho=\tr(\bPsi)/\tr(\bThe)$}
    
    \Output{Parameters $\bThe, \bPsi$}
    
    Initialize $\bThe \leftarrow \bI_p, \bPsi \leftarrow \bI_q$. 
    
    \For{$t=0,1,\dots$}{
        
        Eigendecompose $\bThe$ and $\bPsi$.\;
        
        Compute $\bG$ and $\bH$.\;
        
        Determine active sets $\mathcal{A}_{\bThe}$ and $\mathcal{A}_{\bPsi}$.\; 
        
        Compute $(D_{\bThe},D_{\bPsi})$ via coordinate descent over the active sets.\;
        
        Use the Armijo rule to compute a step-size $\alpha$ as in Algorithm \ref{alg:linesearch}.
        
	Update $\bThe \leftarrow \bThe + \alpha D_{\bThe}$, $\bPsi \leftarrow \bPsi + \alpha D_{\bPsi}$.
        
        Check convergence.\;
   }
   Adjust the diagonal elements of $\bThe$ and $\bPsi$ according to the trace ratio $\rho$.\;
\end{algorithm}

\subsubsection{Efficient Computation of Gradient and Hessian via Eigendecomposition}

In Lemma \ref{lem:gh} below, we provide the form of the gradient $\bG$ in Eq. \eqref{eq:compactg} and 
Hessian $\bH$ in Eq. \eqref{eq:compacth}. 
We show that $\bG$ can be represented in a significantly more compact form 
than what has been previously presented in \citet{biglasso} and \citet{TeraLasso}. 
We exploit this compact representation to further reduce the computation time  
via eigendecomposition of $\bThe$ and $\bPsi$ in Theorem \ref{thm:gheigcomp}.

\begin{lemma}\label{lem:gh}
The gradient $\bG$ in Eq. \eqref{eq:compactg} is given in terms of $\bW=(\bThe\oplus\bPsi)^{-1}$
as 
\begin{align}\label{eq:partialmatrix0}
            \bG_{\bThe} = q\bm{S} - \bW_{\bThe}, 
            \quad
            \bG_{\bPsi} = p\bm{T} - \bW_{\bPsi}, 
\end{align}
where
\begin{align}\label{eq:partialmatrixW}
            \bW_{\bThe} =  \sum_{i=1}^q(\bI_p\otimes\bm{e}_{q,i})^T\bW(\bI_p\otimes\bm{e}_{q,i}), \quad
            \bW_{\bPsi} = \sum_{i=1}^p(\bm{e}_{p,i}\otimes\bI_q)^T\bW(\bm{e}_{p,i}\otimes\bI_q).
\end{align}
The Hessian $\bH$ in Eq. \eqref{eq:compacth} is given as
\begin{equation}\label{eq:HcollapseWW}
    \bH = \begin{bmatrix}
            \bH_{\bThe}  & \bH_{\bThe\bPsi}\\
            \bH_{\bThe\bPsi}^T  & \bH_{\bPsi}
        \end{bmatrix}
    = \begin{bmatrix}
    \bP_{\bThe}^T(\bW\otimes\bW)\bP_{\bThe} & \bP_{\bThe}^T(\bW\otimes\bW)\bP_{\bPsi}  \\
    \bP_{\bPsi}^T(\bW\otimes\bW)\bP_{\bThe} & \bP_{\bPsi}^T(\bW\otimes\bW)\bP_{\bPsi}
    \end{bmatrix}
    = \bP^T(\bW\otimes\bW)\bP,
\end{equation}
where $\bP=\left[\bP_{\bThe},\bP_{\bPsi}\right]$ is a $(p^2q^2)\times (p^2+q^2)$ matrix with
$\bP_{\bThe}=\sum_{i=1}^q \bI_p\otimes\bm{e}_{q,i}\otimes\bI_p\otimes\bm{e}_{q,i}$ and $ 
    \bP_{\bPsi} = \sum_{j=1}^p \bm{e}_{p,j}\otimes\bI_q\otimes\bm{e}_{p,j}\otimes\bI_q$.
\end{lemma}
\begin{proof}
    Let $\bE_{p, ij}=\bm{e}_{p,i}\bm{e}_{p,j}^T$ be a $p\times p$ one-hot matrix with $1$ at 
    the $(i,j)$th element and zero elsewhere. Then, for the gradient $\bG$, we have
    \begin{align*}
         [\bG_{\bThe}]_{ij}  &= q[\bm{S}]_{ij} - \tr\left(\bW(\bE_{p, ij}\otimes\bI_q)\right) \\
         &= q[\bm{S}]_{ij} - \tr\big(\bW(\bE_{p, ij}\otimes\sum_l\bm{e}_{q,l}\bm{e}_{q,l}^T)\big) \\
         &= q[\bm{S}]_{ij} - \tr\big(\sum_l\bW(\bI_p\otimes\bm{e}_{q,l})(\bE_{p,ij}\otimes1)(\bI_p\otimes\bm{e}_{q,l})^T\big) \\
         &= q[\bm{S}]_{ij} - \tr\big(\sum_l(\bI_p\otimes\bm{e}_{q,l})^T\bW(\bI_p\otimes\bm{e}_{q,l})\bE_{p,ij}\big).
    \end{align*}         
    By using $\tr(\bm{A}\bE_{p,ij}) = \bm{A}_{ij}$, we collect the elements $[\bG_{\bThe}]_{ij}$ to obtain Eqs. \eqref{eq:partialmatrix0} and \eqref{eq:partialmatrixW}. The case for $\bG_{\bPsi}$ can be shown similarly. Now, for the Hessian, we have
    \begin{align*}
        [\bH_{\bThe}&]_{jp+i,lp+k}
        = \tr\left(\bW(\bE_{p,ij}\otimes\bI_q)\bW(\bE_{p,kl}\otimes\bI_q)\right)\\
        &=\tr\left(\left[\sum_r  \bW(\bI_p\otimes\bm{e}_{q,r})\bE_{p,ij}(\bI_p\otimes\bm{e}_{q,r})^T \right] \left[\sum_s\bW(\bI_p\otimes\bm{e}_{q,s})\bE_{p,kl}(\bI_p\otimes\bm{e}_{q,s})^T\right]  \right)\\
        &= \tr\left( \sum_{r,s} (\bI_p\otimes\bm{e}_{q,s})^T\bW(\bI_p\otimes\bm{e}_{q,r})\bE_{p,ij}(\bI_p\otimes\bm{e}_{q,r})^T\bW(\bI_p\otimes\bm{e}_{q,s})\bE_{p,kl} \right),\\
        \intertext{by $\tr(\bm{A}\bm{B}\bm{A}\bm{C}) = \tvec(\bm{B})^T(\bm{A}\otimes\bm{A})\tvec(\bm{C})$ for symmetric matrices $\bm{A}$, $\bm{B}$, and $\bm{C}$, and $\tvec(\bE_{p,ij}) = \bm{e}_{p^2,jp+i}$,}
        &= \bm{e}_{p^2,jp+i}^T \left( \sum_{r,s} (\bI_p\otimes\bm{e}_{q,r})^T\bW(\bI_p\otimes\bm{e}_{q,s}) \otimes (\bI_p\otimes\bm{e}_{q,r})^T\bW(\bI_p\otimes\bm{e}_{q,s}) \right) \bm{e}_{p^2,lp+k}\\
        &= \bm{e}_{p^2,jp+i}^T \left( \sum_{r,s} (\bI_p\otimes\bm{e}_{q,r}\otimes\bI_p\otimes\bm{e}_{q,r})^T(\bW\otimes\bW)(\bI_p\otimes\bm{e}_{q,s}\otimes\bI_p\otimes\bm{e}_{q,s}) \right) \bm{e}_{p^2,lp+k}.
    \end{align*}
We collect the elements into a matrix 
\begin{align}
 \bH_{\bThe} =  \bigg(\sum_{r=1}^q\bI_p\otimes\bm{e}_{q,r}\otimes\bI_p\otimes\bm{e}_{q,r}\bigg)^T(\bW\otimes\bW)\bigg(\sum_{s=1}^q \bI_p\otimes\bm{e}_{q,s}\otimes\bI_p\otimes\bm{e}_{q,s}\bigg). \nonumber
\end{align}    
The cases for $\bH_{\bPsi}$ and $\bH_{\bThe\bPsi}$ can be shown in a similar way.
\end{proof}

According to Lemma \ref{lem:gh}, the gradient $\bG$ is obtained by collapsing
$\bW$ as in Eq. \eqref{eq:partialmatrixW} and 
the Hessian $\bH$ is obtained by collapsing $\bW \otimes \bW$ as in Eq. \eqref{eq:HcollapseWW}.
This collapse is illustrated in Figure \ref{fig:collapseW}(b) for the gradient. 
While the Kronecker sum inflates $\bThe$ and $\bPsi$ into $\bOmega$ (Figure \ref{fig:collapseW}(a)), 
this inflated structure gets deflated, 
when computing $\bW_{\bThe}$ and $\bW_{\bPsi}$ in the gradients from $\bW=\bOmega^{-1}$ 
(Figure \ref{fig:collapseW}(b)). 
This deflation can be viewed as applying the mask $\One_{p} \oplus \One_{q}$,
where $\One_{a}$ is a $a\times a$ matrix of one's, to $\bW$ 
to obtain $\bW^{\text{masked}}$ such that $[\bW^{\text{masked}}]_{ij} =[\bW]_{ij}$  
if $[\One_{p} \oplus \One_{q}]_{ij} \neq 0$, otherwise $[\bW^{\text{masked}}]_{ij} = 0 $,
and replacing $\bW$ in Eq. \eqref{eq:partialmatrixW} with $\bW^{\text{masked}}$. 
In other words, only the non-zero elements of $\bW^{\text{masked}}$ 
contribute towards $\bH_{\bThe}$ (blue in Figure \ref{fig:collapseW}(b))
and $\bH_{\bPsi}$ (yellow in Figure \ref{fig:collapseW}(b)). 
The collapse of $\bW \otimes \bW$ in Eq. \eqref{eq:HcollapseWW} to obtain the Hessian
can be viewed as the same type of deflation applied twice to $\bW \otimes \bW$ (details in Appendix A).

Lemma \ref{lem:gh} reveals the challenge in computing $\bG$ and $\bH$ 
in a direct application of QUIC to the Kronecker-sum model: $\bG$
involves computing the large $pq\times pq$ matrix $\bW$ and $\bH$ involves
computing the even larger $p^2q^2\times p^2q^2$ matrix $\bW \otimes \bW$.
In Theorem \ref{thm:gheigcomp} below, we show that $\bG$ and $\bH$ can be obtained
via eigendecomposition of $\bThe$ and $\bPsi$ without explicitly constructing $\bW$ and $\bW\otimes \bW$. 


\begin{theorem}\label{thm:gheigcomp}
Given the eigendecomposition  $\bThe=\bm{Q}_{\bThe}\bLambda_{\bThe}\bm{Q}_{\bThe}^T$  
and  $\bPsi=\bm{Q}_{\bPsi}\bLambda_{\bPsi}\bm{Q}_{\bPsi}^T$ 
with the $k$th smallest eigenvalue
$\lambda_{\bThe,k}$ and $\lambda_{\bPsi,k}$ in the $(k,k)$th element of  $\bLambda_{\bThe}$ and $\bLambda_{\bPsi}$,
the gradient in Eq. \eqref{eq:partialmatrix0} is given as
\begin{equation}\label{eq:partialmatrixeiggrad}
\begin{aligned}
     \bG_{\bThe} = q\bm{S} - \bm{Q}_{\bThe}\left(\sum_{k=1}^q 
     \bXi_{\bThe,k}
     \right)\bm{Q}_{\bThe}^T, \quad\quad
     \bG_{\bPsi} = p\bm{T} - \bm{Q}_{\bPsi}\left(\sum_{l=1}^p 
     \bXi_{\bPsi,l}
     \right)\bm{Q}_{\bPsi}^T,
\end{aligned}
\end{equation}
where
\begin{align}
    \bXi_{\bThe,k} = \left(\bm{\Lambda}_{\bThe} + \lambda_{\bPsi,k}\bm{I}_p\right)^{-1}, \quad\quad \bXi_{\bPsi,l} = \left(\bm{\Lambda}_{\bPsi} + \lambda_{\bThe,l}\bm{I}_q\right)^{-1}. \nonumber 
\end{align}
For the Hessian, we have
\begin{equation}
\!\bH = \begin{bmatrix}
            \bH_{\bThe} & \!\!\! \bH_{\bThe\bPsi}\\
    \bH_{\bThe\bPsi}^T & \!\!\! \bH_{\bPsi}
        \end{bmatrix}
    = \begin{bmatrix}
    \bm{Q}_{\bThe}\otimes\bm{Q}_{\bThe}&\!\!\!\!\bm{0}\\
    \bm{0}& \!\!\!\!\bm{Q}_{\bPsi}\otimes\bm{Q}_{\bPsi}
    \end{bmatrix}
    \begin{bmatrix}
    \bLambda_{\bH_{\bThe}} & \!\!\!\!\! \bm{\Upsilon} \\
    \bm{\Upsilon}^T & \!\!\!\!\! \bLambda_{\bH_{\bPsi}}
    \end{bmatrix}
    \begin{bmatrix}
    \bm{Q}_{\bThe}\otimes\bm{Q}_{\bThe}&\!\!\!\!\bm{0}\\
    \bm{0}&\!\!\!\!\bm{Q}_{\bPsi}\otimes\bm{Q}_{\bPsi}
    \end{bmatrix}^T\!\!\!,
\label{eq:partialmatrixeigH}
\end{equation}
where 
\begin{align}
    \bLambda_{\bH_{\bThe}} &= \sum_{k=1}^q \bXi_{\bThe,k} \otimes \bXi_{\bThe,k}, \quad
     \bLambda_{\bH_{\bPsi}} = \sum_{l=1}^p \bXi_{\bPsi,l} \otimes \bXi_{\bPsi,l}, \label{eq:Hcomp} \\
     \bm{\Upsilon} &= \tilde{\bm{I}}_{p^2,p}
        \begin{bmatrix}
        \lambda^2_{\bW,11} & \cdots & \lambda^2_{\bW,1q} \\
        \vdots & \ddots & \vdots \\
        \lambda^2_{\bW,p1} & \cdots & \lambda^2_{\bW,pq}
        \end{bmatrix}
        \tilde{\bm{I}}_{q,q^2}. \nonumber
\end{align}
$\tilde{\bm{I}}_{p^2,p}$ and $\tilde{\bm{I}}_{q,q^2}$ above are stretched 
identity matrices, $\tilde{\bm{I}}_{p^2,p} =
        \begin{bmatrix}
        \bE_{p,11}\;
        \bE_{p,22}\;
        \cdots\;
        \bE_{p,pp}
        \end{bmatrix}^T$ and $\tilde{\bm{I}}_{q,q^2} =
        \begin{bmatrix}
        \bE_{q,11}\;
        \bE_{q,22}\;
        \cdots\;
        \bE_{q,qq}
        \end{bmatrix}$, 
where $\bE_{p, ij}$ is a $p\times p$ one-hot matrix with $1$ at the $(i,j)$th element and zero elsewhere, 
and $\lambda_{\bW,lk} = (\lambda_{\bThe,l} + \lambda_{\bPsi,k})^{-1}$.
\end{theorem}
\begin{proof}
We prove for $\bG_{\bThe}$, $\bH_{\bThe}$, and $\bH_{\bThe\bPsi}$, 
because the case for $\bG_{\bPsi}$ and $\bH_{\bPsi}$ can be proved similarly. 
Let $\bm{q}_{\bThe,i}$ and $\bm{q}_{\bPsi,i}$ be the $i$th eigenvectors of $\bThe$ and $\bPsi$, given as the $i$th columns of $\bm{Q}_{\bThe}$ and $\bm{Q}_{\bPsi}$. Then, we can re-write Eq. \eqref{eq:Weigdecomp} as
\begin{equation}\label{eq:Weigdecompsum}
    \bW = \sum_{l=1}^p \sum_{k=1}^q \lambda_{\bW,lk}(\bm{q}_{\bThe,l}\otimes\bm{q}_{\bPsi,k})(\bm{q}_{\bThe,l}\otimes\bm{q}_{\bPsi,k})^T.
\end{equation}

For $\bG_{\bThe}$, we substitute $\bW$ in $\bW_{\bThe}$ in Eq. \eqref{eq:partialmatrixW} with Eq. \eqref{eq:Weigdecompsum}:
\begin{align*}
    \bG_{\bThe}
    &= q\bm{S} - \sum_{i=1}^q(\bm{I}_p\otimes\bm{e}_{q,i})^T\left(\sum_{l=1}^p \sum_{k=1}^q \lambda_{\bW,lk}(\bm{q}_{\bThe,l}\otimes\bm{q}_{\bPsi,k})(\bm{q}_{\bThe,l}\otimes\bm{q}_{\bPsi,k})^T\right)(\bm{I}_p\otimes\bm{e}_{q,i})\\
    &= q\bm{S} - \sum_{i=1}^q\sum_{l=1}^p \sum_{k=1}^q\lambda_{\bW,lk}([\bm{q}_{\bPsi,k}]_i\bm{q}_{\bThe,l})([\bm{q}_{\bPsi,k}]_i\bm{q}_{\bThe,l})^T \quad\text{since $(\bI\otimes\bm{e}_{q,i})^T(\bm{a}\otimes\bm{b})=\bm{b}_i\bm{a}$}\\
    &= q\bm{S} - \sum_{l=1}^p \sum_{k=1}^q \lambda_{\bW,lk}\bm{q}_{\bThe,l}\bm{q}_{\bThe,l}^T
    \;\quad\quad\quad\quad\quad\quad\quad\quad\quad\quad\text{since $\bm{Q}_{\bPsi}$ is orthonormal}\\
    &= q\bm{S} - \sum_{k=1}^q \bm{Q}_{\bThe}(\bLambda_{\bThe} + \lambda_{\bPsi,k}\bI_p)^{-1}\bm{Q}_{\bThe}^T  .
\end{align*}
This is equivalent to $\bG_{\bThe}$ in Eq. \eqref{eq:partialmatrixeiggrad}, 
since the summation can be performed over $\bXi_{\bThe,k}$, 
instead of over $\bm{Q}_{\bThe}\bXi_{\bThe,k}\bm{Q}_{\bThe}^T$. 

To show $\bH_{\bThe}$, we again substitute $\bW$ in $\bH_{\bThe}$ of Eq. \eqref{eq:HcollapseWW}
with Eq. \eqref{eq:Weigdecompsum}. 
\begin{align*}
    \bH_{\bThe} &=  \sum_{i=1}^q \sum_{j=1}^q \bigg(\sum_{l=1}^p \sum_{k=1}^q\lambda_{\bW,lk}([\bm{q}_{\bPsi,k}]_i\bm{q}_{\bThe,l})([\bm{q}_{\bPsi,k}]_j\bm{q}_{\bThe,l})^T
    \bigg. \\
	& \quad\quad\quad\quad\quad\quad\quad\quad\quad\quad\quad\quad\quad\quad\quad\quad \bigg.\otimes\sum_{r=1}^p \sum_{s=1}^q\lambda_{\bW,rs}([\bm{q}_{\bPsi,s}]_i\bm{q}_{\bThe,r})([\bm{q}_{\bPsi,s}]_j\bm{q}_{\bThe,r})^T \bigg)\\
    &=\sum_{k=1}^q\bigg(\sum_{l=1}^p \lambda_{\bW,lk}\bm{q}_{\bThe,l}\bm{q}_{\bThe,l}^T\otimes\sum_{r=1}^p\lambda_{\bW,rk}\bm{q}_{\bThe,r}\bm{q}_{\bThe,r}^T\bigg) \\
    &= \sum_{k=1}^q \bigg(\bm{Q}_{\bThe}(\bLambda_{\bThe} + \lambda_{\bPsi,k}\bI_p)^{-1}\bm{Q}_{\bThe}^T\bigg) \otimes \bigg(\bm{Q}_{\bThe}(\bLambda_{\bThe} + \lambda_{\bPsi,k}\bI_p)^{-1}\bm{Q}_{\bThe}^T\bigg)\\
    &= (\bm{Q}_{\bThe}\otimes\bm{Q}_{\bThe}) \bigg(\sum_{k=1}^q (\bLambda_{\bThe} + \lambda_{\bPsi,k}\bI_p)^{-1}\otimes (\bLambda_{\bThe} + \lambda_{\bPsi,k}\bI_p)^{-1}\bigg)(\bm{Q}_{\bThe}\otimes\bm{Q}_{\bThe})^T.
\end{align*}
For $\bH_{\bThe\bPsi}$, we substitute $\bW$ in $\bH_{\bThe\bPsi}$ of Eq. \eqref{eq:HcollapseWW}
with Eq. \eqref{eq:Weigdecompsum}.
\begin{align*}
        \bH_{\bThe\bPsi} &= \sum_{i=1}^q \sum_{j=1}^p \bigg(\sum_{l=1}^p \sum_{k=1}^q\lambda_{\bW,lk}([\bm{q}_{\bPsi,k}]_i\bm{q}_{\bThe,l})([\bm{q}_{\bThe,l}]_j\bm{q}_{\bPsi,k})^T\bigg. \\
	    &\bigg. \quad\quad\quad\quad\quad\quad\quad\quad\quad\quad\quad\quad\quad\quad \otimes\sum_{r=1}^p \sum_{s=1}^q\lambda_{\bW,rs}([\bm{q}_{\bPsi,s}]_i\bm{q}_{\bThe,r})([\bm{q}_{\bThe,r}]_j\bm{q}_{\bPsi,s})^T \bigg)\\
        &=\sum_{l=1}^p \sum_{k=1}^q \lambda_{\bW,lk}^2(\bm{q}_{\bThe,l}\bm{q}_{\bPsi,k}^T \otimes \bm{q}_{\bThe,l}\bm{q}_{\bPsi,k}^T) \\
        &= \sum_{l=1}^p (\bm{q}_{\bThe,l}\otimes\bm{q}_{\bThe,l})\sum_{k=1}^q\lambda_{\bW,lk}^2(\bm{q}_{\bPsi,k}\otimes\bm{q}_{\bPsi,k})^T\\
        &= \begin{bmatrix}
        \bm{q}^T_{\bThe,1}\otimes\bm{q}^T_{\bThe,1}\\
        \bm{q}^T_{\bThe,2}\otimes\bm{q}^T_{\bThe,2}\\
        \vdots\\
        \bm{q}^T_{\bThe,p}\otimes\bm{q}^T_{\bThe,p}
        \end{bmatrix}^T
        \begin{bmatrix}
        \lambda^2_{\bW,11} & \cdots & \lambda^2_{\bW,1q} \\
        \vdots & \ddots & \vdots \\
        \lambda^2_{\bW,p1} & \cdots & \lambda^2_{\bW,pq}
        \end{bmatrix}
        \begin{bmatrix}
        \bm{q}^T_{\bPsi,1}\otimes\bm{q}^T_{\bPsi,1}\\
        \bm{q}^T_{\bPsi,2}\otimes\bm{q}^T_{\bPsi,2}\\
        \vdots\\
        \bm{q}^T_{\bPsi,q}\otimes\bm{q}^T_{\bPsi,q}
        \end{bmatrix}\\
        &=(\bm{Q}_{\bThe}\otimes\bm{Q}_{\bThe})\tilde{\bm{I}}_{p^2,p}
        \begin{bmatrix}
        \lambda^2_{\bW,11} & \cdots & \lambda^2_{\bW,1q} \\
        \vdots & \ddots & \vdots \\
        \lambda^2_{\bW,p1} & \cdots & \lambda^2_{\bW,pq}
        \end{bmatrix}
        \tilde{\bm{I}}_{q,q^2}(\bm{Q}_{\bPsi}\otimes\bm{Q}_{\bPsi})^T.
    \end{align*}
\end{proof}

Theorem \ref{thm:gheigcomp} provides a significantly more efficient method for computing $\bG$ and $\bH$, 
compared to Lemma \ref{lem:gh}. Since $\bG$ and $\bH$ can be written entirely in terms of 
the eigenvectors and eigenvalues of $\bThe$ and $\bPsi$, the operations that involve
$\bW=(\bThe \oplus \bPsi)^{-1}$ 
is replaced with the cheaper operation of eigendecomposition with time $\mathcal{O}(p^3+q^3)$ and space $\mathcal{O}(p^2 + q^2)$. 
TeraLasso also used the eigendecomposition of $\bThe$ and $\bPsi$ in each iteration of the optimization. However, TeraLasso has done this 
for an efficient projection of the gradient to the re-parameterized space, whereas 
EiGLasso performs the eigendecomposition for an efficient computation of $\bG$ and $\bH$ in the original parameter space. 

It follows from Theorem \ref{thm:gheigcomp} that $(\bThe,\bPsi)$'s in the same equivalence class 
in Eq. \eqref{eqn:noniddiag} differ only in their eigenvalues, but not in eigenvectors. 
The equivalence class in Eq. \eqref{eqn:noniddiag} can be written equivalently as
 \begin{align}\label{eqn:noniddiagEig}
  S[\bOmega] = \{ (\bThe,\bPsi)\in \mathbb{KS}^{p,q}| 
	\left(\bLambda_{\bThe}-c\bm{I}_p\right)\oplus\left(\bLambda_{\bPsi}+c\bm{I}_q\right)=\bLambda_{\bOmega},
		c\in \mathbb{R}\}, 
 \end{align}
since $\bThe + c\bI_p = \bm{Q}_{\bThe}\bLambda_{\bThe}\bm{Q}_{\bThe}^T +c\bm{Q}_{\bThe}\bm{Q}_{\bThe}^T 
= \bm{Q}_{\bThe}(\bLambda_{\bThe}+c\bI_p)\bm{Q}_{\bThe}^T$ and similarly for $\bPsi$.
We exploit this feature 
to handle the unidentifiable parameters during estimation in Section \ref{sec:identify} and 
to analyze the convergence in Section \ref{sec:convg}.

In the theorem below, we prove several properties of $\bH$ that we will use to prove
the line-search properties, convergence guarantees, and convergence rates 
in Section \ref{sec:convg}.

\begin{theorem}\label{thm:Hpd}
The Hessian $\bH$ in EiGLasso has the following properties:
\begin{itemize}
    \item It is positive semi-definite with the null space 
$\text{null}(\bH)=\{[\tvec(\bm{X}_p)^T \tvec(\bm{X}_q)^T]^T| \bm{X}_p\oplus\bm{X}_q = 0,
\bm{X}_p \in \mathbb{R}^{p\times p}, \bm{X}_p \in \mathbb{R}^{q\times q}\}$.
    
    \item On the Kronecker-sum space in Eq. \eqref{def:ksspace},
$\bH$ is positive definite.
The minimum eigenvalue $\lambda_{\bH,\min_0}$ outside of the nullspace and the maximum eigenvalue 
$\lambda_{\bH,\max}$ are given in terms of the eigenvalues of $\bThe$ and $\bPsi$:
\begin{equation}
\lambda_{\bH,\min_0}=\min\{p,q\}(\lambda_{\bThe,p}+\lambda_{\bPsi,q})^{-2}, \quad
\lambda_{\bH,\max}=(p+q)(\lambda_{\bThe,1}+\lambda_{\bPsi,1})^{-2}. \nonumber 
\end{equation}    
\end{itemize}
\end{theorem}
\begin{proof}
To prove the first property, from Lemma \ref{lem:gh}, we notice  that since $\bW$ is positive definite, 
the null space of $\bH$ is given by $\bm{X}_p$ and $\bm{X}_q$ that satisfy
\begin{align*}
    \bP\begin{bmatrix}
        \tvec(\bm{X}_p)\\
        \tvec(\bm{X}_q)
    \end{bmatrix}
    &=0.
\end{align*} 
Since the left-hand side of the above can be written as
\begin{align*}
  \bP\begin{bmatrix}
        \tvec(\bm{X}_p)\\
        \tvec(\bm{X}_q)
    \end{bmatrix}
    &=\sum_{i=1}^q \tvec\left((\bI_p\otimes\bm{e}_{q,i})\bm{X}_p(\bI_p\otimes\bm{e}_{q,i})^T \right) + \sum_{j=1}^p \tvec\left((\bm{e}_{p,j}\otimes\bI_q)\bm{X}_q(\bm{e}_{p,j}\otimes\bI_q)^T \right) \\
    &=\tvec\left( \bm{X}_p\otimes\bI_q\right) +
    \tvec\left( \bI_p\otimes\bm{X}_q\right) \\
    &= \tvec\left(\bm{X}_p\oplus\bm{X}_q\right),
\end{align*}
the null space of $\bH$ is 
given by $\bm{X}_p$ and $\bm{X}_q$ that satisfy $\bm{X}_p\oplus\bm{X}_q = 0$.

To prove the second property, we notice that the null space of $\bH$ is outside of
the constraint region $\bThe\oplus\bPsi\succ0$;  
thus, $\bH$ is positive definite on the Kronecker-sum space in Eq. \eqref{def:ksspace}.
To find the eigenvalues of $\bH$, we first find the eigenvalues of 
\begin{align*}
    \bP^T\bP =
    \begin{bmatrix}
        \bP_{\bThe}^T\bP_{\bThe} & \bP_{\bThe}^T\bP_{\bPsi}\\
        \bP_{\bPsi}^T\bP_{\bThe} & \bP_{\bPsi}^T\bP_{\bPsi}
    \end{bmatrix}
    =
    \begin{bmatrix}
        q\bI_{p^2} & \tvec(\bI_p)\tvec(\bI_q)^T\\
        \tvec(\bI_q)\tvec(\bI_p)^T & p\bI_{q^2}
    \end{bmatrix},
\end{align*}
by finding the solutions of the characteristic equation $\det(\bP^T\bP - \lambda\bI)=0$:
\begin{align*}
    \det(&\bP^T\bP - \lambda\bI_{p^2+q^2})  
    = \det(q\bI_{p^2}-\lambda\bI_{p^2})\det\left(p\bI_{q^2} 
    - \lambda\bI_{q^2} - \frac{p}{q-\lambda}\tvec(\bI_q)\tvec(\bI_q)^T \right)\\
    &=  (q-\lambda)^{p^2} 
    \det\left((p-\lambda)\bI_{q^2} - \frac{p}{q-\lambda}\tvec(\bI_q)\tvec(\bI_q)^T \right)\\
    &=  (q-\lambda)^{p^2-1} (p-\lambda)^{q^2-1}
    \big((p-\lambda)(q-\lambda)-pq\big) \quad\quad \text{from matrix determinant lemma}\\
    &= (p-\lambda)^{q^2-1}(q-\lambda)^{p^2-1}\lambda\big(\lambda-(p+q)\big) = 0.
\end{align*}
Thus, the eigenvalues of $\bP^T\bP$ are $0$, $p$, $q$, and $(p+q)$, each with the algebraic multiplicity
$1$, $q^2-1$, $p^2-1$, and 1. 
From Eq. \eqref{eq:HcollapseWW}, for any unit vector $\bm{u}\in\Real^{p^2+q^2}\setminus\text{null}(\bH)$,
\begin{align*}
    \lambda_{\bW,\min}^2(\bm{u}^T\bP^T\bP\bm{u})
    \le
        \bm{u}^T
    \bP^T(\bW\otimes\bW)\bP \bm{u}
    \le \lambda_{\bW,\max}^2(\bm{u}^T\bP^T\bP\bm{u}),
\end{align*}
where $\lambda_{\bW,\min}= (\lambda_{\bThe,p}+\lambda_{\bPsi,q})^{-2}$ and 
$\lambda_{\bW,\max}= (\lambda_{\bThe,1}+\lambda_{\bPsi,1})^{-2}$ are the minimum and maximum eigenvalues of $\bW$.
Combining this with the minimum eigenvalue $\min\{p,q\}$ 
and the maximum eigenvalue $(p+q)$ of $\bP^T\bP$ outside of the null space of $\bH$ results in
\begin{equation*}
    \min\{p,q\}\lambda_{\bW,\min}^2\le\lambda_{\bW,\min}^2(\bm{u}^T\bP^T\bP\bm{u})\quad\text{and}
    \quad \lambda_{\bW,\max}^2(\bm{u}^T\bP^T\bP\bm{u})\le (p+q)\lambda_{\bW,\max}^2,
\end{equation*}
which completes the proof. 
\end{proof}


\subsection{EiGLasso with Approximate Hessian}\label{sec:approxH}

While EiGLasso pre-computes and stores $\bG$ after eigendecomposing  $\bThe$ and $\bPsi$
at the beginning of the coordinate descent to solve Eq. \eqref{eq:newton}, 
explicitly computing and storing the $(p^2+q^2)\times(p^2+q^2)$ 
matrix $\bH$ is still expensive for large $p$ and $q$ even with the eigendecomposition. 
To further reduce the computation time and memory, we approximate $\bH$ 
based on its form in Theorem \ref{thm:gheigcomp} as follows.
We first drop $\bH_{\bThe\bPsi}$ while keeping only $\bH_{\bThe}$ and $\bH_{\bPsi}$.
Since pre-computing and storing the $p^2\times p^2$ matrix $\bH_{\bThe}$ and 
the $q^2\times q^2$ matrix $\bH_{\bPsi}$ is still expensive,
we approximate $\bH_{\bThe}$ and $\bH_{\bPsi}$ by approximating their eigenvalues $\bLambda_{\bH_{\bThe}}$ 
and $\bLambda_{\bH_{\bPsi}}$ in Eq. \eqref{eq:partialmatrixeigH}, while keeping the same eigenvectors
$\bm{Q}_{\bThe}\otimes\bm{Q}_{\bThe}$ and $\bm{Q}_{\bPsi}\otimes\bm{Q}_{\bPsi}$. 
The resulting approximate Hessian $\hat{\bH}$ is 
\begin{equation}\label{eq:approxH}
    \hat{\bH} = \begin{bmatrix}
            \hat{\bH}_{\bThe} & \!\!\!\bm{0} \\
    \bm{0} & \!\!\! \hat{\bH}_{\bPsi}
        \end{bmatrix}\\
    = \begin{bmatrix}
    \bm{Q}_{\bThe}\otimes\bm{Q}_{\bThe}&\!\!\!\bm{0}\\
    \bm{0}&\!\!\!\bm{Q}_{\bPsi}\otimes\bm{Q}_{\bPsi}
    \end{bmatrix}
    \begin{bmatrix}
    \hat{\bLambda}_{\bH_{\bThe}} & \!\!\bm{0} \\
    \bm{0} & \!\!\hat{\bLambda}_{\bH_{\bPsi}}
    \end{bmatrix}
    \begin{bmatrix}
    \bm{Q}_{\bThe}\otimes\bm{Q}_{\bThe}& \!\!\!\bm{0}\\
    \bm{0}&\!\!\!\bm{Q}_{\bPsi}\otimes\bm{Q}_{\bPsi}
    \end{bmatrix}^T\!\!\!,
\end{equation}
where
\begin{equation}\label{eqn:approxHL}
\begin{alignedat}{2}
    \hat{\bLambda}_{\bH_{\bThe}} &= \sum_{k=1}^{K} \bXi_{\bThe,k}\otimes\bXi_{\bThe,k} 
        + (q-K) \bXi_{\bThe,K} \otimes \bXi_{\bThe,K},\\
    \hat{\bLambda}_{\bH_{\bPsi}} &= \sum_{l=1}^{K} \bXi_{\bPsi,l}\otimes\bXi_{\bPsi,l} 
        + (p-K) \bXi_{\bPsi,K} \otimes \bXi_{\bPsi,K}.
\end{alignedat}
\end{equation}
In $\hat{\bLambda}_{\bH_{\bThe}}$, we keep only the components $\bXi_{\bThe,k}\otimes \bXi_{\bThe,k}$'s 
of $\bLambda_{\bH_{\bThe}}$ in Eq. \eqref{eq:Hcomp} 
that contain the $K$ smallest eigenvalues $\lambda_{\bPsi,1}, \ldots, \lambda_{\bPsi,K}$. 
These smallest eigenvalues have the largest contribution toward $\hat{\bH}_{\bThe}$, 
since the eigenvalues of $\bPsi$ contribute to the eigenvalues of $\bH_{\bThe}$ through their inverses. 
We replace the remaining  $(q-K)$ components with $\bXi_{\bThe,K}\otimes \bXi_{\bThe,K}$, 
because dropping the $(q-K)$ eigenvalue components in ${\bLambda}_{\bH_{\bThe}}$ would amount to 
assuming that $\bPsi$ has $(q-K)$ eigenvalues of infinite magnitude. 
This approximation leads to the following form 
\begin{align*}
    	\hat{\bm{H}}_{\bThe} &= \sum_{k=1}^K \bm{V}_{\bThe, k}\otimes\bm{V}_{\bThe, k} 
		+ (q-K) \bm{V}_{\bThe, K+1} \otimes \bm{V}_{\bThe, K+1}, \\
    	\hat{\bm{H}}_{\bPsi} &= \sum_{k=1}^K \bm{V}_{\bPsi, k}\otimes\bm{V}_{\bPsi, k} 
		+ (q-K) \bm{V}_{\bPsi, K+1} \otimes \bm{V}_{\bPsi, K+1},
\end{align*}
where $\bm{V}_{\bThe,k} = \bm{Q}_{\bThe}\bXi_{\bThe,k}\bm{Q}_{\bThe}^T$
and $\bm{V}_{\bPsi,k} = \bm{Q}_{\bPsi}\bXi_{\bPsi,k}\bm{Q}_{\bPsi}^T$.
Then, during the coordinate descent optimization, we pre-compute and store $\bm{V}_{\bThe,k}$'s 
and $\bm{V}_{\bPsi,k}$'s for $k=1,\ldots,K$, 
and compute the Hessian entries from $\bm{V}_k$'s as needed in each coordinate descent update (details in Appendix B).
We assume the same  $K$ for both $\hat{\bH}_{\bThe}$ and $\hat{\bH}_{\bPsi}$, though this can be relaxed.
In our experiments in Section \ref{sec:exp} we demonstrate that often $K=1$ suffices.

We provide a geometric interpretation of our approximation of $\bH_{\bThe}$ with $\hat{\bH}_{\bThe}$.
Assuming Eq. \eqref{eq:newton} without the $L_1$-regularization,
we have the  descent direction $-\bH_{\bThe}^{-1}\tvec(\bG_{\bThe})$ with the exact Hessian
and $-\hat{\bH}^{-1}_{\bThe}\tvec(\bG_{\bThe})$ with the approximate Hessian. 
Since $\bH_{\bThe}$ and $\hat{\bH}_{\bThe}$ share the same eigenvectors, 
in the coordinate system defined  by these eigenvectors as bases, 
the descent directions become $-\bH_{\bThe}^{-1/2} \tvec(\bG_{\bThe}) =
-{\bLambda_{\bH_{\bThe}}^{-\frac{1}{2}}} (\bm{Q}_{\bThe}\otimes\bm{Q}_{\bThe})^T
    \tvec(\bG_{\bThe})$ with the exact Hessian and
$-\hat{\bH}_{\bThe}^{-\frac{1}{2}}\tvec(\bG_{\bThe}) 
    =  -\hat{\bLambda}_{\bH_{\bThe}}^{-\frac{1}{2}}
    (\bm{Q}_{\bThe}\otimes\bm{Q}_{\bThe})^T \tvec(\bG_{\bThe})$ with the approximate Hessian
    \citep{cvxbook}. The latter is an element-wise scaled matrix of the former.
Furthermore, since $[\bLambda_{\bH_{\bThe}}^{-\frac{1}{2}}]_{ii} <= [\hat{\bLambda}_{\bH_{\bThe}}^{-\frac{1}{2}}]_{ii}$, 
in this new coordinate system,
each element of the descent direction with $\hat{\bH}_{\bThe}$ is a convex combination of $0$ and 
the corresponding element of the descent direction with $\bH_{\bThe}$.


The following theorem provides the properties of $\hat{\bH}$, which will be used to analyze
the convergence of EiGLasso in Section \ref{sec:convg}.

\begin{theorem}\label{thm:Happroxpd}
The approximate Hessian $\hat{\bH}$ is positive definite. Furthermore, its minimum and maximum eigenvalues are 
\begin{equation} 
\lambda_{\hat{\bH},\min}=\min\big\{\lambda_{\hat{\bH}_{\bThe},\min}, \lambda_{\hat{\bH}_{\bPsi},\min}\big\}, \quad
\lambda_{\hat{\bH},\max}=\max\big\{\lambda_{\hat{\bH}_{\bThe},\max},\lambda_{\hat{\bH}_{\bPsi},\max}\big\},
\nonumber
\end{equation}
where
\begin{align*}
    \lambda_{\hat{\bH}_{\bThe},\min} &= \sum_{i=1}^K (\lambda_{\bThe, p} + \lambda_{\bPsi,i})^{-2} +  (q-K)(\lambda_{\bThe, p} + \lambda_{\bPsi,K})^{-2}\\
    \lambda_{\hat{\bH}_{\bThe},\max} &= \sum_{i=1}^K (\lambda_{\bThe, 1} + \lambda_{\bPsi,i})^{-2} +  (q-K)(\lambda_{\bThe, 1} + \lambda_{\bPsi,K})^{-2}\\
    \lambda_{\hat{\bH}_{\bPsi},\min} &= \sum_{j=1}^{K} (\lambda_{\bThe, j} + \lambda_{\bPsi,q})^{-2} +  (p-K)(\lambda_{\bThe, K} + \lambda_{\bPsi,q})^{-2}\\
    \lambda_{\hat{\bH}_{\bPsi},\max} &= \sum_{j=1}^{K} (\lambda_{\bThe, j} + \lambda_{\bPsi,1})^{-2} +  (p-K)(\lambda_{\bThe, K} + \lambda_{\bPsi,1})^{-2}.
\end{align*}
\end{theorem}
\begin{proof}
Since a block-diagonal matrix has the same set of eigenvalues as those of its diagonal blocks, 
the eigenvalues of $\hat{\bH}$ are the same as a union of those of $\hat{\bH}_{\bThe}$ 
and $\hat{\bH}_{\bPsi}$ in Eq. \eqref{eqn:approxHL}.
The eigenvalues of $\hat{\bH}_{\bThe}$ and $\hat{\bH}_{\bPsi}$ are positive, because they
are also the eigenvalues of $\bW$ 
in Eq. \eqref{eq:Weigdecomp}, the inverses of which are the eigenvalues of the positive definite 
matrix $\bOmega$ in Eq. \eqref{eq:omegaeigdecomp}. Thus, $\hat{\bH}$ is positive definite.
\end{proof}

\subsection{Estimation with the Unidentifiable Parameters} \label{sec:identify}

We describe a simple approach to estimating the diagonals of $\bThe$ and $\bPsi$ with EiGLasso.
We show that to solve the optimization problem with the constraint on a fixed trace ratio 
in Eq. \eqref{eq:objrho}, 
it is sufficient to solve the unconstrained problem in Eq. \eqref{eq:obj} and 
to adjust the diagonals once to enforce the given trace ratio after EiGLasso converges. 

In Lemma \ref{lem:invariantGH} below, we show that given the current estimate of $\bThe$ and $\bPsi$, 
the quadratic approximation in Eq. \eqref{eq:newton} for determining descent directions 
is uniquely defined. 


\begin{lemma}\label{lem:invariantGH}
Within the same equivalence class that contains the current estimate of $(\bThe, \bPsi)$,
$\bG$, $\bH$, and $\hat{\bH}$ are uniquely defined. Thus,
the second-order approximation in Eq. \eqref{eq:newton} is uniquely defined.
\end{lemma}
\begin{proof}
According to Eq. \eqref{eqn:noniddiagEig}, the unidentifiability in the diagonal elements of $\bThe$ and $\bPsi$ reduces to the unidentifiability in the eigenvalues of $\bThe$ and $\bPsi$. In the gradient in Eq. \eqref{eq:partialmatrixeiggrad} and Hessians in Eqs. \eqref{eq:partialmatrixeigH} and \eqref{eq:approxH},
the eigenvalues of $\bThe$ and $\bPsi$ always appear as a pair, where the shift $c$ cancels out as follows:
\begin{align*}
\big((\lambda_{\bThe,l}+c) + (\lambda_{\bPsi,k}-c)\big)^{-1} = (\lambda_{\bThe,l} + \lambda_{\bPsi,k})^{-1},
\forall (l,k)\in\{1,\ldots,p\}\times\{1,\ldots,q\}.
\end{align*}
Thus, $\bG$, $\bH$, and $\hat{\bH}$ are unaffected by the shift $c$ in the diagonals of $\bThe$ and $\bPsi$.
\end{proof}

Lemma \ref{lem:invariantGH} allows us to solve the problem in Eq. \eqref{eq:newton} to obtain
the descent directions $(D_{\bThe},D_{\bPsi})$, ignoring the constraint on the trace ratio. 
However, if $(D_{\bThe},D_{\bPsi})$ is a solution to the problem in Eq. \eqref{eq:newton}, then $(D_{\bThe}-c\bI,D_{\bPsi}+c\bI)$ 
for $c\in \mathbb{R}$ is also a solution, forming an equivalence class 
\begin{align} \label{eq:equivD}
S[D_{\bThe} \oplus D_{\bPsi}]= 
    \{ (D_{\bThe},D_{\bPsi})| 
	\left(D_{\bThe}-c\bm{I}_p\right)\oplus\left(D_{\bPsi}+c\bm{I}_q\right)= D_{\bThe} \oplus D_{\bPsi}, 
		c\in \mathbb{R}\}.
\end{align}

\begin{lemma}\label{lem:newtonnonid}
Within the equivalence class in Eq. \eqref{eq:equivD}, the line-search method 
in Algorithm \ref{alg:linesearch} identifies a unique step size $\alpha$.
\end{lemma}
\begin{proof}
We only need to show that $\tvec(D)^T\tvec(\bG)$, a term that appears in the computation of $\delta$ 
in Algorithm \ref{alg:linesearch}, is invariant within the equivalence class in Eq. \eqref{eq:equivD}, 
since $D_{\bThe}$ and $D_{\bPsi}$ always appear as Kronecker sum of the two in all the other parts 
of the line search such that $c$ and $-c$ in Eq. \eqref{eq:equivD} cancel out.
We have $\tr(\bG_{\bThe}) = \tr(\bG_{\bPsi})$, which
 can be directly verified from Eqs. \eqref{eq:partialmatrix0} and \eqref{eq:partialmatrixW} in Lemma \ref{lem:gh}
as $\tr(\bG_{\bThe}) = \sum_{i=1}^n\frac{1}{n}\tr(\bY^{iT}\bY^{i}) - \tr(\bW) = \tr(\bG_{\bPsi})$. 
From this, we have $\tvec(D)^T\tvec(\bG)=\tr(D_{\bThe}\bG_{\bThe}) + \tr(D_{\bPsi}\bG_{\bPsi})
= \tr((D_{\bThe}-c\bI)\bG_{\bThe}) + \tr((D_{\bPsi}+c\bI)\bG_{\bPsi})
= \tr(D_{\bThe}) + \tr(D_{\bPsi})$, which proves $\tvec(D)^T\tvec(\bG)$ is unaffected by the
unidentifiable diagonals of $D_{\bThe}$ and $D_{\bPsi}$.
\end{proof}

With Lemmas \ref{lem:invariantGH} and \ref{lem:newtonnonid},
we can minimize the EiGLasso objective in Eq. \eqref{eq:objrho} by updating $\bThe$ and $\bPsi$
with $(D_{\bThe}, D_{\bPsi})$ and step size $\alpha$ 
and adjusting the diagonals of $\bThe$ and $\bPsi$ using Theorem \ref{thm:uniquepair} after each update.
The theorem below shows this procedure can be simplified even further. 



\begin{theorem}\label{thm:adjustonce}
In EiGLasso, given the trace ratio $\rho=\tr(\bPsi)/\tr(\bThe)$, it is sufficient to identify 
the diagonals of $\bThe$ and $\bPsi$ only once after convergence. This leads to 
the identical estimate obtained by identifying the diagonals of $\bThe$ and $\bPsi$ 
in every iteration to maintain the trace ratio $\rho$. 
At convergence, the diagonal elements of $(\bThe,\bPsi)$ are adjusted by the scalar factor
\begin{align}\label{eq:c}
\bThe &\gets \bThe - \frac{\tr({\bPsi}) - \rho\tr({\bThe})}{q + \rho p} \bm{I}_p.
\end{align}
\end{theorem}
\begin{proof}
Regardless of which member of the equivalence class $S[D_{\bThe}\oplus D_{\bPsi}]$ is used to update 
the current estimate of $(\bThe, \bPsi)$, we arrive at the same equivalence class $S[\bThe\oplus \bPsi]$ 
for the estimate of $(\bThe, \bPsi)$ after the update. From Lemma \ref{lem:invariantGH},
given this equivalence class $S[\bThe\oplus \bPsi]$, the problem in Eq. \eqref{eq:equivD} 
in the next iteration is uniquely defined. Thus, regardless of whether we adjust the diagonals
to meet the constraint on the trace ratio, the sequence of equivalence class
$S[\bThe\oplus \bPsi]$ over iterations is the same, and it is not necessary to identify the diagonals 
in each iteration of Newton's method. The adjustment in Eq. \eqref{eq:c} is found 
by applying Theorem \ref{thm:uniquepair} to ${\bThe}\oplus {\bPsi}$ with the current estimate of $\bThe$
and $\bPsi$.
\end{proof}

When we set $\rho=\frac{q}{p}$, 
the one-time identification of diagonal parameters with Eq. \eqref{eq:c} in EiGLasso becomes 
identical to 
the identification of the parameters that TeraLasso~\citep{TeraLasso} 
performs at the end of every iteration.
At the end of each iteration, TeraLasso evenly distributes $\tau=\frac{\tr(\bThe\oplus\bPsi)}{pq}$
between the re-parameterized $\bThe$ and $\bPsi$ in Eq. \eqref{eq:reparam} 
 and performs the update
$\bThe \gets \bThe - \frac{\tr(\bThe)}{p}\bm{I}_p + \frac{\tau}{2}\bm{I}_p$
and similarly for $\bPsi$. It is straightforward to show
that with $\rho=\frac{q}{p}$ and from Eq. \eqref{eq:tromega}, Eq. \eqref{eq:c} reduces to this update in TeraLasso.
Theorem \ref{thm:adjustonce} is analogous to the simple approach to
handling the unidentifiability of the parameters in Kronecker-product inverse covariance estimation~\citep{YIN2012},
where $\bThe\otimes\bPsi = (c \bThe)\otimes (\frac{1}{c}\bPsi)$ for any positive constant $c$. 
The parameters are identified by rescaling $\bThe$ and $\bPsi$ 
as $\bThe \gets c\bThe$ and  $\bPsi \gets \frac{1}{c}\bPsi$ with some constant $c$ such that 
$\bThe_{11}$ is equal to 1 after convergence, similar to the one-time identification in EiGLasso.


\subsection{Active Set and Automatic Detection of Block-diagonal Structure}
\label{sec:block}

In the graphical lasso, a simple strategy for reducing computation time has been introduced that
detects the block-diagonal structure in the inverse covariance parameter from the sample covariance 
matrix prior to estimation~\citep{witten2011,JMLR:v13:mazumder12a}. Then,
only the parameters within the blocks corresponding to the connected components in the graph need to be estimated.
In Theorem \ref{thm:threshold} below, we show that a similar strategy can be applied to 
EiGLasso with both the exact and approximate Hessian,
to detect the block diagonal structures in $\bThe$ and $\bPsi$ from the sufficient statistics $\bm{S}$ and $\bm{T}$. 

\begin{theorem}\label{thm:threshold}
The block-diagonal structure of $\bThe$ can be detected by thresholding $\bm{S}$ such that
$[\bThe]_{ij}=0$ iff $|[\bm{S}]_{ij}|\le\gamma_{\bThe}$.
Similarly, the block-diagonal structure of $\bPsi$ can be detected by thresholding $\bm{T}$ 
such that $[\bPsi]_{ij}=0$ iff $|[\bm{T}]_{ij}|\le\gamma_{\bPsi}$.
\end{theorem}
\begin{proof}
Let $\partial |\cdot|$ denote the subgradient of the $L_1$ norm of a matrix, i.e., $[\partial|\bm{A}|]_{ij}$ is $1$ if $[\bm{A}]_{ij} > 0$, $-1$ if $[\bm{A}]_{ij} < 0$, and
$[\partial|\bm{A}|]_{ij} \in [-1,1]$ if $[\bm{A}]_{ij} = 0$. Then, the Karush-Kuhn-Tucker conditions \citep{cvxbook,witten2011} for $\bThe$ in Eq. \eqref{eq:obj} is
\begin{equation}\label{eq:kkttheta}
    \bW_{\bThe} - q\bm{S} - q\gamma_{\bThe}\partial|\bThe| = 0,
\end{equation}
where $\bW_{\bThe}$ is given in Eq. \eqref{eq:partialmatrixW}.
If $\bThe$ is block-diagonal, $\bW_{\bThe}$ is also block-diagonal, since $\bThe$ and $\bW_{\bThe}$ 
have the same eigenvectors according to Lemma \ref{lem:gh} and Theorem \ref{thm:gheigcomp}.
This implies that if $[\bThe]_{ij}=0$
in the off-diagonal blocks,
$|[\bm{S}]_{ij}|\le\gamma_{\bThe}$ in Eq. \eqref{eq:kkttheta}.
The case for $\bPsi$ can be proven similarly.
\end{proof}

QUIC further showed that their active-set strategy amounts to detecting a block-diagonal structure  
in the first iteration, if the parameters are initialized to a diagonal matrix. This strategy
can be extended to EiGLasso, when an approximate Hessian is used.
To reduce the computation time,
as in QUIC, in each Newton iteration, EiGLasso detects the active set of $\bThe$ and $\bPsi$ 
\begin{align*}
\mathcal{A}_{\bThe} =\{(i,j)\mid [\bThe]_{ij}\neq0 \text{ or } |[\bG_{\bThe}]_{ij}| 
    > q\gamma_{\bThe} \}, \;\;
\mathcal{A}_{\bPsi} =\{(i,j)\mid [\bPsi]_{ij}\neq0 \text{ or } |[\bG_{\bPsi}]_{ij}| 
    > p\gamma_{\bPsi}\},
\end{align*}
and update only the parameters in the active sets during the coordinate descent optimization, 
while setting those in the fixed set to zero.
When $\bThe$ and $\bPsi$ are initialized to diagonal matrices, the approximate Hessian $\hat{\bH}^1$ 
in iteration 1 is diagonal, since the eigenvector matrices $\bm{Q}^1_{\bThe}$ and $\bm{Q}^1_{\bPsi}$ 
are diagonal. Then, the optimization problem in Eq. \eqref{eq:newton} 
decouples into 
a set of optimization problems, each of which involves a single element of $(D^1_{\bThe},D^1_{\bPsi})$
and has a closed-form solution for $[D^1_{\bThe}]_{ij}$ as
\begin{equation*}
    [D^1_{\bThe}]_{ij} = -q\left( \sum_{k=1}^K[\bm{V}^1_{\bThe, k}]_{ii}[\bm{V}^1_{\bThe, k}]_{jj} 
	+ (q-K)[\bm{V}^1_{\bThe, K}]_{ii}[\bm{V}^1_{\bThe, K}]_{jj}\right)^{-1}\mathcal{S}\left([\bm{S}]_{ij},
	\gamma_{\bThe}\right),\quad\forall i\neq j.
\end{equation*}
The soft-thresholding 
operator $\mathcal{S}(a,b)$ above is defined as $\mathcal{S}(a,b)=\textrm{sign}(a)(|a|-b)_+$,  
where $(c)_+ = 0$ if $c<0$ and $(c)_+ = c$ if $c>0$.
This closed-form solution is $[D^1_{\bThe}]_{ij}=0$ if 
$|[\bm{S}]_{ij}| < \gamma_{\bThe}$, which is equivalent to the condition for detecting 
the block-diagonal structure in $\bThe$ from $\bm{S}$ in Theorem \ref{thm:threshold}. 
The case for $\bPsi$ can be shown similarly.

\section{Convergence Analysis}\label{sec:convg}

We examine the properties of the line-search method in Algorithm \ref{alg:linesearch}
and analyze the global and local convergence of EiGLasso.

\subsection{Line Search Properties}

EiGLasso inherits some of the line-search properties shown for QUIC \citep{quic}.
This is because our objective in Eq. \eqref{eq:obj} can be written in terms of 
$\bOmega=\bThe\oplus\bPsi$ in the form that resembles the objective of QUIC
\begin{equation*}\label{eq:objomega}
    f(\bOmega) =
    \tr(\bm{U} \bOmega)
    - \log|\bOmega| + \| \bm{\Gamma}\circ\bOmega \|_{1,\text{off}}, 
\end{equation*}
where $\bm{U}=\frac{1}{n}\sum_{i=1}^n \tvec(\bY^i)^T\tvec(\bY^i)$, $\bm{\Gamma}=\gamma_{\bThe}\mathbbm{1}_p\oplus\gamma_{\bPsi}\mathbbm{1}_q$ with an $r\times r$ matrix of one's $\mathbbm{1}_r$, and $\circ$ is an element-wise multiplication operator. 
Then, the EiGLasso's update
$\bThe \gets \bThe + \alpha D_{\bThe}$ and $\bThe \gets \bThe + \alpha D_{\bThe}$ can be written as 
$\bOmega \gets \bOmega + \alpha D_{\bOmega}$, where $D_{\bOmega} = D_{\bThe} \oplus D_{\bPsi}$, since $\bOmega + \alpha D_{\bOmega} 
    =(\bThe+\alpha D_{\bThe})\oplus(\bPsi+\alpha D_{\bPsi})$.
We extend these results for $\bOmega$, to prove the results for 
individual $\bThe$ and $\bPsi$, where $(\bThe, \bPsi) \in \mathbb{KS}^{p,q}$, 
when  unlike in QUIC the exact Hessian $\bH$ is not positive definite everywhere (Theorem \ref{thm:Hpd}), 
and when the approximate Hessian $\hat{\bH}$ is used.

We begin by showing that both $\bH$ and $\hat{\bH}$ have bounded eigenvalues
in the level set
$\mathcal{U}=\{(\bThe,\bPsi)\mid \bThe\oplus\bPsi\in\mathbb{S}^{pq}_{++},\;\;f(\bThe,\bPsi)
\le f(\bThe^0,\bPsi^0)\}$
and are Lipschitz-continuous. This result will be used to prove the line-search properties and global and local convergence.   

\begin{lemma}\label{lemma:eigvalbound}
In the level set $\mathcal{U}$,
$\bH$ and $\hat{\bH}$ are Lipschitz continuous and have bounded eigenvalues:
\begin{align*}
    \min\{p,q\}\bar{\lambda}^{-2}\bI_{p^2+q^2} \preceq
    & \bH\preceq  (p+q)\underline{\lambda}^{-2}\bI_{p^2+q^2},
    \nonumber \\
    \min\{p,q\}\bar{\lambda}^{-2}\bI_{p^2+q^2} \preceq 
    & \hat{\bH} \preceq  \max\{p,q\}\underline{\lambda}^{-2}\bI_{p^2+q^2},
\end{align*}
for some constants $\underline{\lambda},\overline{\lambda}>0$ that depend on $\gamma_{\bThe}$, $\gamma_{\bPsi}$, $f(\bThe^0,\bPsi^0)$, and $\{\bm{Y}^1,\ldots, \bm{Y}^n\}$.
\end{lemma}
\begin{proof}
We bound the eigenvalues of $\bOmega=\bThe\oplus \bPsi$, $\bThe$, and $\bPsi$, and
use these bounds to bound the eigenvalues of $\bH$ and $\hat{\bH}$.
When the diagonals are not identified,
it directly follows from Lemma 2 in \citet{quic}
that all EiGLasso iterates of $(\bThe,\bPsi)$ 
are contained in the set with bounded eigenvalues of $\bThe\oplus \bPsi$  
    \begin{equation}\label{eq:compactset_omega}
        \mathcal{C}=\{(\bThe,\bPsi)\mid \underline{\lambda}\bI\preceq\bThe\oplus\bPsi\preceq\bar{\lambda}\bI\}.
    \end{equation}
Next, given a fixed trace ratio $\rho=\frac{\tr(\bPsi)}{\tr(\bThe)}$,  
we bound the eigenvalues of $\bThe$ and $\bPsi$.
Since Eq. \eqref{eq:compactset_omega} implies 
$\underline{\lambda}\bI\preceq\bLambda_{\bThe}\oplus\bLambda_{\bPsi}\preceq\bar{\lambda}\bI$,
we apply Theorem \ref{thm:uniquepair} to $\bLambda_{\bOmega}=\bLambda_{\bThe}\oplus\bLambda_{\bPsi}$ 
to identify $\bLambda_{\bThe}$ and
$\bLambda_{\bPsi}$ in the equivalence class in Eq. \eqref{eqn:noniddiagEig} given $\rho$, thus, identifying
the diagonals of $\bThe$ and $\bPsi$. 
We bound each of the two terms in Eq. \eqref{eq:diagmapping} for $\bThe$ as
\begin{align*}
    q\underline{\lambda}\bI\preceq \sum_{i=1}^q(\bI_p\otimes\bm{e}_{q,i})^T\bLambda_{\bOmega}(\bI_p\otimes\bm{e}_{q,i})\preceq q\bar{\lambda}\bI \quad\textrm{and}\quad
    \frac{\rho pq}{q+\rho p}\underline{\lambda}\bI\preceq\frac{\rho}{q+\rho p}\tr(\bLambda_{\bOmega})\bI\preceq\frac{\rho pq}{q+\rho p}\bar{\lambda}\bI,
\end{align*}
and combine these to obtain the eigenvalue bounds for $\bThe$ and similarly for $\bPsi$
    \begin{align}
        \mathcal{C}_{\rho} = 
        \left\{(\bThe,\bPsi)\mid
        \Big(\underline{\lambda}-\frac{\rho p}{q + \rho p}\bar{\lambda}\Big)\bI\preceq\bThe\preceq\Big(\bar{\lambda}-\frac{\rho p}{q + \rho p}\underline{\lambda}\Big)\bI, 
        \quad\quad\quad\quad 
        \right.\nonumber \\
        \left.\Big(\underline{\lambda}-\frac{q}{q + \rho p}\bar{\lambda}\Big)\bI
        \preceq\bPsi\preceq\Big(\bar{\lambda}-\frac{q}{q + \rho p}\underline{\lambda}\Big)\bI\right\}. \label{eq:compactset}
    \end{align}
    
From Eq. \eqref{eq:compactset_omega} and Theorem \ref{thm:Hpd}, 
we obtain the bound on the eigenvalues
of $\bH$. From Eq. \eqref{eq:compactset_omega} and Theorem \ref{thm:Happroxpd}, 
we obtain the bound on the eigenvalues
of $\hat{\bH}$, 
since $\lambda_{\hat{\bH}_{\bThe},\min} \ge q\bar{\lambda}^{-2}$,
    $\lambda_{\hat{\bH}_{\bThe},\max} \le q\underline{\lambda}^{-2}$,
    $\lambda_{\hat{\bH}_{\bPsi},\min} \ge p\bar{\lambda}^{-2}$,
    $\lambda_{\hat{\bH}_{\bPsi},\max} \le p\underline{\lambda}^{-2}$.
On the set in Eq. \eqref{eq:compactset},
since the log-determinant is a continuous function of class $C^\infty$, and a continuous function on a compact set is bounded, both $\bH$ and $\hat{\bH}$ are locally Lipschitz-continuous. 
\end{proof}

Now, for EiGLasso with exact and approximate Hessian,
we show that the following three line-search properties hold:
the line search method is guaranteed to terminate  for any symmetric matrices $D_{\bThe}$ and $D_{\bPsi}$,
as the two line search conditions in Algorithm \ref{alg:linesearch} 
are satisfied for some step size $\alpha$ (Lemma \ref{lem:term}); 
the update with the Newton direction is guaranteed to decrease the objective
(Lemma \ref{lem:delta_exact}); and EiGLasso with the exact Hessian
is guaranteed to enter 
pure-Newton phase where the step size $\alpha=1$ is always chosen.
We state and prove the first two properties.
The proof of the last property follows directly from the proof in \citet{Tseng2009} and \citet{quic}.

\begin{lemma}\label{lem:term}
For any $\bThe$ and $\bPsi$, where $\bThe\oplus\bPsi\succ0$, and symmetric matrices $D_{\bThe}$ and $D_{\bPsi}$ 
for descent directions found with either the exact or inexact Hessian, there exists a step size $\tilde{\alpha}\in(0,1]$ 
such that for all $\alpha<\tilde{\alpha}$ the two conditions in the line search 
in Algorithm \ref{alg:linesearch} are satisfied.
\end{lemma}
\begin{proof}
If $\alpha < \lambda_{\min}(\bThe\oplus\bPsi)/\|D_{\bThe}\oplus D_{\bPsi}\|_2$, 
the updated estimates $\bThe + \alpha D_{\bThe}$ and $\bPsi + \alpha D_{\bPsi}$
satisfy $(\bThe + \alpha D_{\bThe})\oplus(\bPsi + \alpha D_{\bPsi}) \succ 0$,
since 
$(\bThe + \alpha D_{\bThe})\oplus(\bPsi + \alpha D_{\bPsi})=\bThe\oplus\bPsi + \alpha(D_{\bThe} \oplus D_{\bPsi})$ 
and we have $\bThe\oplus\bPsi\succ 0$ and $\|\alpha (D_{\bThe} \oplus D_{\bPsi})\|_2  < \lambda_{\min}(\bThe\oplus\bPsi)$. 
Thus, the first line-search condition in Algorithm \ref{alg:linesearch} is satisfied.
From Lemma 1 in \citet{Tseng2009} and Proposition 3 in \citet{quic}, it is straightforward
to show the second condition in Algorithm \ref{alg:linesearch} is satisfied.
\end{proof}

\begin{lemma}\label{lem:delta_exact}
Let $\tvec(D)=[\tvec(D_{\bThe})^T,\tvec(D_{\bPsi})^T]^T$ for all symmetric $D_{\bThe}$ and $D_{\bPsi}$.
With the exact Hessian $\bH$, 
$\delta$ in Algorithm \ref{alg:linesearch} is upper bounded if not at the optimum,
\begin{align}\label{eq:deltaexactbound}
\delta \le -\tvec(D)^T\bH\tvec(D) &\le -\lambda_{\bH,\min_0}\|\tvec(D)\|_2^2< 0, 
\end{align} 
where $\|\tvec(D)\|_2^2 = \|D_{\bThe}\|_F^2 + \|D_{\bPsi}\|_F^2$, and 
$\lambda_{\bH,\min_0}$ is given in Theorem \ref{thm:Hpd}.
With the approximate Hessian $\hat{\bH}$, 
$\delta$ is upper bounded everywhere 
\begin{align} 
\delta \le -\tvec(D)^T\hat{\bH}\tvec(D)   & \le   -\lambda_{\hat{\bH},\min}\|\tvec({D})\|_2^2  <  0, \label{eq:delta_approx}
\end{align} 
where $\lambda_{\hat{\bH},\min}$ is given in Theorem \ref{thm:Happroxpd}.
\end{lemma}
\begin{proof}
    For $\bH$, the first inequality in Eq. \eqref{eq:deltaexactbound} can be shown by a  
    straightforward application of Lemma 1 and Theorem 1 in \citet{Tseng2009} and Proposition 4 
    in \citet{quic}. To prove the second and third inequalities, since our $\bH$ is not 
    positive definite everywhere, we need to show that $\tvec(D)$ is outside of the nullspace of $\bH$ described in Theorem \ref{thm:Hpd}, unless EiGLasso is at the optimum, where $D={0}$.
 The null space of $\bH$, $\{\tvec({D})|\tvec({D}_{\bThe}\oplus {D}_{\bPsi})={0}\}$, 
 is equivalent to $\{\tvec(D)|D_{\bThe} = c\bI, D_{\bPsi} = -c\bI \textrm{ for } c\in \mathbb{R}\}$, 
 which is the equivalence class of the optimality condition $D={0}$.
     This proves the third inequality in Eq. \eqref{eq:deltaexactbound} that holds except when $D=0$.
     For $\hat{\bH}$, the first inequality in Eq. \eqref{eq:delta_approx} can be again shown 
     from Lemma 1 and Theorem 1 in \citet{Tseng2009} and Proposition 4 in \citet{quic}.
    The second and third inequalities hold since $\hat{\bH}$ is positive definite. 
\end{proof}


\subsection{Convergence Analysis}


To show the global convergence of EiGLasso, as in QUIC, we use a more general non-smooth optimization framework, 
the block coordinate descent studied in \citet{Tseng2009}.
EiGLasso satisfies the following two conditions required 
to guarantee that the block coordinate descent algorithm converges to the global optimum.
First, the objective function of EiGLasso has a bounded positive definite 
exact or approximate Hessian in the level set $\mathcal{U}$: 
$a\bI \preceq \bH, \hat{\bH} \preceq b\bI$  for some positive constants $a,b\in\mathbb{R}^+$, 
according to Theorems \ref{thm:Hpd} and \ref{thm:Happroxpd}, and Lemma \ref{lemma:eigvalbound}. 
Second, EiGLasso with the exact or approximate Hessian chooses a subset of variables to be updated in each iteration 
according to the \textit{Gauss-Seidel} rule: with $T=2$, at iteration $t$, 
EiGLasso updates one of the two subsets of variables,
$\mathcal{J}^{2t}=\mathcal{A}^{t}_{\bThe}\sqcup\mathcal{A}^{t}_{\bPsi}$ and
$\mathcal{J}^{2t+1}=\mathcal{J}\setminus\mathcal{A}^{t}_{\bThe}\sqcup\mathcal{A}^{t}_{\bPsi}$,    
where $\mathcal{J}=\mathcal{J}^{2t} \cup \mathcal{J}^{2t+1}$ is the entire set of variables 
and $\mathcal{A}^{t}_{\bThe}$ and $\mathcal{A}^{t}_{\bPsi}$ are active sets in Algorithm \ref{alg:eiglasso}. 
Therefore, EiGLasso is guaranteed to converge to the global optimum according to \citet{Tseng2009}.


Now we analyze the local convergence of EiGLasso.
We adopt a similar strategy used in QUIC \citep{quic}: 
convergence analysis on a smooth function is applied near the global optimum,  
where the $L_1$-regularized non-smooth objective becomes locally smooth. 

\begin{theorem}\label{thm:convgrateexact}
Near the optimum, where step size $\alpha=1$ is chosen,
EiGlasso with the exact Hessian $\bH$
converges to the optimum  quadratically.
EiGLasso with the approximate Hessian $\hat{\bH}$
converges to the optimum linearly.
\end{theorem}
\begin{proof}
Since the exact Hessian $\bH$ is Lipschitz continuous from Lemma \ref{lemma:eigvalbound},
the proof for EiGLasso with the exact Hessian follows from Lemma 2.5 and Theorem 3.1 of
\citet{dunn1980}. 
While the convergence analysis in \citet{dunn1980} assumes that the Hessian is positive definite,
the exact Hessian in EiGLasso is not positive definite everywhere. 
However, according to Theorem \ref{thm:Hpd}, our Hessian 
is positive definite  for all iterates $(\bThe^t,\bPsi^t)\in \mathbb{KS}^{pq}_{++}$,
since  $\mathbb{KS}^{pq}_{++}$ is outside of the null space of $\bH$, so
the analysis in \citet{dunn1980} can be applied to EiGLasso.
The proof for EiGLasso with the approximate Hessian $\hat{\bH}$ follows from the analysis of steepest descent with the quadratic norm~\citep{cvxbook}, 
given the bounded eigenvalues of $\bH$ according to Lemma \ref{lemma:eigvalbound}.
\end{proof}

\section{Experiments} \label{sec:exp}

We compare the performance of EiGLasso with that of TeraLasso \citep{TeraLasso} on simulated data 
and on real-world data from genomics and finance. TeraLasso is the state-of-the art method
for Kronecker-sum inverse covariance estimation, and has been shown to be substantially
more efficient than BiGLasso~\citep{biglasso}, so we did not include BiGLasso in our experiments.
We implemented EiGLasso 
in C++ with the sequential version of Intel Math Kernel Library. 
We downloaded the authors' implementation of TeraLasso and modified it to perform more iterations during line search, 
when the safe-step approach suggested by the authors failed 
to find a step-size that satisfies the positive definite condition on $\bThe\oplus\bPsi$.
All experiments were run on a single core of Intel(R) Xeon(R) CPU E5-2630 v3 @ 2.40GHz.
In all of our experiments, we selected the regularization parameters 
$\gamma=\gamma_{\bThe}=\gamma_{\bPsi}$ for EiGLasso 
and used the selected $\gamma$ for TeraLasso,
because EiGLasso is significantly faster than TeraLasso 
and minimizes the same objective as TeraLasso.
To assess convergence, we used the criterion that the decrease in the objective function 
value $f^t$ at iteration $t$
satisfies the condition $\left|\frac{f^t-f^{t-1}}{f^{t}}\right| < \epsilon$ for three consecutive iterations.

\begin{figure*}[t!]
\centering
\setlength\tabcolsep{0pt}
\begin{tabular}{ccl}
     \begin{subfigure}{0.33\linewidth}
        \includegraphics[width=\linewidth]{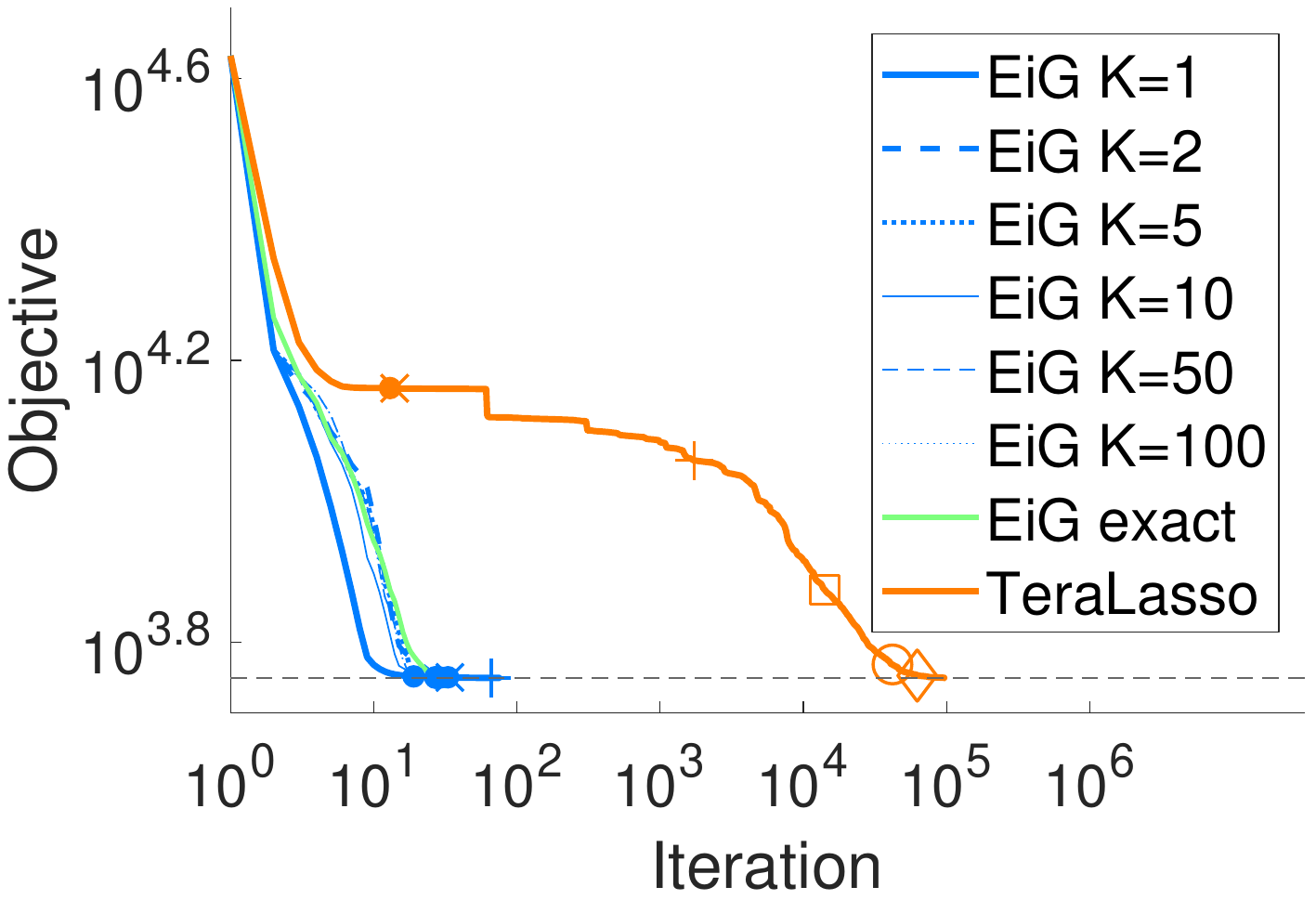}
        \caption*{(a)}
    \end{subfigure} &
    \begin{subfigure}{0.33\linewidth}
        \includegraphics[width=\linewidth]{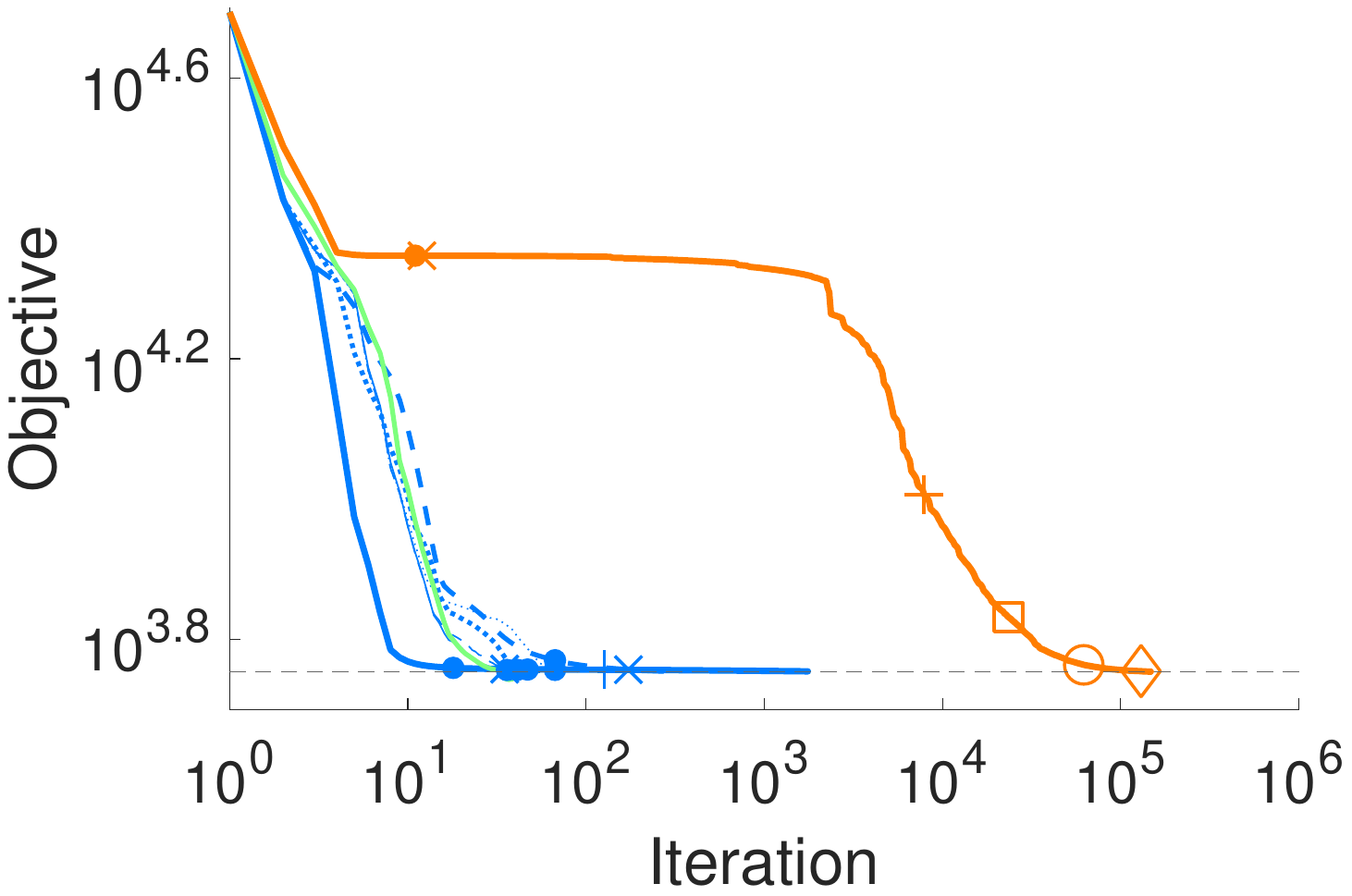}
        \caption*{(b)}
    \end{subfigure} &
    \multirow{2}{0.3\linewidth}[-1em]{
        \begin{subfigure}{\linewidth}
            \includegraphics[width=\linewidth]{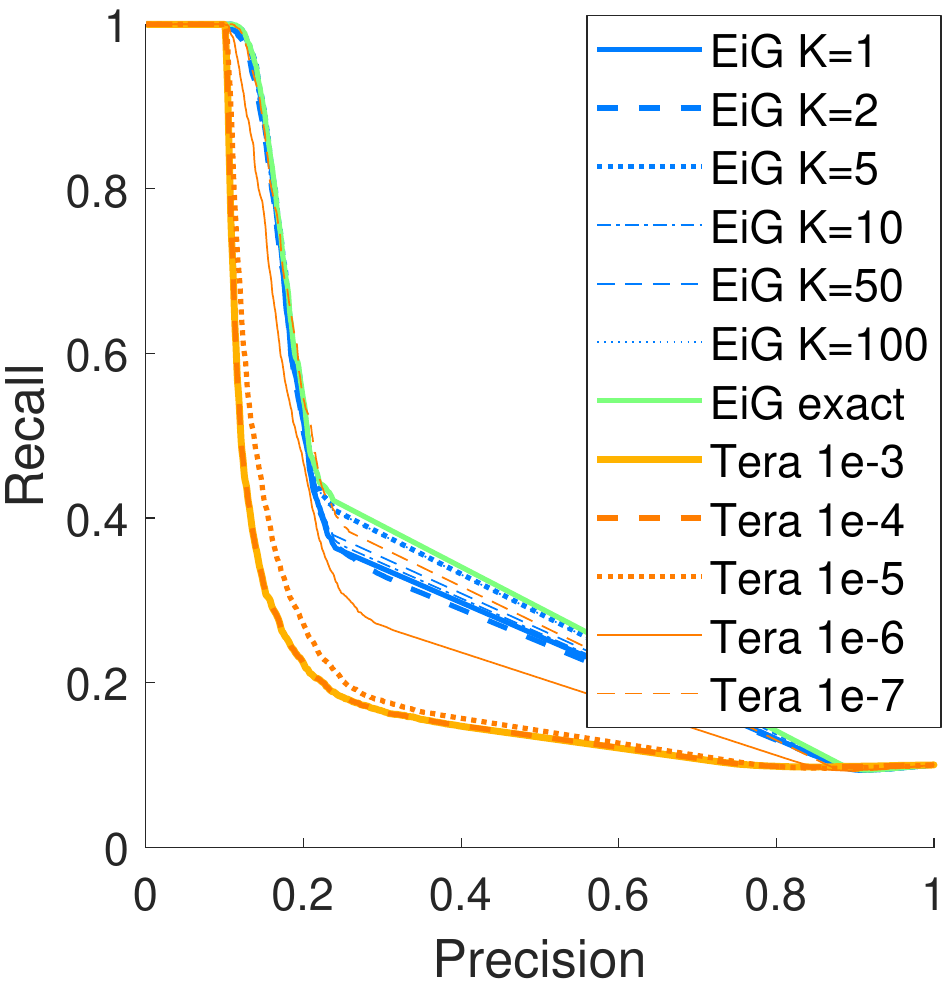}
            \caption*{(e)}
        \end{subfigure}
    }
    \\
    \begin{subfigure}{0.33\linewidth}
        \includegraphics[width=\linewidth]{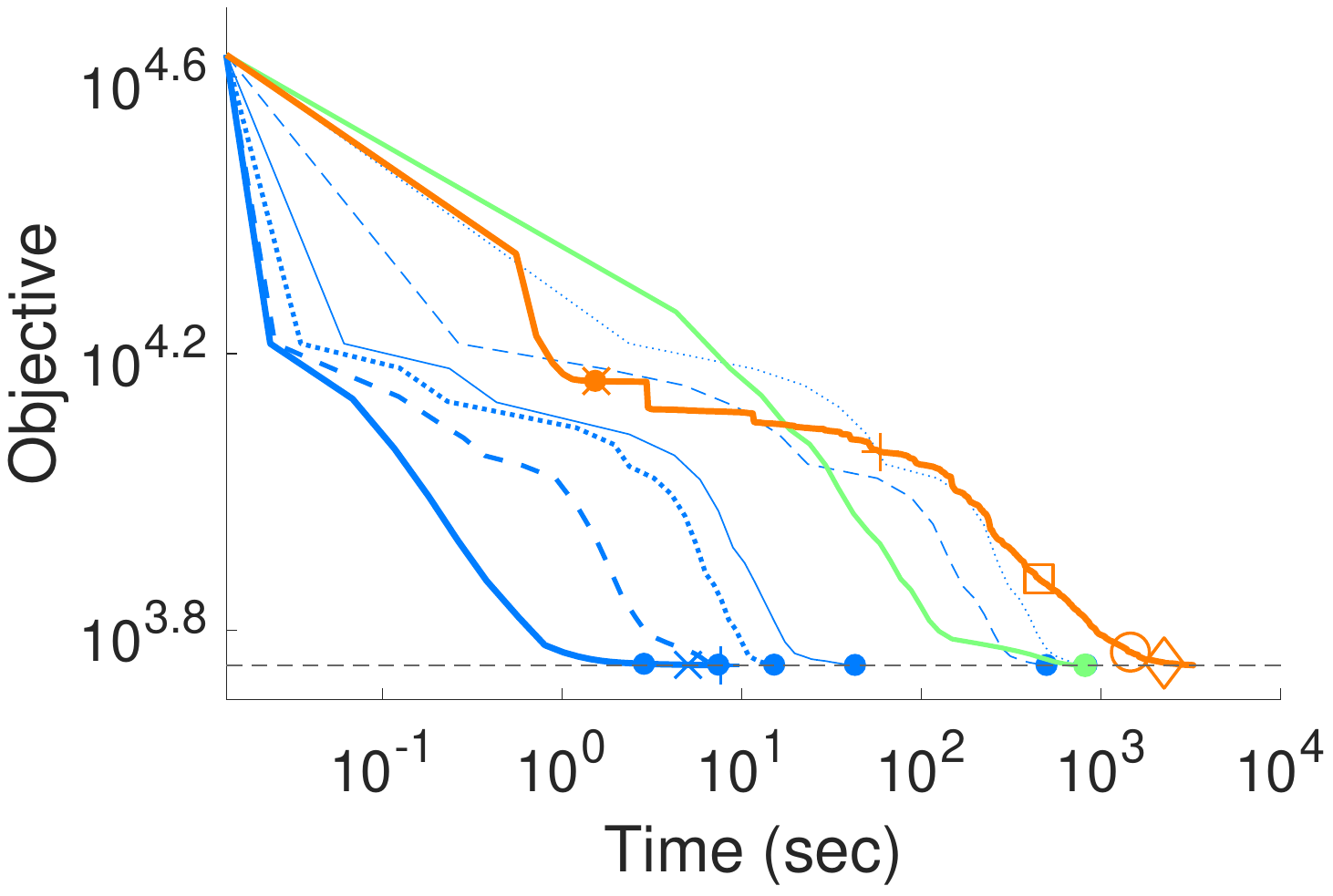}
        \caption*{(c)}
    \end{subfigure} &
    \begin{subfigure}{0.33\linewidth}
        \includegraphics[width=\linewidth]{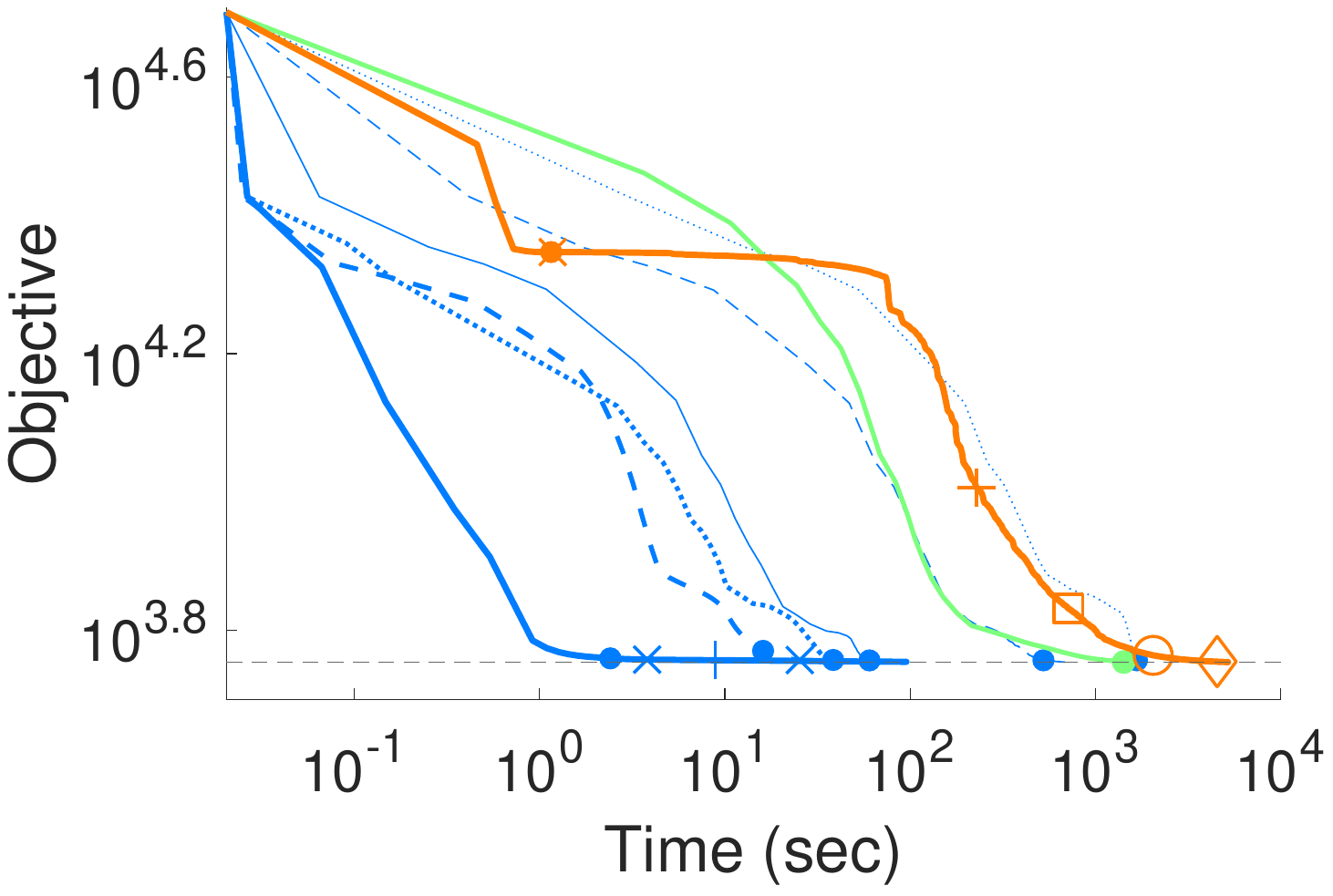}
        \caption*{(d)}
    \end{subfigure} &
\end{tabular}

    \caption{
Comparison of the convergence of EiGLasso and TeraLasso on data simulated from random graphs.
Objective values over iterations are shown for two datasets
in (a) and (b). Objective values over time are shown in (c) and (d) for the same two datasets. 
All methods were run until they reached the objective that EiGLasso with the exact Hessian reached with the convergence criterion $\epsilon=10^{-3}$.
The `$\bullet$', `$\times$', `+', `$\Box$', `$\bigcirc$' and `$\diamondsuit$' mark the points that satisfy 
the convergence criteria $\epsilon=10^{-3},\,10^{-4},\,10^{-5},\,10^{-6},\,10^{-7}$, and $10^{-8}$, respectively. 
Precision-recall curves for EiGLasso at $\epsilon=10^{-3}$ and TeraLasso at different $\epsilon$'s averaged over 10 datasets are shown in (e).
Graph size $p, q=100$ was used.
}
    \label{fig:100r}
\end{figure*}

\subsection{Simulated Data}\label{subsec:simulation}

We compared EiGLasso and TeraLasso on data simulated from the known $\bThe$ and $\bPsi$. 
We used the true $\bThe$ and $\bPsi$ of different sizes ($p,q=$ 100, 200, 500, 1000, 2000, and 5000),
assuming two types of graph structures. 
\begin{itemize}
\item
{\bf Random graph}: To set the ground-truth $\bThe$, 
we first generated a sparse $p\times p$ matrix $\bm{A}$  by assigning $-1$, 0, or 1 to each element with 
probabilities $\frac{1-\rho}{2}$, $\rho$, and $\frac{1-\rho}{2}$, respectively. We chose $\rho$ such that 
the number of non-zero elements of $\bThe$ is $10p$. 
To ensure $\bThe$ is positive definite, we set $\bThe$ to $\bm{A}\bm{A}^T$ after adding $\sigma + 10^{-4}$ with $\sigma\sim\text{Unif}(0,0.1)$
to each diagonal element of $\bm{A}\bm{A}^T$. 
\item{\bf Graph with clusters}: We set $\bThe$ to a block-diagonal matrix such that each block corresponds to a cluster. 
For graphs with $p=100$ and $200$, we assumed five blocks, each with size $\frac{p}{5}\times \frac{p}{5}$. 
For larger graphs with $p=500$, 1000, 2000, and $5000$, we assumed 10 blocks, each with size $\frac{p}{10}\times \frac{p}{10}$. 
Each block was generated as a random graph described above, setting $\rho$ so that we have $p$ nonzero elements in the block. 
\end{itemize}
The ground-truth $\bPsi$ was set similarly. Given these parameters, we simulated matrix-variate data 
from Gaussian distribution with mean zeros and inverse covariance $\bThe \oplus \bPsi$.

First, we evaluated TeraLasso and EiGLasso on simulated data with graph size $p, q=100$.
The two methods were compared in terms of computation time and the number of iterations required to reach 
the same level of optimality,  which we define as the objective value that
EiGLasso with the exact Hessian converged to with the convergence criterion $\epsilon=10^{-3}$. 
The regularization parameters were selected such that the number of non-zero elements in the estimated 
parameters roughly  matches that of the true parameters.
EiGLasso with approximate Hessian was run with different $K$'s ranging from 
$K=1$ to $K=100$.
Results are shown for four datasets, two from random graphs (Figures \ref{fig:100r}(a)-(d))
and two from graphs with clusters (Figures \ref{fig:100bd}(a)-(d)).
Regardless of the degree of Hessian approximation, EiGLasso required significantly 
fewer iterations than TeraLasso (Figures \ref{fig:100r}(a)-(b) and 
\ref{fig:100bd}(a)-(b)),
as expected since methods that use the second-order information in general converge in fewer iterations 
than the first-order methods. 
With the exact Hessian and approximate Hessian with large $K$, 
EiGLasso took longer in each iteration than TeraLasso, but overall required
two to three times less computation time than TeraLasso
(Figures \ref{fig:100r}(c)-(d) and \ref{fig:100bd}(c)-(d)).
As we reduced $K$, the time taken by EiGLasso 
decreased substantially, and EiGLasso with $K=1$ achieved two to three orders-of-magnitude 
speed-up compared to TeraLasso.

\begin{figure*}[t!]
\centering
\setlength\tabcolsep{0pt}
\begin{tabular}{ccl}
   \begin{subfigure}{0.33\linewidth}
        \includegraphics[width=\linewidth]{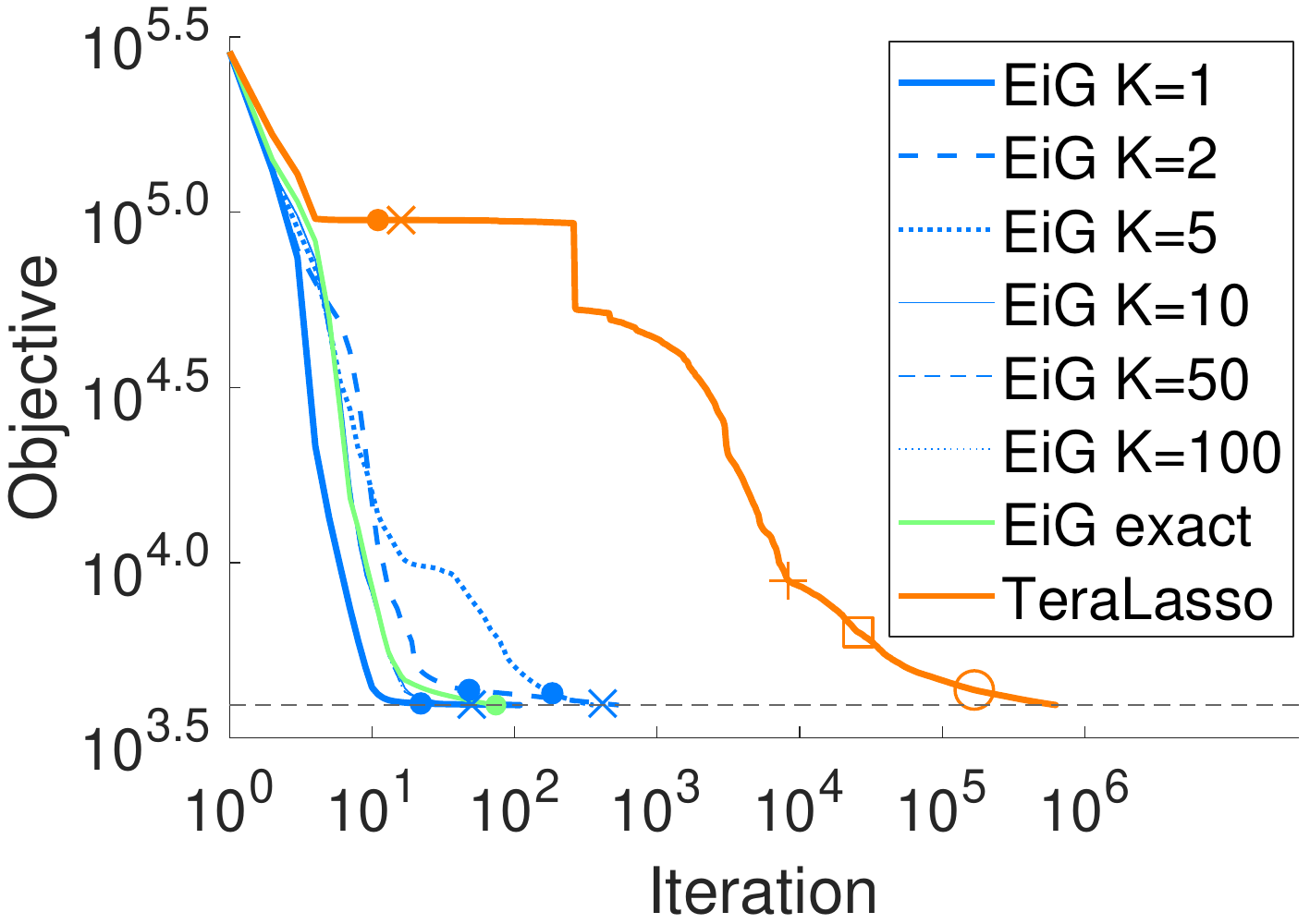}
        \caption*{(a)}
    \end{subfigure}
    &
    \begin{subfigure}{0.33\linewidth}
        \includegraphics[width=\linewidth]{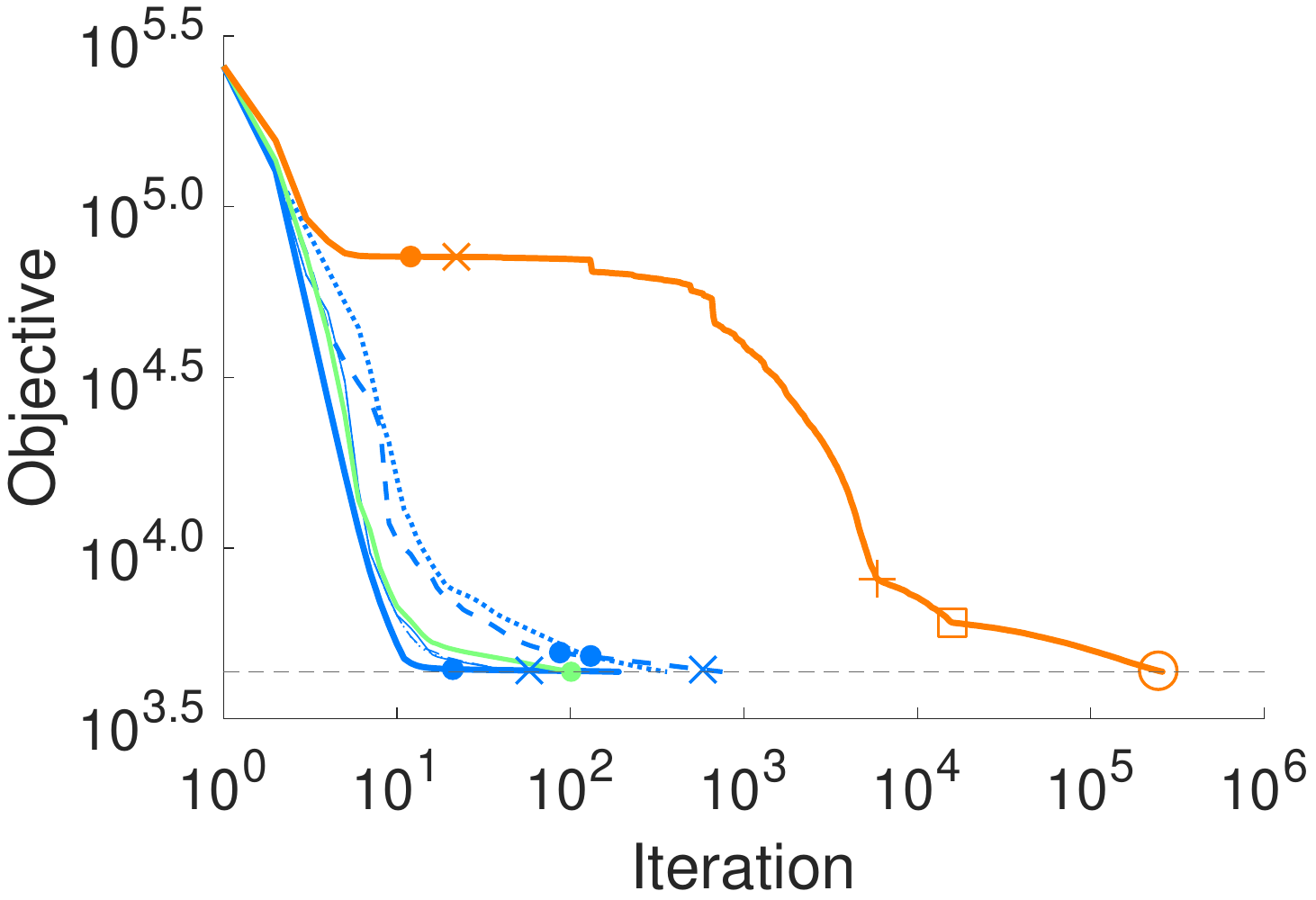}
        \caption*{(b)}
    \end{subfigure} &
    \multirow{2}{0.3\linewidth}[-1em]{
        \begin{subfigure}{\linewidth}
            \includegraphics[width=\linewidth]{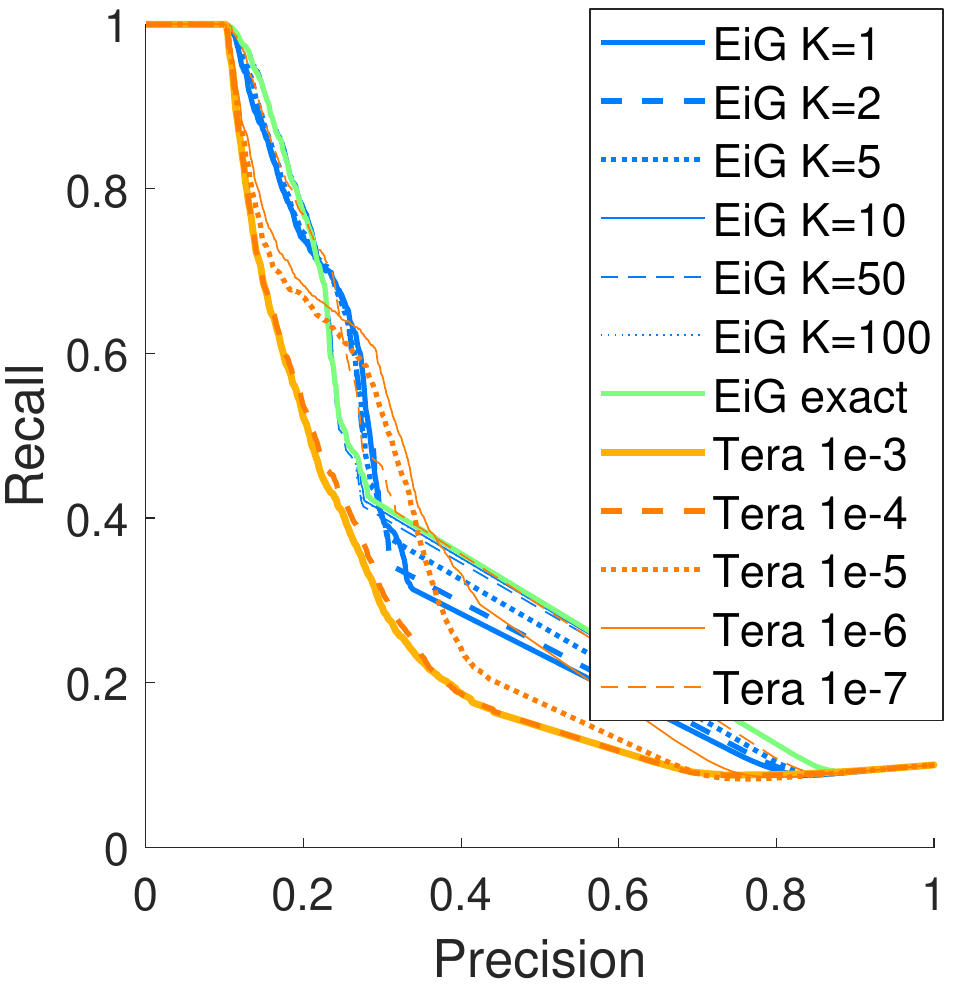}
            \caption*{(e)}
        \end{subfigure}
    }
    \\
   \begin{subfigure}{0.33\linewidth}
        \includegraphics[width=\linewidth]{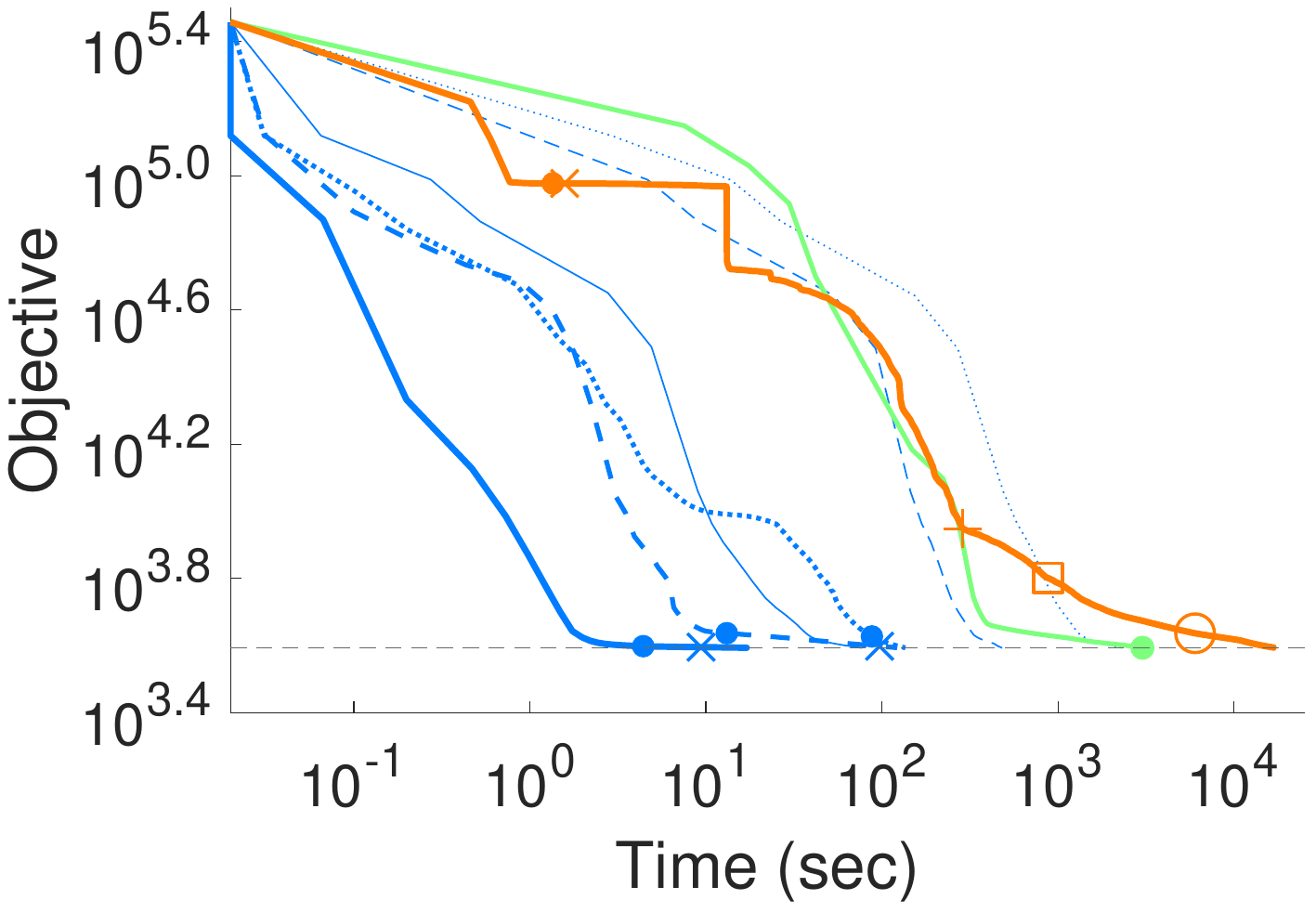}
        \caption*{(c)}
    \end{subfigure}
    &
    \begin{subfigure}{0.33\linewidth}
        \includegraphics[width=\linewidth]{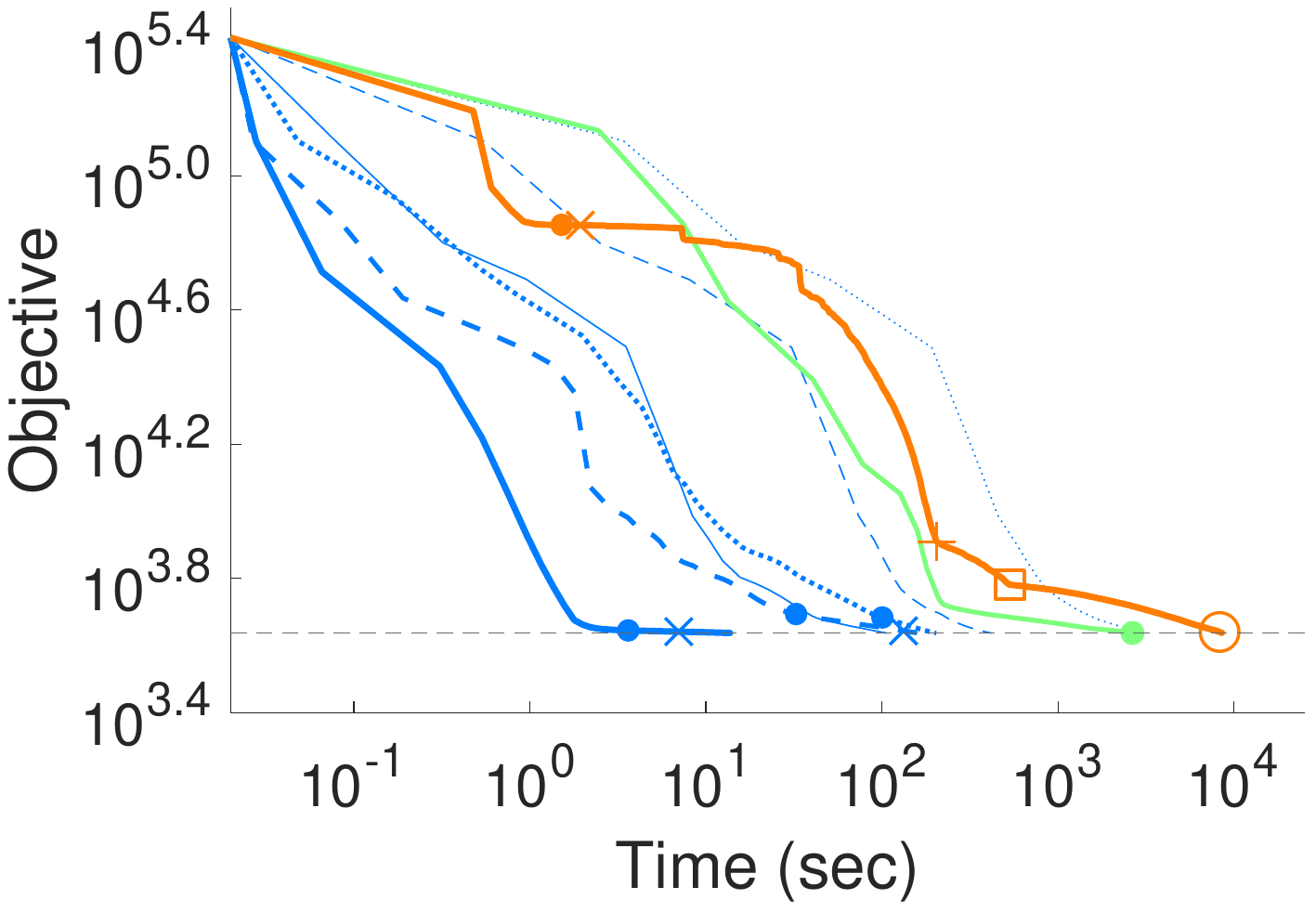}
        \caption*{(d)}
    \end{subfigure} &
\end{tabular}

    \caption{
Comparison of the convergence of EiGLasso and TeraLasso on data simulated from graphs with clusters.
Objective values over iterations are shown for two datasets
in (a) and (b). Objective values over time are shown in (c) and (d) for the same two datasets. 
All methods were run until they reached the objective that EiGLasso with the exact Hessian reached with the convergence criterion $\epsilon=10^{-3}$.
The `$\bullet$', `$\times$', `+', `$\Box$', and `$\bigcirc$' mark the points that satisfy 
$\epsilon=10^{-3},\,10^{-4},\,10^{-5},\,10^{-6},$ and $\,10^{-7}$, respectively. 
Precision-recall curves for EiGLasso at $\epsilon=10^{-3}$ and TeraLasso at different $\epsilon$'s averaged over 10 datasets are shown in (e).
Graph size $p, q=100$ was used.
}
    \label{fig:100bd}
\end{figure*}

\begin{table}[h!]
    \centering
    \begin{tabular}{|c|c|rrrrrr|r|}
    \hline
    \multicolumn{1}{|c|}{\multirow{2}{*}{Graph}} & \multicolumn{1}{c}{\multirow{2}{*}{$p,q$}} & \multicolumn{6}{|c|}{EiGLasso} & TeraLasso\\
    && $K=1$ & $K=2$ & $K=5$ & $K=10$ & $K=50$ & $K=100$ & \\
    \hline
    \multirow{6}{*}{\shortstack{Random\\Graphs}} & 100 & 1 & 4 & 6 & 40 & 729 & 1019 & 4194 \\
    & 200 & 7 & 164 & 551 & 3059 & 27049 & 70311 & 16765 \\
    & 500 & 29 & 155 & 2252 & 5994 & 31826 &&\\
    & 1000 & 515 & 11601 & 39568 &&&&\\
    & 2000 & 6361 & $28539$ &&&&&\\
    & 5000 & 10458 & $62685$ &&&&&\\
    \hline
    \multirow{6}{*}{\shortstack{
    Graphs\\with\\Clusters}} & 100 & 2 & 21 & 28 & 629 & 6516 & 12443 & 21115 \\
    & 200 & 11 & 410 & 930 & 1042 & 11281 & 59813 & $67669$ \\
    & 500 & 28 & 5231 & 9252 & 10549 &&&\\
    & 1000 & 2169 & 12614 & 58168 &&&&\\
    & 2000 & 9310 & $41524$ &&&&&\\
    & 5000 & 10461 & $67608$ &&&&&\\
    \hline
    \end{tabular}
    \caption{Comparison of computation time (in seconds) on simulated data. 
    }
    \label{tab:simresult}
\end{table}

We compared TeraLasso and EiGLasso on the accuracy of the recovered non-zero elements in $\bThe$ and $\bPsi$, 
using 10 datasets simulated as above.
The precision-recall curves averaged over the 10 simulated datasets 
(Figure \ref{fig:100r}(e) and \ref{fig:100bd}(e)) show that  
the TeraLasso estimates at the convergence criterion $\epsilon=10^{-3}$ and $10^{-4}$ are significantly inferior to
the EiGLasso estimates at the same $\epsilon$. TeraLasso needed $\epsilon=10^{-6}$ 
to achieve the accuracy similar to EiGLasso at $\epsilon=10^{-3}$,
whereas more stringent criteria such as $\epsilon=10^{-4}$ and $10^{-5}$ 
led to nearly identical results in EiGLasso.
Thus, in all experiments in the rest of Section \ref{sec:exp},
we used $\epsilon=10^{-3}$ as convergence criterion for EiGLasso with all $K$'s, 
but ran TeraLasso until it reached the similar objective value 
that EiGLasso reached with $\epsilon=10^{-3}$, 
which typically required $\epsilon = 10^{-6}$, $10^{-7},$ or $10^{-8}$.

\begin{figure*}[b!]
\centering
\setlength\tabcolsep{0pt}
\begin{tabular}{c@{}c@{}c}
    \begin{subfigure}{0.33\linewidth}
        \includegraphics[width=\linewidth]{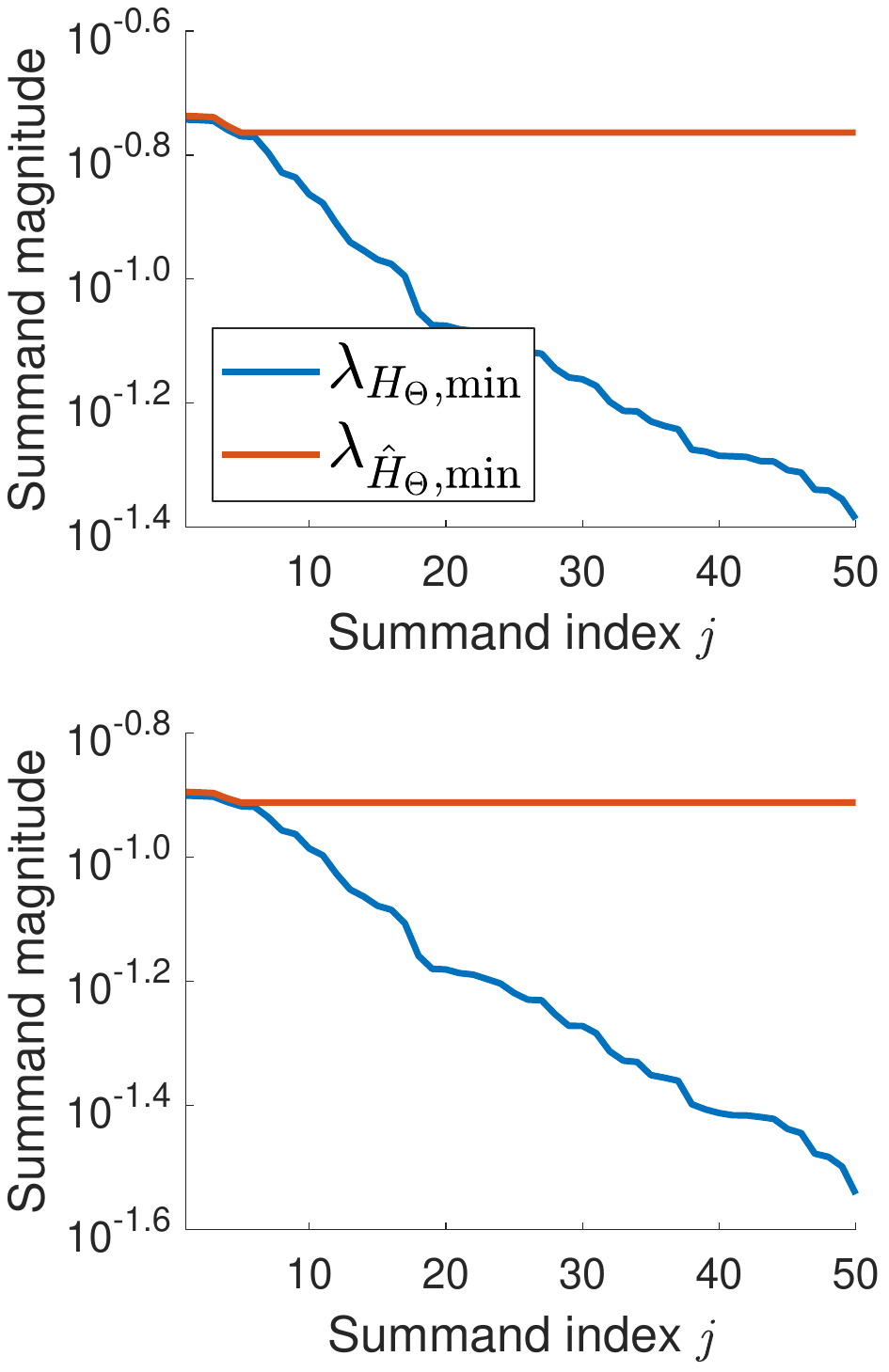}
        \caption{}
    \end{subfigure} 
    &
   \begin{subfigure}{0.33\linewidth}
        \includegraphics[width=\linewidth]{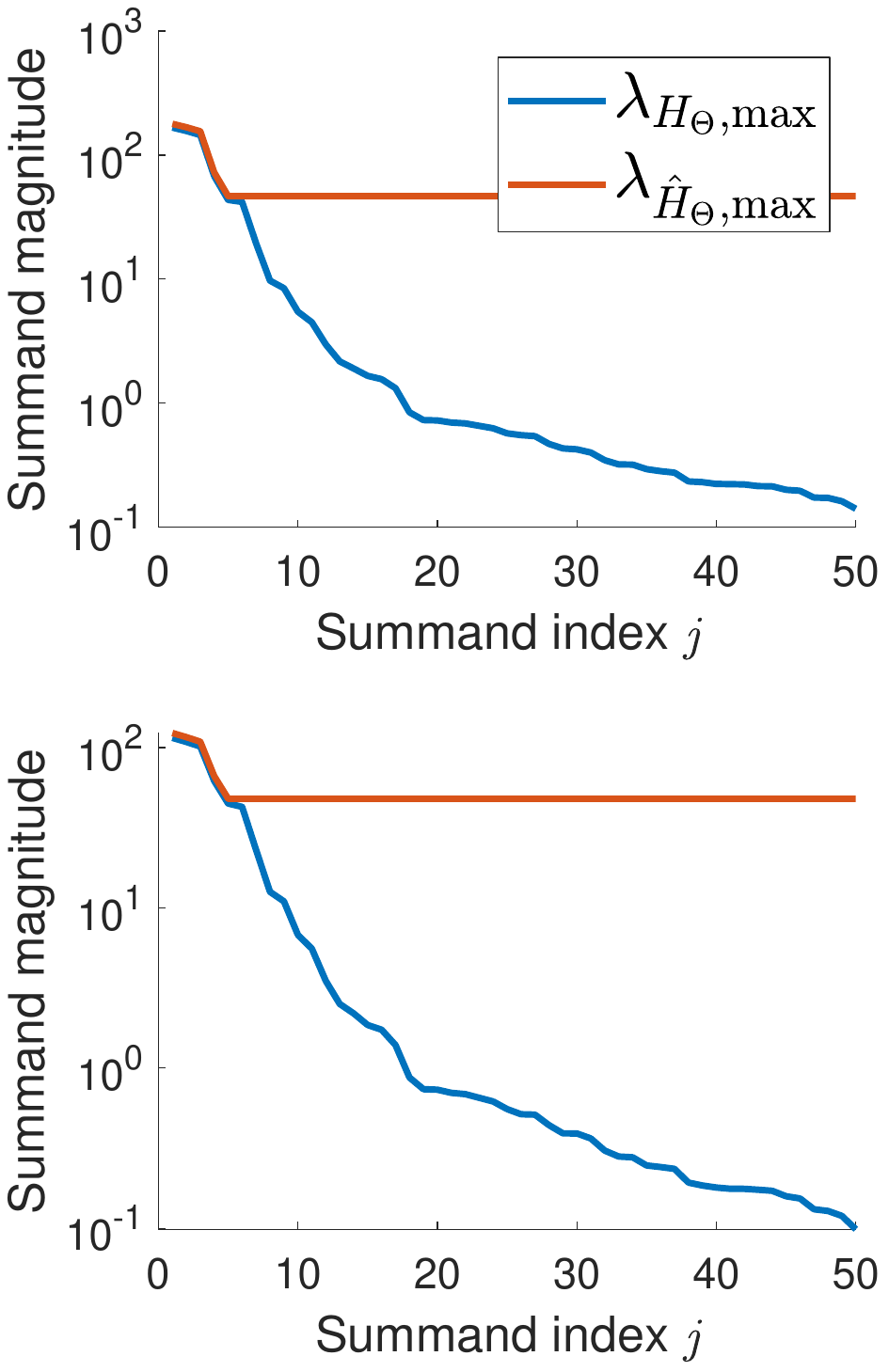}
        \caption{}
    \end{subfigure}
    &
    \begin{subfigure}{0.33\linewidth}
        \includegraphics[width=\linewidth]{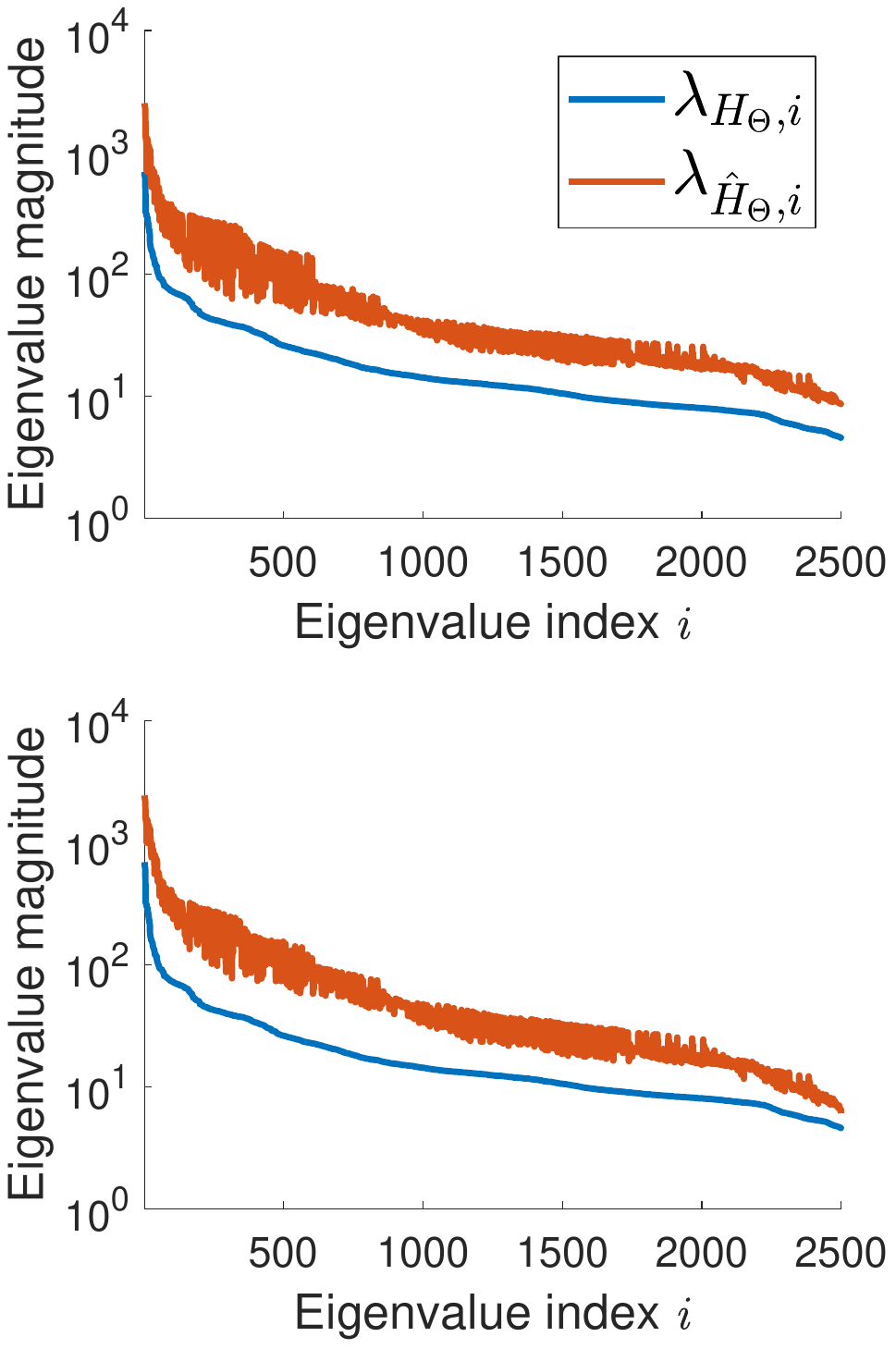}
        \caption{}
    \end{subfigure} 
    \\
    \begin{subfigure}{0.33\linewidth}
        \includegraphics[width=\linewidth]{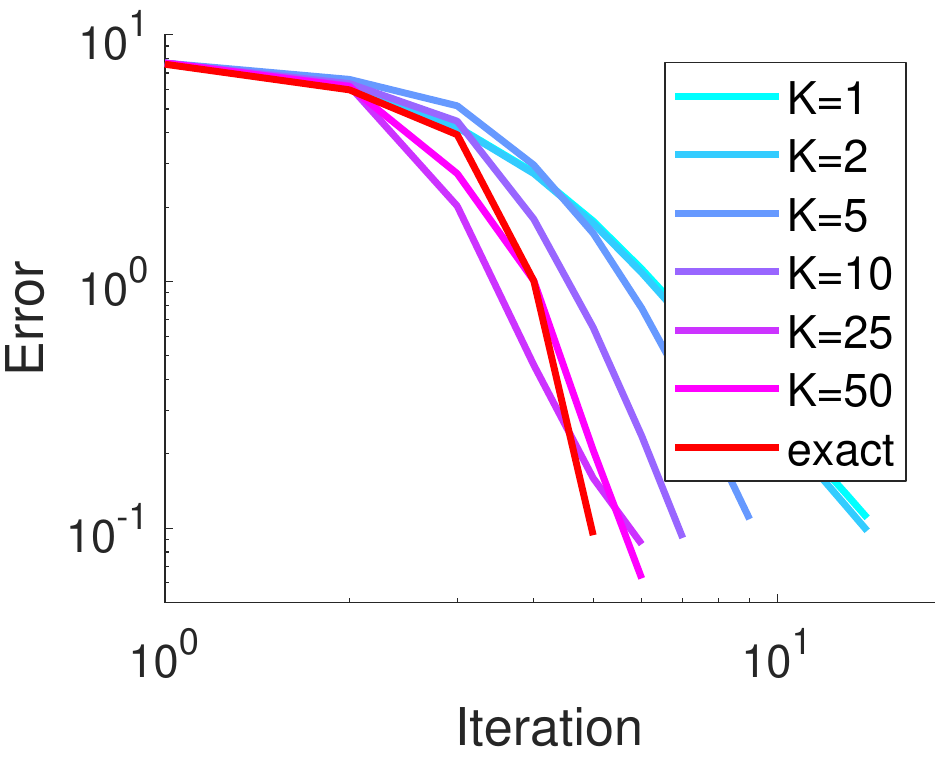}
        \caption{}
    \end{subfigure} 
    &
   \begin{subfigure}{0.33\linewidth}
        \includegraphics[width=\linewidth]{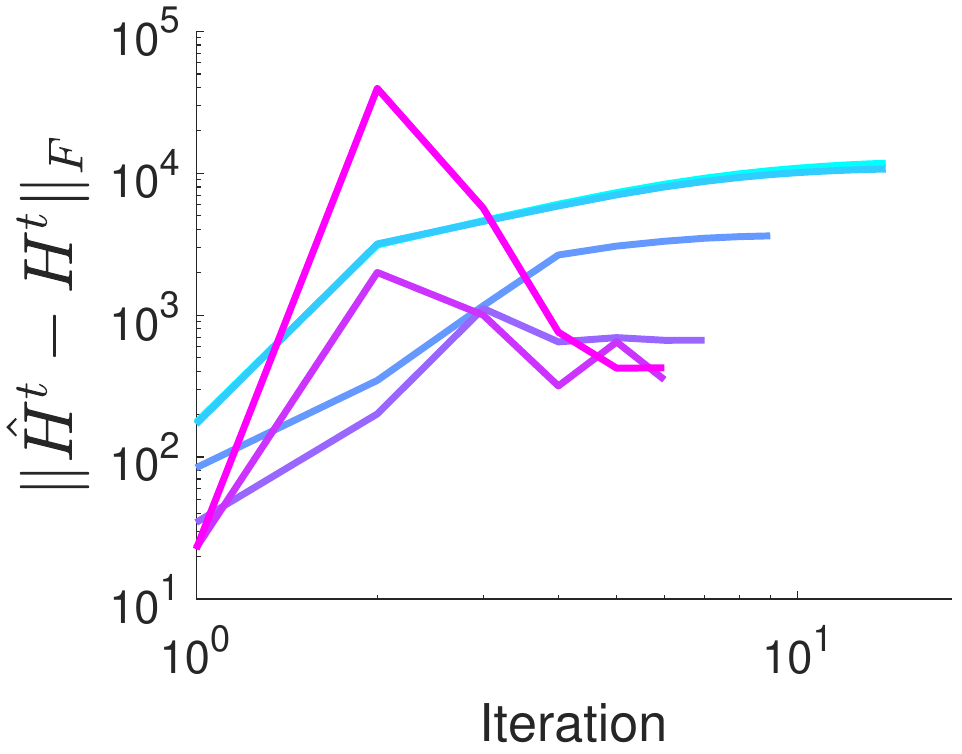}
        \caption{}
    \end{subfigure}
    &
\end{tabular}
\caption{
    Effects of the Hessian approximation on the convergence behavior of EiGLasso 
    for data simulated from a random graph. 
    (a) The summands in $\lambda_{\bH_{\bThe},\min}=\sum_{j=1}^q (\lambda_{\bThe, p} 
    + \lambda_{\bPsi,j})^{-2}$ and 
    $\lambda_{\hat{\bH}_{\bThe},\min}=\sum_{j=1}^K (\lambda_{\bThe, p} + \lambda_{\bPsi,j})^{-2} 
    + (q-K)(\lambda_{\bThe, p} + \lambda_{\bPsi,K})^{-2}$ for EiGLasso with $K=5$. 
    (b) The summands in $\lambda_{\bH_{\bThe},\max}$ and $\lambda_{\hat{\bH}_{\bThe},\max}$ 
    as in Panel (a).
    (c) The eigenvalues of $\bH_{\bThe}$ and $\hat{\bH}_{\bThe}$, sorted 
    with the eigenvalues of $\bH_{\bThe}$.
    In Panels (a)-(c), the results are shown 
    at the 5th iteration (top) and at the last iteration (bottom).
    (d) Error measured as 
   $\left\|\protect\begin{bmatrix}
    \tvec(\bThe^{t}-\bThe^*) \\
    \tvec(\bPsi^{t}-\bPsi^*)
    \protect\end{bmatrix}\right\|_2$ 
    over iterations.
    (e) $\|\hat{\bH}-\bH\|_F$ over iterations.
    A dataset simulated from a random graph with $p, q=50$ was used.
}
    \label{fig:50rndemp}
\end{figure*}

\begin{figure*}[b!]
\centering
\setlength\tabcolsep{0pt}
\begin{tabular}{c@{}c@{}c}
    \begin{subfigure}{0.33\linewidth}
        \includegraphics[width=\linewidth]{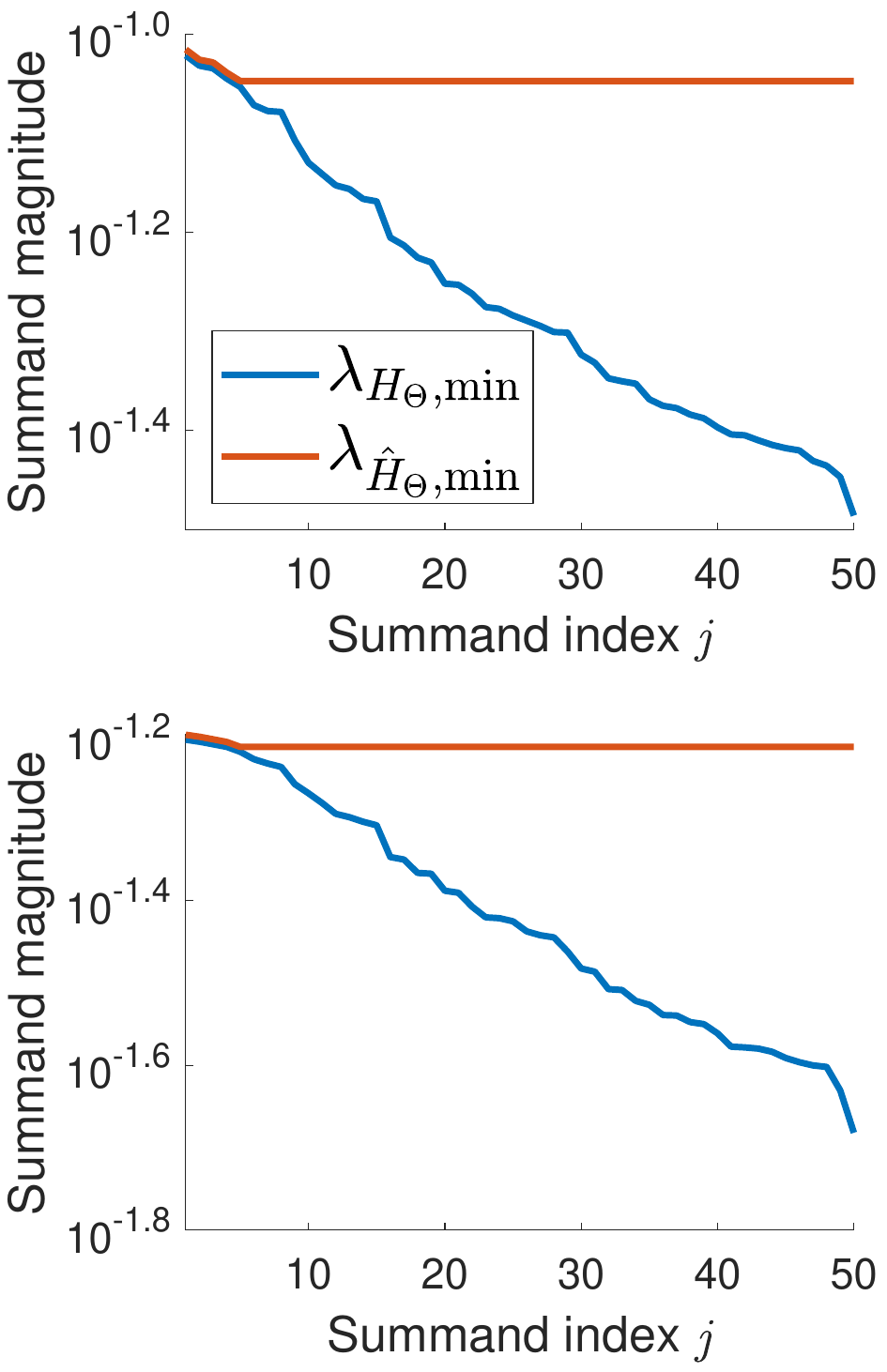}
        \caption{}
    \end{subfigure} 
    &
   \begin{subfigure}{0.33\linewidth}
        \includegraphics[width=\linewidth]{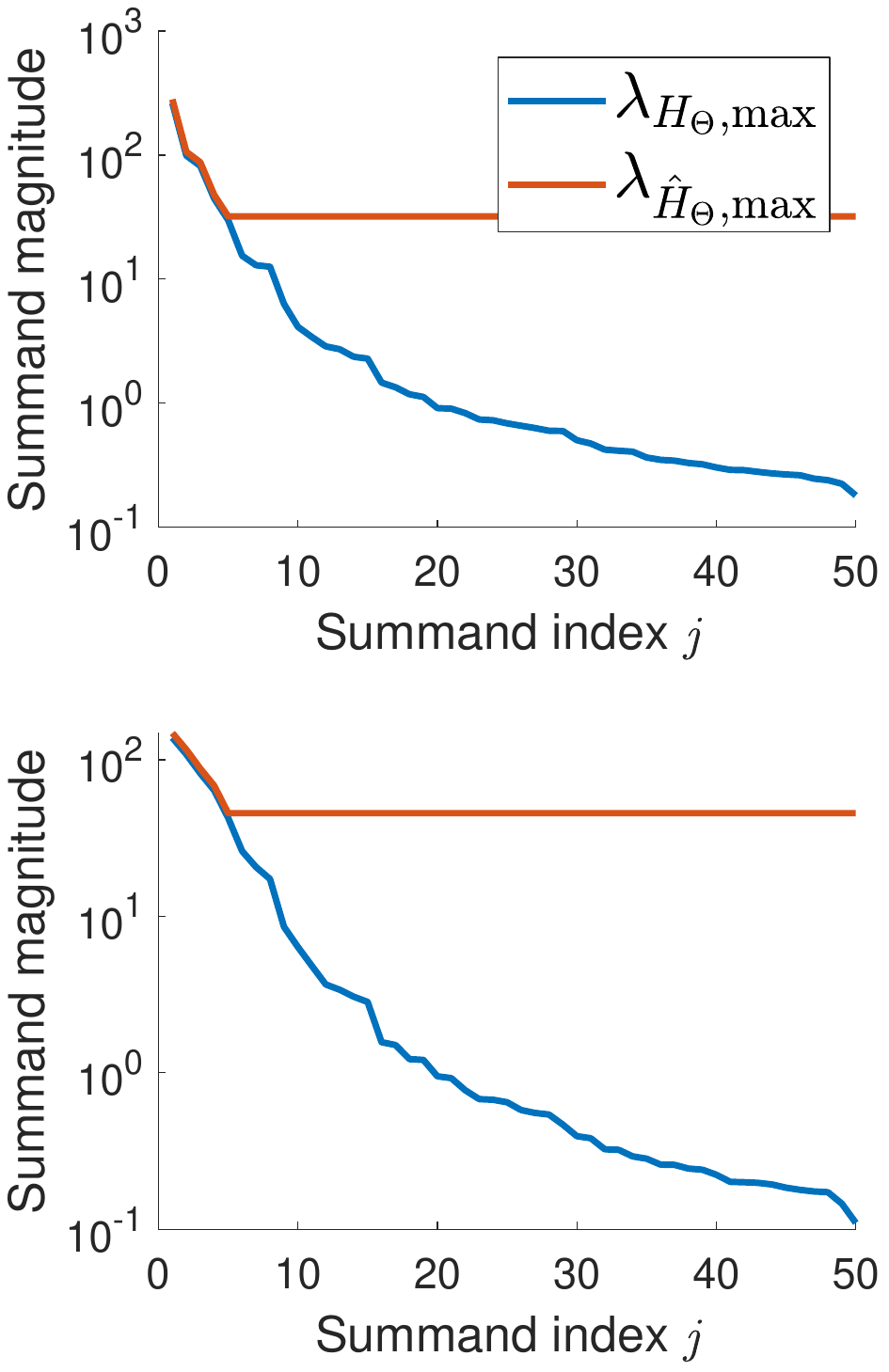}
        \caption{}
    \end{subfigure}
    &
    \begin{subfigure}{0.33\linewidth}
        \includegraphics[width=\linewidth]{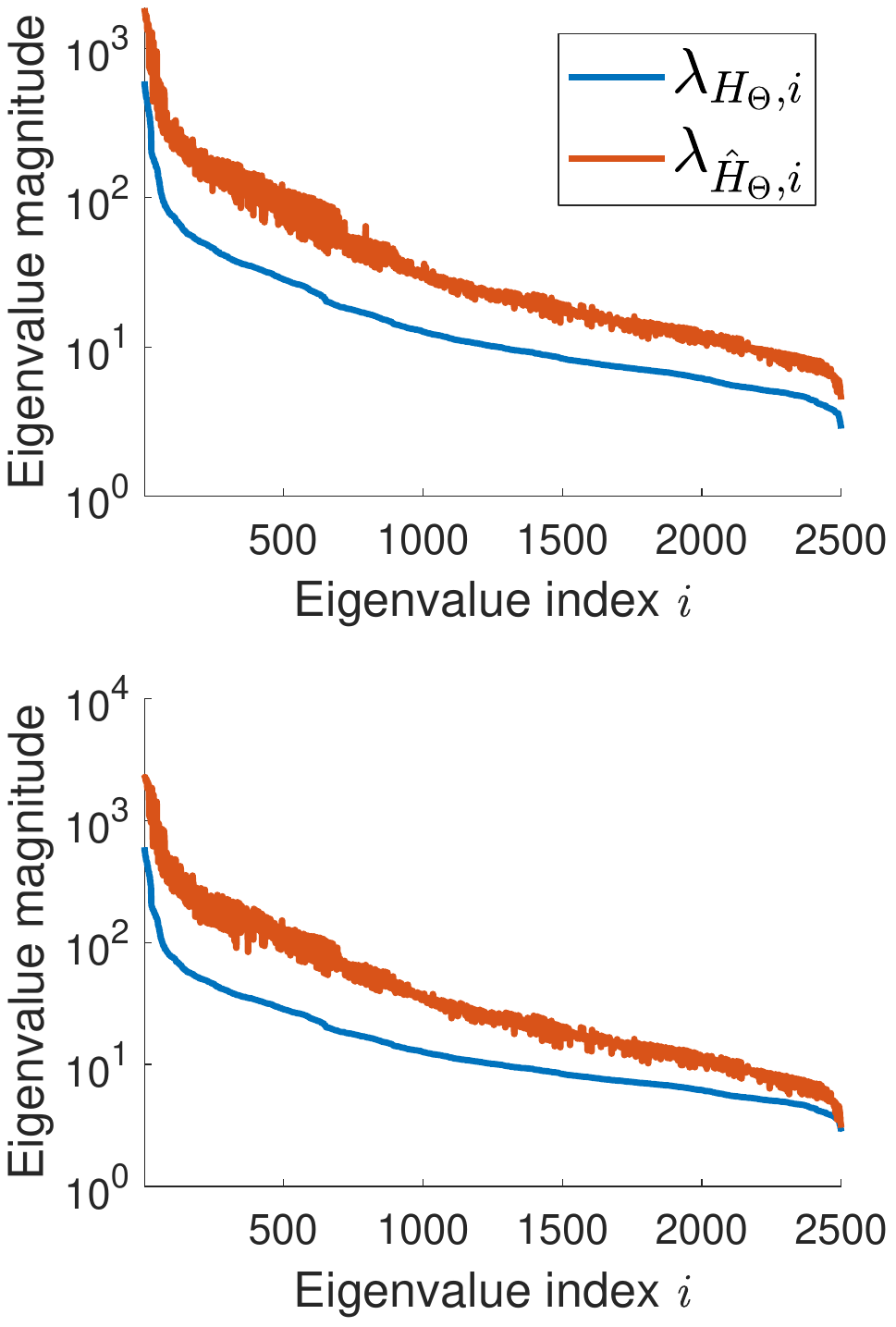}
        \caption{}
    \end{subfigure} 
    \\
    \begin{subfigure}{0.33\linewidth}
        \includegraphics[width=\linewidth]{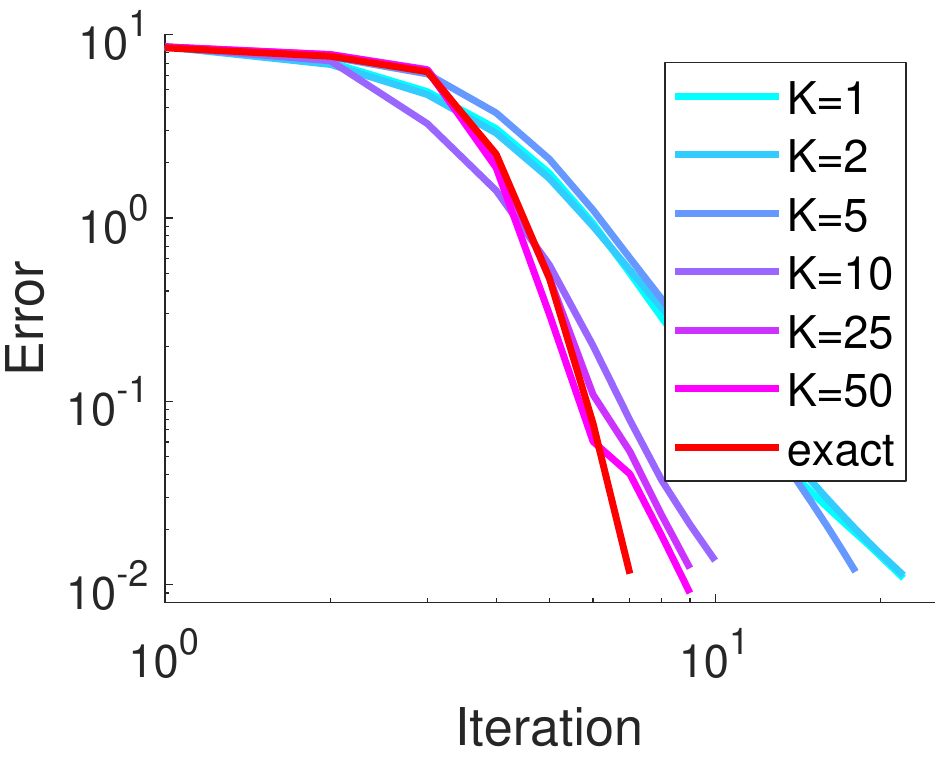}
        \caption{}
    \end{subfigure} 
    &
   \begin{subfigure}{0.33\linewidth}
        \includegraphics[width=\linewidth]{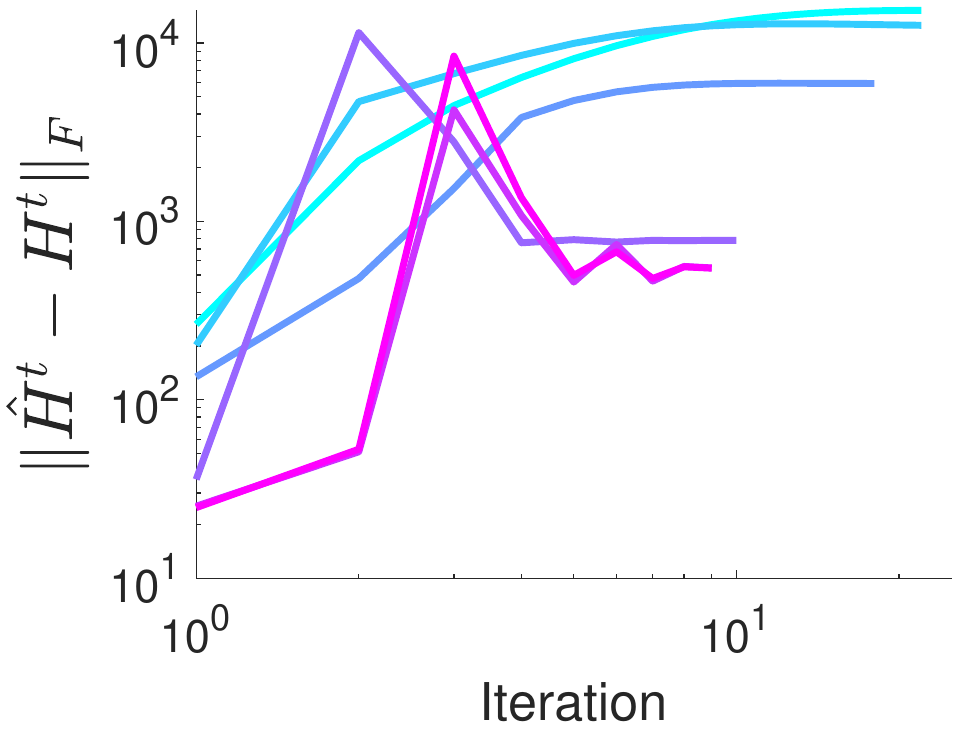}
        \caption{}
    \end{subfigure}
    &
\end{tabular}
\caption{
    Effects of Hessian approximation on the convergence behavior of EiGLasso
    for data simulated from a graph with clusters. 
    (a) The summands in $\lambda_{\bH_{\bThe},\min}=\sum_{j=1}^q (\lambda_{\bThe, p} 
    + \lambda_{\bPsi,j})^{-2}$ and 
    $\lambda_{\hat{\bH}_{\bThe},\min}=\sum_{j=1}^K (\lambda_{\bThe, p} + \lambda_{\bPsi,j})^{-2} 
    + (q-K)(\lambda_{\bThe, p} + \lambda_{\bPsi,K})^{-2}$ for EiGLasso with $K=5$. 
    (b) The summands in $\lambda_{\bH_{\bThe},\max}$ and $\lambda_{\hat{\bH}_{\bThe},\max}$, 
    as in Panel (a).
    (c) The eigenvalues of $\bH_{\bThe}$ and $\hat{\bH}_{\bThe}$, sorted 
    with the eigenvalues of $\bH_{\bThe}$.
    In Panels (a)-(c), the results are shown 
    at the 5th iteration (top) and at the last iteration (bottom).
    (d) Error measured as $\left\|\protect\begin{bmatrix}
    \tvec(\bThe^{t}-\bThe^*)\\
    \tvec(\bPsi^{t}-\bPsi^*)
    \protect\end{bmatrix}\right\|_2$ over iterations.
    (e) $\|\hat{\bH}-\bH\|_F$ over iterations.
    A dataset simulated from a random graph with $p, q=50$ was used.
}
    \label{fig:50bdemp}
\end{figure*}

Next, we compared the computation time of TeraLasso and EiGLasso on data simulated from larger graphs.
For datasets with size $p,q=100, 200, 500, 1000, 2000$, and $5000$, we ran TeraLasso and EiGLasso with
varying $K$ and recorded the time taken by each method if it converged within 24 hours
(Table \ref{tab:simresult}).
Across the different graph types and sizes, EiGLasso with different $K$'s almost always converged 
faster than TeraLasso. In particular, EiGLasso with $K=1$ and $2$ was 
consistently two to three orders-of-magnitude faster than
TeraLasso. For example, EiGLasso with $K=1$ took 3 hours for graphs with $p,q=5000$, 
while TeraLasso did not converge on the smaller graphs with $p,q=500$ in 24 hours.

We empirically evaluated the effects of Hessian approximation on the convergence behavior of EiGLasso,
using data simulated from a random graph of size $p,q=50$. 
In $\hat{\bH}_{\bThe}$ and $\hat{\bH}_{\bPsi}$ of the approximate Hessian $\hat{\bH}$, 
the components of the eigenvalues 
of $\bH_{\bThe}$ and $\bH_{\bPsi}$ are modified
from the summands in Eq. \eqref{eq:Hcomp} to
the summands in Eq. \eqref{eqn:approxHL}. 
For EiGLasso with $K=5$,
these summands in the eigenvalues before and after the modification
are shown for the minimum and maximum eigenvalues of $\bH_{\bThe}$
at the $5$th and last iterations
in Figures \ref{fig:50rndemp}(a) and \ref{fig:50rndemp}(b).
The full spectrums of the same $\bH_{\bThe}$ and $\hat{\bH}_{\bThe}$
are shown in Figure \ref{fig:50rndemp}(c).
Next, we examined the effects of such Hessian approximation on the convergence 
of EiGLasso (Figures \ref{fig:50rndemp}(d) and \ref{fig:50rndemp}(e)).
While EiGLasso with the exact Hessian converges quadratically,
with the approximate Hessian, 
as the degree of Hessian approximation increases with smaller $K$, 
the convergence of EiGLasso slows down. However,
even with $K=1$, the convergence
is not significantly slower than EiGLasso with the exact Hessian. 
Even if $\|\hat{\bH}-\bH\|_F$ does not approach zero over iterations 
(Figure \ref{fig:50rndemp}(e)), EiGLasso with the approximate Hessian
quickly reaches the same level of optimality as EiGLasso with the exact Hessian. 
We obtained similar results from a graph with clusters 
(Figure \ref{fig:50bdemp}). 


\subsection{Mouse Gene Expression Data from RNA-seq}

We compared EiGLasso and TeraLasso using the gene expression levels obtained from RNA sequencing of brain tissues 
for multiple related mice  from the same pedigree \citep{Gonzales2018}. 
The brain tissue samples were collected from generations 50-56 of the LG/J 
$\times$ SM/J advanced intercross line of mice. RNA was extracted 
from three mouse brain tissue types: prefrontal cortex (PFC; 185 mice), striatum (STR; 169 mice), and hippocampus (HIP; 208 mice). 
Several genes from these tissues have been found relevant to psychiatric and  metabolic disorders \citep{Gonzales2018}.  

We analyzed this data with EiGLasso and TeraLasso to estimate $\bThe$
for the relationship among the $p$ mice and $\bPsi$ for the network over $q$ genes.
We selected $10,000$ genes from each tissue type by excluding the genes with low variance in expression 
across the mice. 
As TeraLasso did not converge within 24 hours on this dataset with $q=10,000$, 
to compare the computation time of EiGLasso and TeraLasso,
we formed three smaller datasets for each tissue type 
by performing hierarchical clustering and selecting clusters of genes.
We selected the regularization parameters from the range $[0.1, 1.0]$ 
using Bayesian Information Criterion (BIC). 
EiGLasso was run with $K=1$ with the convergence criterion $\epsilon = 10^{-3}$. 
TeraLasso was run up to 24 hours, until it reached similar objective values to those from EiGLasso, 
which corresponded to $\epsilon = 10^{-6}$ or $10^{-7}$. 

\begin{table}[t!]
\centering
\begin{tabular}{|c|r|r|r|}
    \hline
    Tissue Type & Genes ($q$) & EiGLasso (sec) & TeraLasso (sec) \\
    \hline
    \multicolumn{1}{|c|}{\multirow{4}{*}{\shortstack{PFC\\($p=185$)}}} & 436 & 2 & 887\\
    & 877 & 13 & 18194\\
    & 5108 & 6141 &\\
    & 10000 & 25360 &\\
    \hline
    \multicolumn{1}{|c|}{\multirow{4}{*}{\shortstack{STR\\($p=169$)}}} & 447 & 9 & 6917\\
    & 977 & 77 & 39645\\
    & 5007 & 32977 &\\
    & 10000 & 65797 &\\
    \hline
    \multicolumn{1}{|c|}{\multirow{4}{*}{\shortstack{HIP\\($p=208$)}}} & 508 & 90 & 54454\\
    & 879 & 708 & \\
    & 5529 & 3106 &\\
    & 10000 & 20104 &\\
    \hline
\end{tabular}
\caption{Comparison of the computation time of EiGLasso with $K=1$ and TeraLasso on mouse gene-expression data.}
\label{tab:realruntime}
\end{table}

\begin{figure*}[t!]
\centering
\setlength\tabcolsep{0.1pt}
\begin{tabular}{ccc}
    \begin{subfigure}{0.33\linewidth}
        \includegraphics[width=\linewidth]{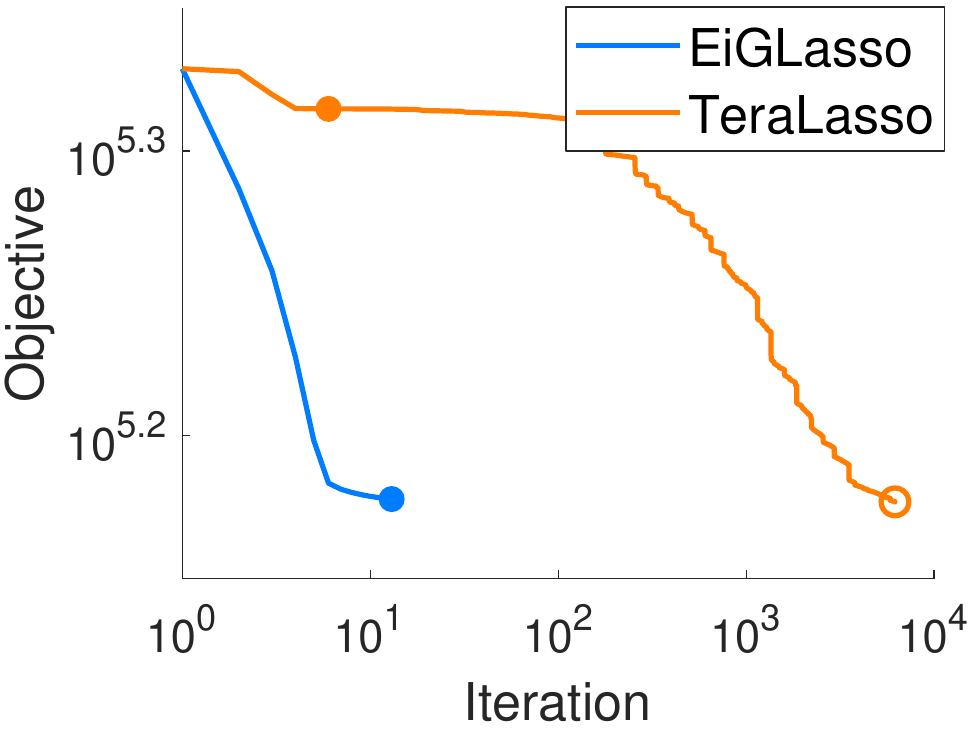}
    \end{subfigure} &
    \begin{subfigure}{0.33\linewidth}
        \includegraphics[width=\linewidth]{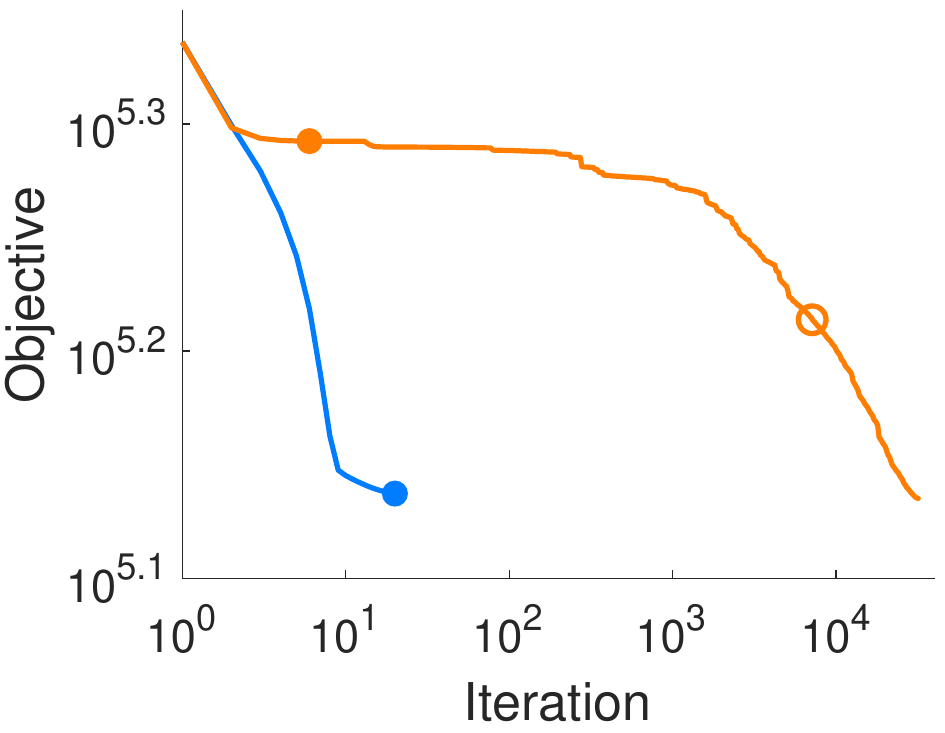}
    \end{subfigure} &
    \begin{subfigure}{0.33\linewidth}
        \includegraphics[width=\linewidth]{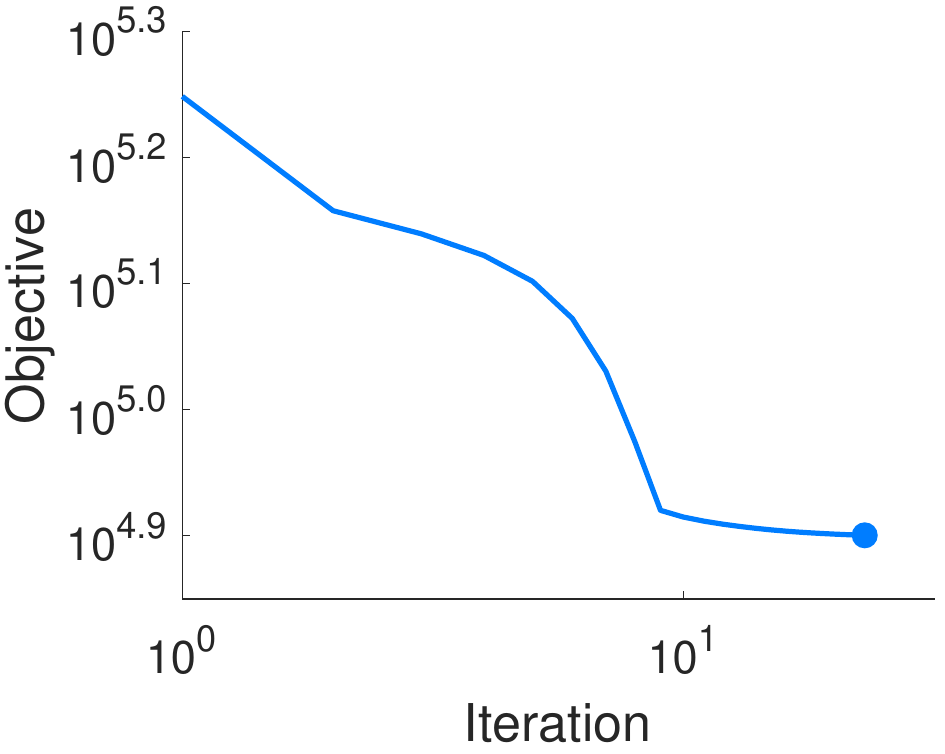}
    \end{subfigure}
    \\
    \begin{subfigure}{0.33\linewidth}
        \includegraphics[width=\linewidth]{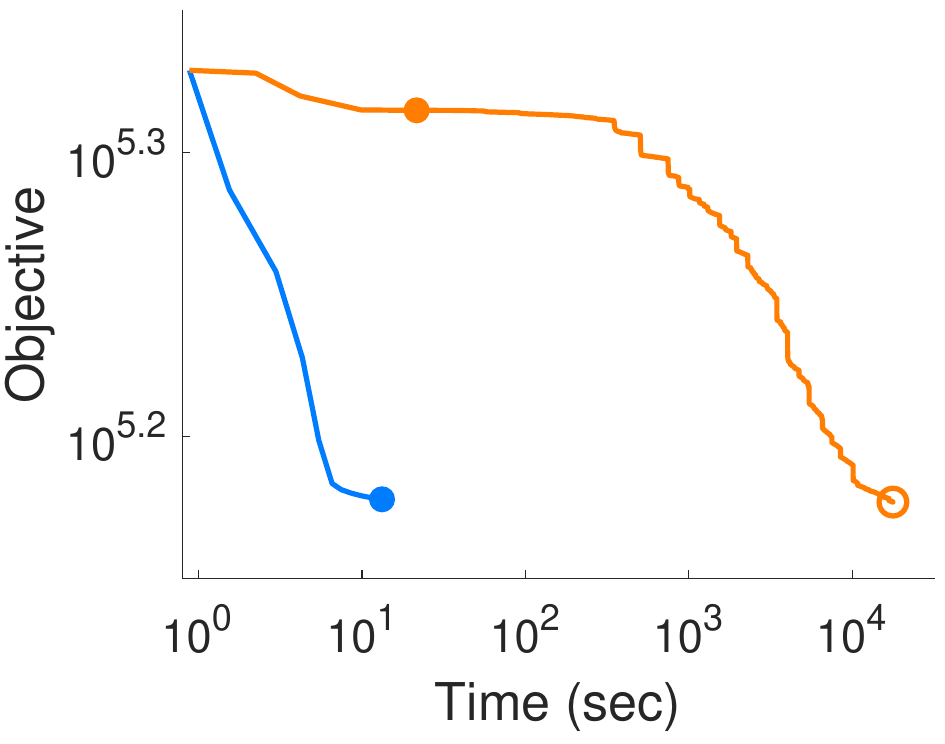}
        \caption{}
    \end{subfigure} &
    \begin{subfigure}{0.33\linewidth}
        \includegraphics[width=\linewidth]{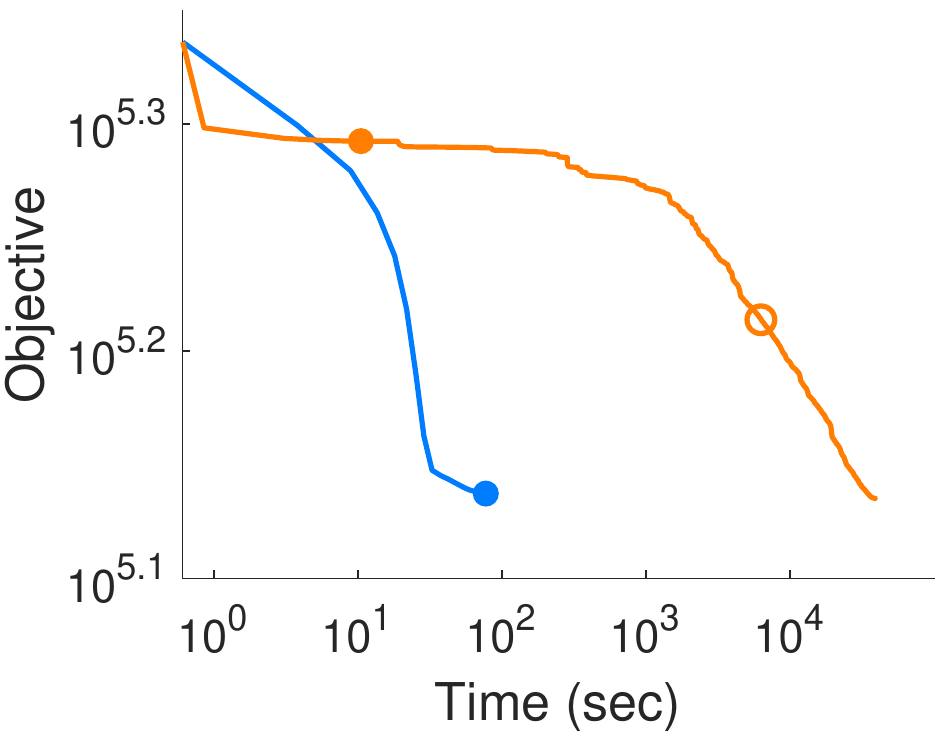}
        \caption{}
    \end{subfigure} &
    \begin{subfigure}{0.33\linewidth}
        \includegraphics[width=\linewidth]{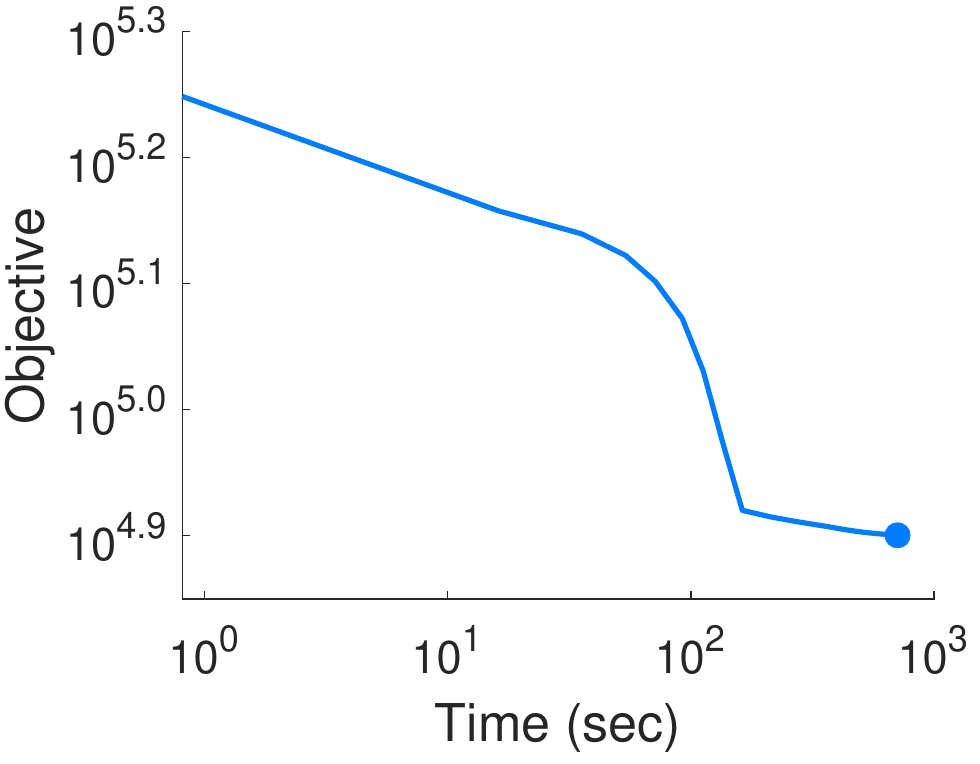}
        \caption{}
    \end{subfigure}
\end{tabular}
\caption{Comparison of the convergence of EiGLasso and TeraLasso on mouse gene-expression data. 
(a) PFC with $q=877$ genes. 
(b) STR with $q=977$ genes. 
(c) HIP with $q=879$ genes.
The objective values over iterations (top) and over time (bottom) are shown for 
each of the three brain tissue types. 
The `$\bullet$' and `$\bigcirc$' mark the points that satisfy the convergence criteria $\epsilon=10^{-3}$ and $10^{-7}$, respectively. 
EiGLasso with $K=1$ was used. 
}
    \label{fig:realconvg}
\end{figure*}

\begin{figure*}[t!]
\centering
\setlength\tabcolsep{0.2pt}
\begin{tabular}{ccc}
    \begin{subfigure}{0.33\linewidth}
        \includegraphics[width=\linewidth]{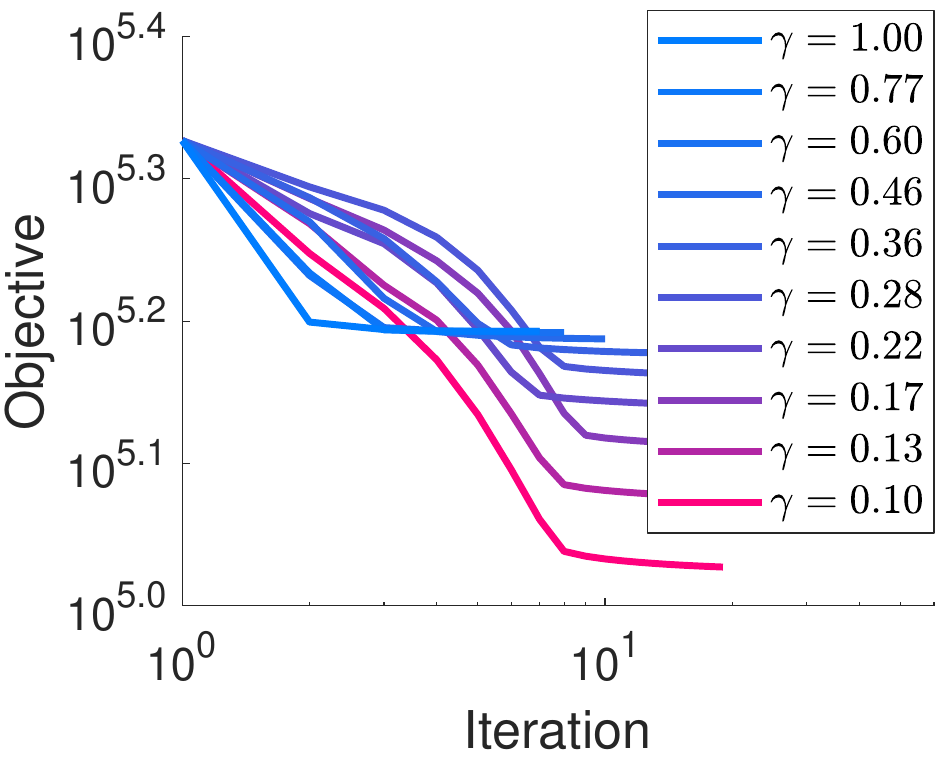}
        \caption*{(a)}
    \end{subfigure} &
    \begin{subfigure}{0.33\linewidth}
        \includegraphics[width=\linewidth]{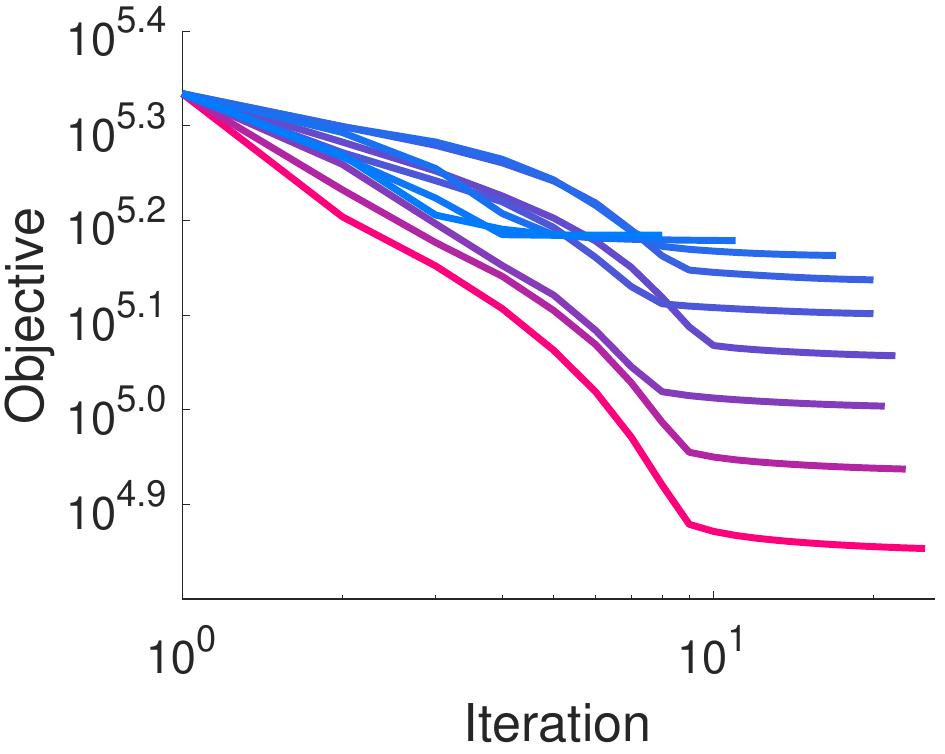}
        \caption*{(b)}
    \end{subfigure} &
    \begin{subfigure}{0.33\linewidth}
        \includegraphics[width=\linewidth]{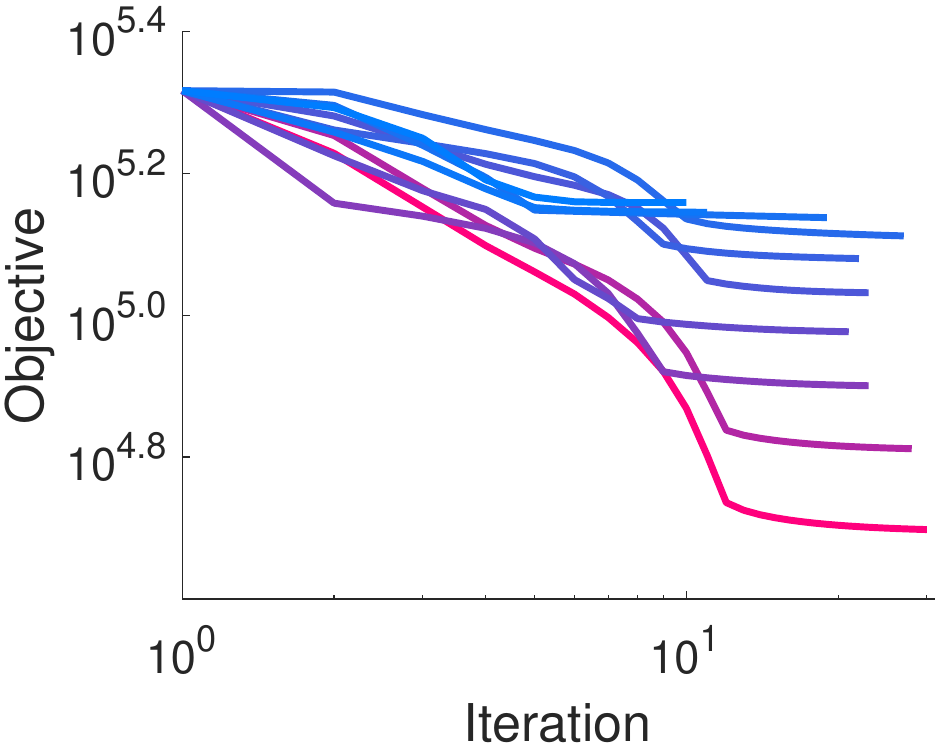}
        \caption*{(c)}
    \end{subfigure}
\end{tabular}
\caption{Convergence of EiGLasso on mouse gene-expression data for different regularization parameters. 
For ten different values of the regularization parameter $\gamma=\gamma_{\bThe}=\gamma_{\bPsi}$, 
the objective values over iterations are shown for  (a) PFC, (b) STR, and (c) HIP. 
EiGLasso with $K=1$ was run on the data for the same set of genes used in Figure \ref{fig:realconvg}.
}
    \label{fig:real1000cv}
\end{figure*}

Across all datasets with different tissue types and with different sizes,
EiGLasso was significantly faster than TeraLasso (Table \ref{tab:realruntime}).
On smaller datasets  with less than 1,000 genes,
EiGLasso was two to three orders-of-magnitude faster than TeraLasso. 
On HIP tissue with 879 genes, within 24 hours, TeraLasso was not able to converge to the similar objective value 
that EiGLasso reached, while EiGLasso converged in about 12 minutes. 
On the larger datasets with more than 5,000 genes, TeraLasso was not able to reach
the same objective as EiGLasso with $K=1$ within 24 hours.
On the full dataset with 10,000 genes, EiGLasso took 6 and 7 hours for the HIP and PFC tissue types
and 18 hours for the STR tissue type.

On the second largest dataset from each tissue type in Table \ref{tab:realruntime},
we examined the objective values over iterations and over time for EiGLasso and TeraLasso 
(Figure \ref{fig:realconvg}).
In all three datasets, EiGLasso converged two orders-of-magnitude faster than TeraLasso 
with much fewer iterations. 
To reach the similar objective value as EiGLasso, TeraLasso needed stricter convergence criterion 
$\epsilon=10^{-6}$ or $10^{-7}$. 
The number of iterations required by EiGLasso was not affected significantly
by different regularization  parameters $\gamma$ (Figure \ref{fig:real1000cv}).

\subsection{Financial data}

We applied EiGLasso and TeraLasso to the historical daily stock price data of 
the S\&P 500 constituents
to model the relationship among companies and dependencies across time points.
We obtained the daily stock closing prices for 306 companies 
that remained in the index from 2/16/2010 to 2/13/2020 for 2,523 days.
We excluded the data for six days when the stocks for only a small subset of the constituents were traded. 
We normalized the data by computing the proportion of price change on day $t$
from the previous day as $\frac{\textrm{price}^t - \textrm{price}^{t-1}}{\textrm{price}^{t-1}}$. 
We ran EiGLasso with $K=1$ and $\epsilon=10^{-3}$ and ran TeraLasso until
it reached the similar objective as EiGLasso.
We selected the optimal regularization parameter from
10 different values within the range of $[0.01, 1.0]$ using BIC.

\begin{figure}[t!]
\centering
\begin{tabular}{cc}
\multirow{2}{*}[2.5em]{
        \begin{subfigure}{.45\linewidth}
            \includegraphics[width=.888\linewidth]{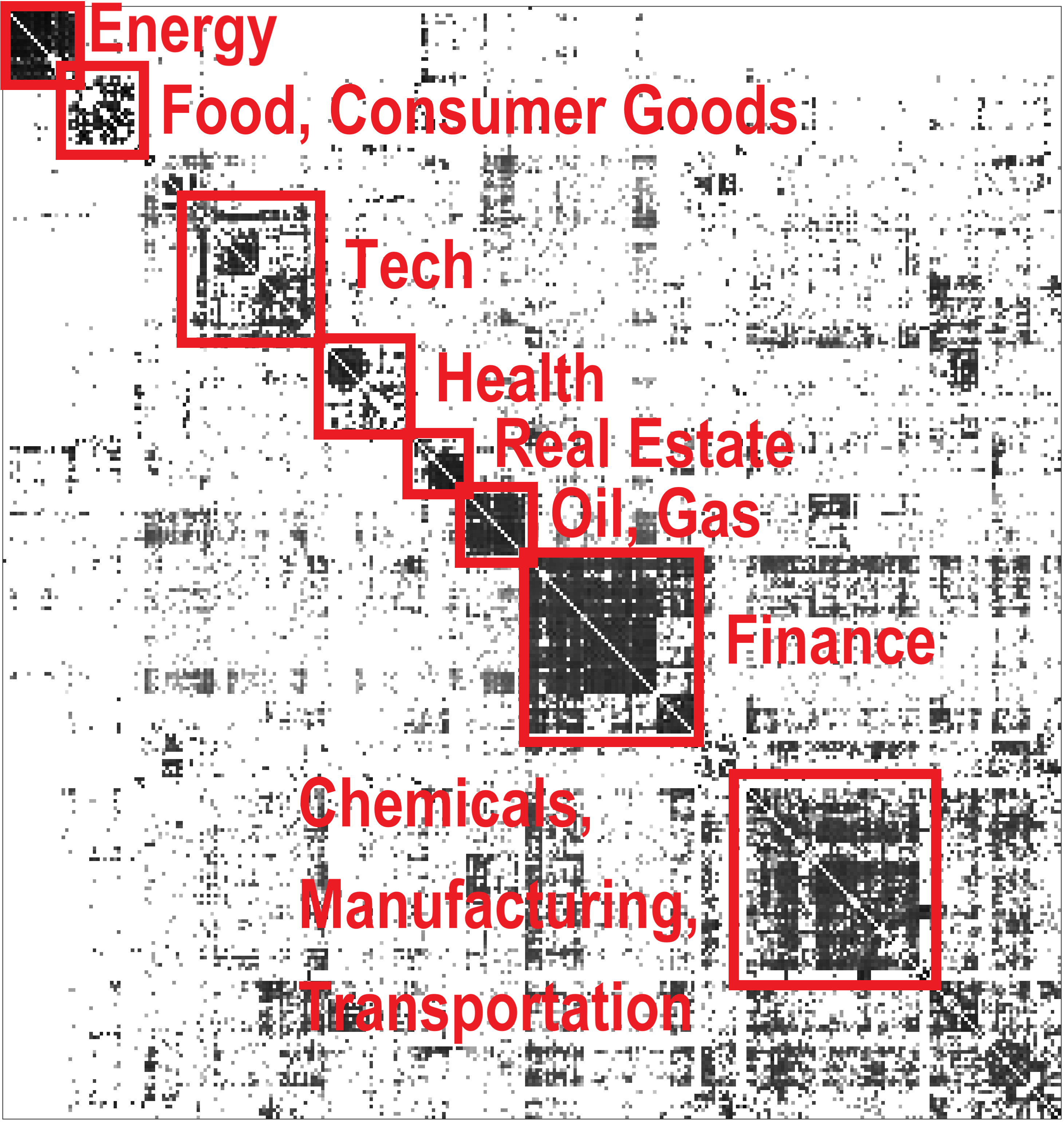}
        \end{subfigure}
        } &
        \begin{subfigure}{.33\linewidth}
            \includegraphics[width=\linewidth]{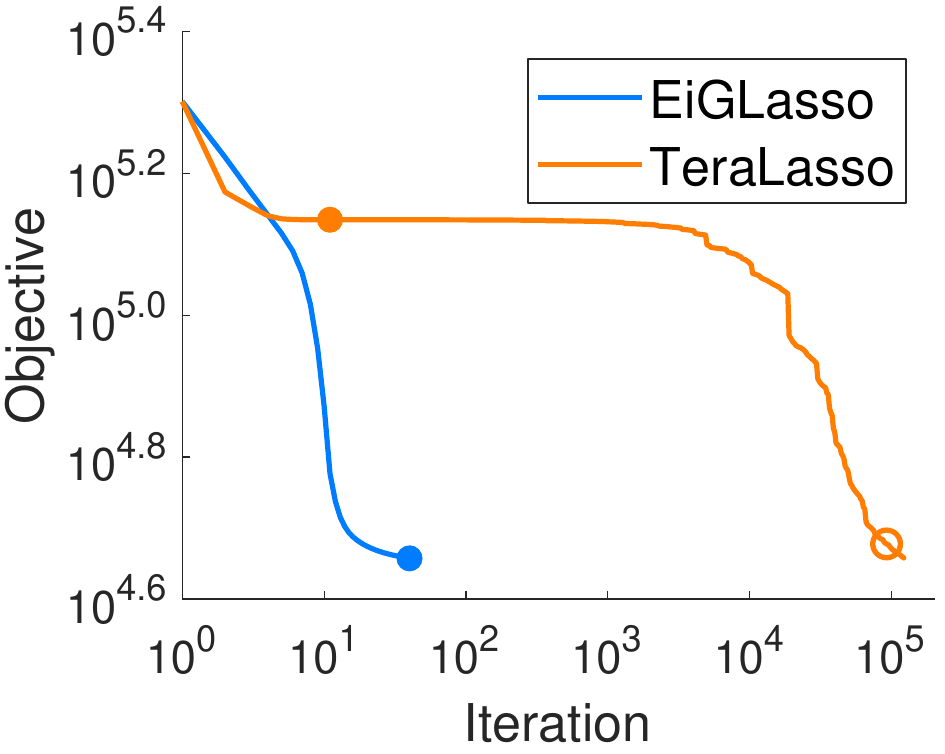}
        \end{subfigure}\\
        &
        \begin{subfigure}{.33\linewidth}
            \includegraphics[width=\linewidth]{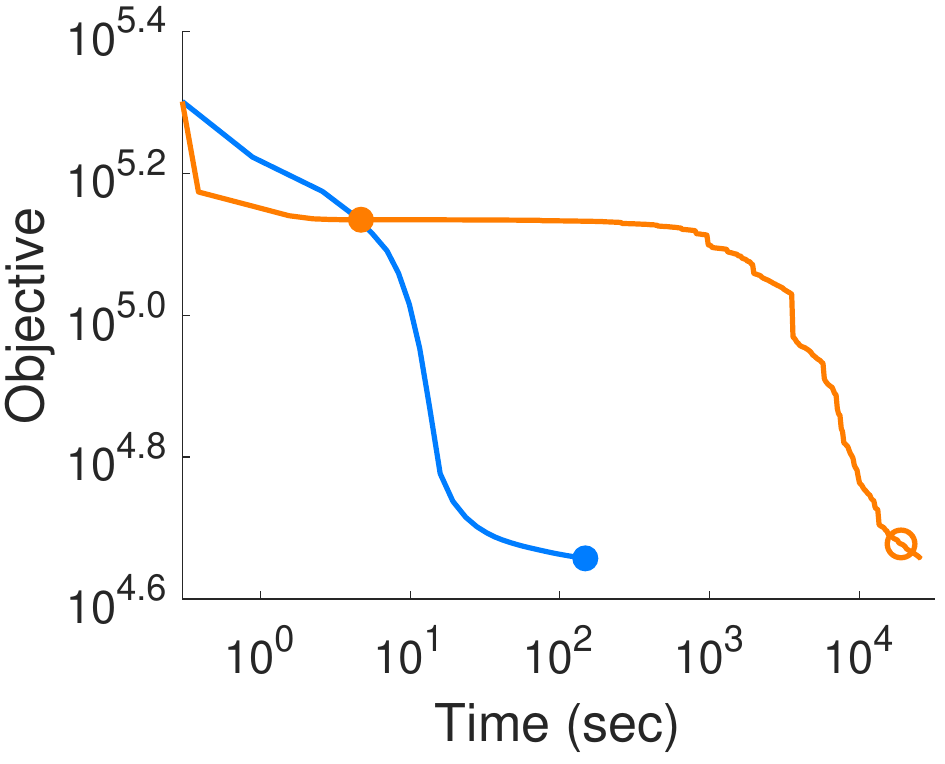}
        \end{subfigure}
\end{tabular}
\caption*{\hspace{3em}(a)\hspace{17em}(b)}
\caption{
Results on S\&P 500 data. (a) Graph over S\&P 500 companies estimated by EiGLasso from the stock closing price data.
Data from 306 companies over 2,516 days were used. 
The off-diagonals of the estimated graph are shown. (b) Comparison of the convergence of EiGLasso and TeraLasso on the data from 306 companies over 500 days. The objective values over iterations (top) and time (bottom) are shown. The `$\bullet$' and `$\bigcirc$' mark the points that satisfy 
the convergence criteria $\epsilon=10^{-3}$ and $10^{-7}$, respectively. 
EiGLasso with $K=1$ was used. 
}
    \label{fig:realrecover}
\end{figure}

EiGLasso with $K=1$ took 46 hours, whereas TeraLasso did not 
converge within 48 hours.
The graph over 306 companies estimated by EiGLasso
match the known sectors (Figure \ref{fig:realrecover}(a)).
We compared the computation time of EiGLasso with $K=1$ and TeraLasso
on a small subset over 500 days for the same 306 companies.
On this smaller dataset, EiGLasso took 148 seconds to converge and
TeraLasso took 25,739 seconds, showing that EiGLasso was again two orders-of-magnitude faster than TeraLasso (Figure \ref{fig:realrecover}(b)).



\section{Conclusion} \label{sec:conclusion}

We introduced EiGLasso, an efficient algorithm for estimating an inverse covariance matrix as
the Kronecker sum of two matrices, each representing a graph over samples and a graph over features. 
Extending the second-order optimization method in QUIC for Gaussian graphical models~\citep{quic}, 
EiGLasso employed the eigendecomposition of the two matrices for individual graphs
during optimization, to avoid an expensive computation of large gradient and Hessian 
matrices with an inflated structure that arises from the Kronecker sum.
EiGLasso approximated the Hessian based on this eigenstructure of the parameters to further 
reduce the computation time. We showed the quadratic convergence of EiGLasso with the exact Hessian
and linear convergence of EiGLasso with the approximate Hessian. In our experiments,
EiGLasso achieved two to three orders-of-magnitude speed-up
over the state-of-the-art method, TeraLasso~\citep{TeraLasso}.

In addition, we introduced a new approach for identifying and estimating
the unidentifiable diagonal elements of the matrices for the feature and sample graphs.
We showed that given the ratio of the traces of the two matrices, the diagonal parameters
can be identified uniquely within the equivalence class of the unidentifiable parameters. 
EiGLasso performed the optimization without identifying 
these parameters, since all quantities involved in the optimization are
invariant within the equivalence class regardless of the trace ratio. 
We further showed that it is sufficient to identify the parameters 
with the fixed trace ratio once,  after the optimization was complete. 
Thus, EiGLasso used a significantly simpler strategy for estimating the
unidentifiable parameters than the previous methods.

There are several possible future directions. One natural extension of EiGLasso
is to adopt Big\&QUIC~\citep{bigquic}, an extension of QUIC to remove the memory 
requirement, within EiGLasso.
Because unlike QUIC, EiGLasso uses the eigendecomposition of the parameter
matrices, which in itself has the space complexity cubic in the graph sizes, 
further research is needed to incorporate the eigendecomposition within
the block-wise update used in Big\&QUIC.  
Another future direction is to develop a Hessian approximation technique with faster
theoretical convergence. 
While we showed that EiGLasso with our approximate Hessian 
converges linearly to the optimum, it may be possible to achieve superlinear convergence
with an approximation Hessian that approaches the exact Hessian over iterations.

\section*{Acknowledgements}
This work was supported by NIH 1R21HG011116 and 1R21HG010948. This work used the Extreme Science and Engineering Discovery Environment (XSEDE), which is supported by NSF ACI-1548562 and ACI-1445606.

\vskip 0.2in
\bibliography{refs}

\section*{Appendix A. Hessian Computation}

We show that the collapse of $\bW \otimes \bW$ in Eq. \eqref{eq:HcollapseWW} to obtain the Hessian
can be viewed as the same type of deflation for the gradient in Figure \ref{fig:collapseW}(b) applied twice 
to $\bW \otimes \bW$.

Let $\bm{A}$ denote a matrix of size $m_1\times m_2$. 
Then, the elements of the Kronecker product 
$\bm{A}\otimes\bm{A}$ are given as
$[\bm{A}\otimes\bm{A}]_{(i-1)m_1 + k,\, (j-1)m_2 + l}=[\bm{A}]_{ij}[\bm{A}]_{kl}$
for $i,k=1, \ldots, m_1,$ and $j, l = 1, \ldots, m_2$.
As a representation of $\bm{A}\otimes\bm{A}$ with different row/column indices,
we consider $\tvec(\bm{A})\tvec(\bm{A})^T$
whose elements are given as
$[\tvec(\bm{A})\tvec(\bm{A})^T]_{(i-1)m_1 + j,\, (k-1)m_1 + l}=[\bm{A}]_{ij}[\bm{A}]_{kl}$. 
We define an operator 
\begin{equation*}
    \mathcal{P}_{m_1,m_2}(\bm{A}\otimes\bm{A}) = \tvec(\bm{A})\tvec(\bm{A})^T, 
\end{equation*}
which maps the elements of $\bm{A}\otimes\bm{A}$ 
to those of $\tvec(\bm{A})\tvec(\bm{A})^T$.
The same operator can be applied to $\bm{D} = \sum_{i=1}^{N}\bm{A}_i\otimes\bm{A}_i$,
where $\bm{A}_i\in\mathbb{R}^{m_1\times m_2}$,
\begin{equation}
    \mathcal{P}_{m_1,m_2}(\bm{D}) =
    \mathcal{P}_{m_1,m_2}(\sum_i\bm{A}_i\otimes\bm{A}_i) = \sum_i \tvec(\bm{A}_i)\tvec(\bm{A}_i)^T. \label{eq:op}
\end{equation}
Collapsing $\bW \otimes \bW$ into $\bH_{\bThe}$ amounts to applying the same type of collapse as in
the gradient (Figure \ref{fig:collapseW}(b)) twice 
on $\mathcal{P}_{p^2q^2,p^2q^2}(\bW\otimes\bW)$  to obtain $\mathcal{P}_{p,p}(\bH_{\bThe})$, 
$\mathcal{P}_{q,q}(\bH_{\bPsi})$, and $\mathcal{P}_{p,q}(\bH_{\bThe\bPsi})$. 
To see this, we re-write $\bH_{\bThe}$ as
\begin{align*}
    \bH_{\bThe}
    &= \bP_{\bThe}^T(\bW\otimes\bW)\bP_{\bThe}\\
    &= \sum_{i=1}^q\sum_{j=1}^q ((\bI_p\otimes\bm{e}_{q,i})\otimes(\bI_p\otimes\bm{e}_{q,i}))^T(\bW\otimes\bW)
    ((\bI_p\otimes\bm{e}_{q,j})\otimes(\bI_p\otimes\bm{e}_{q,j}))\\
    &= \sum_{i=1}^q\sum_{j=1}^q \Big[(\bI_p\otimes\bm{e}_{q,i})^T\bW(\bI_p\otimes\bm{e}_{q,j})\Big]\otimes\Big[(\bI_p\otimes\bm{e}_{q,i})^T\bW(\bI_p\otimes\bm{e}_{q,j})\Big].
\end{align*}
Then, we apply the operator in Eq. \eqref{eq:op} to $\bH_{\bThe}$:
\begin{align}
    \mathcal{P}_{p,p}(\bH_{\bThe}) 
    &= \sum_{i=1}^q\sum_{j=1}^q \tvec\left( (\bI_p\otimes\bm{e}_{q,i})^T\bW(\bI_p\otimes\bm{e}_{q,j}) \right) 
    \tvec\left( (\bI_p\otimes\bm{e}_{q,i})^T\bW(\bI_p\otimes\bm{e}_{q,j}) \right)^T \nonumber \\
    &= \sum_{i=1}^q\sum_{j=1}^q (\bI_p\otimes\bm{e}_{q,j}\otimes\bI_p\otimes\bm{e}_{q,i})^T
    \tvec(\bW)\tvec(\bW)^T(\bI_p\otimes\bm{e}_{q,j}\otimes\bI_p\otimes\bm{e}_{q,i}) \nonumber \\
    &= \sum_{i=1}^q\sum_{j=1}^q \Bigg(\Big[(\bI_p\otimes\bm{e}_{q,j}\otimes\bI_p\otimes\bI_q)
    (\bI_p\otimes1\otimes\bI_p\otimes\bm{e}_{q,i})\Big]^T \nonumber \\
    &\qquad\qquad\qquad\qquad\qquad\mathcal{P}_{pq,pq}(\bW\otimes\bW)
    \Big[(\bI_p\otimes\bm{e}_{q,j}\otimes\bI_p\otimes\bI_q)(\bI_p\otimes1\otimes\bI_p\otimes\bm{e}_{q,i})\Big]\Bigg) \nonumber \\
    &= \sum_{i=1}^q (\bI_p\otimes(\bI_{p}\otimes\bm{e}_{q,i}))^T 
    \bigg[ \sum_{j=1}^q ((\bI_p\otimes\bm{e}_{q,j})\otimes\bI_{pq})^T \nonumber \\
    &\qquad\qquad\qquad\qquad\qquad\mathcal{P}_{pq,pq}(\bW\otimes\bW)((\bI_p\otimes\bm{e}_{q,j})\otimes\bI_{pq}) \bigg]
    (\bI_p\otimes(\bI_{p}\otimes\bm{e}_{q,i})).  \nonumber \\
    &= \sum_{i=1}^q (\bI_p\otimes(\bI_{p}\otimes\bm{e}_{q,i}))^T 
    \bm{M}_{\bThe} (\bI_p\otimes(\bI_{p}\otimes\bm{e}_{q,i})),  \label{eq:collapseH1}
\end{align}
where $\bm{M}_{\bThe}= \sum_{j=1}^q ((\bI_p\otimes\bm{e}_{q,j})\otimes\bI_{pq})^T 
    \mathcal{P}_{pq,pq}(\bW\otimes\bW)((\bI_p\otimes\bm{e}_{q,j})\otimes\bI_{pq})$.
Similarly, 
\begin{align}
    \mathcal{P}_{q,q}(\bH_{\bPsi}) &= \sum_{i=1}^p (\bI_{q}\otimes(\bm{e}_{p,i}\otimes\bI_{q}))^T  
     \bm{M}_{\bPsi} (\bI_{q}\otimes(\bm{e}_{p,i}\otimes\bI_{q})), \label{eq:collapseH2}  \\
    \mathcal{P}_{p,q}(\bH_{\bThe\bPsi}) 
    &= \sum_{i=1}^q (\bI_{p}\otimes(\bI_{p}\otimes\bm{e}_{q,i}))^T  \bm{M}_{\bPsi}
    (\bI_{q}\otimes(\bI_{p}\otimes\bm{e}_{q,i})), \label{eq:collapseH3}
\end{align}
where $\bm{M}_{\bPsi} = \sum_{j=1}^p ((\bm{e}_{p,j}\otimes\bI_{q})\otimes \bI_{pq})^T 
     \mathcal{P}_{pq,pq}(\bW\otimes\bW)((\bm{e}_{p,j}\otimes\bI_{q})\otimes\bI_{pq})$.
Eqs. \eqref{eq:collapseH1}-\eqref{eq:collapseH3} show the two-stage collapse for Hessian computation:
$\mathcal{P}_{pq,pq}(\bW\otimes\bW)$ (Figure \ref{fig:collapseH}(a))
is collapsed to $\bm{M}_{\bThe}$ and $\bm{M}_{\bPsi}$ (Figure \ref{fig:collapseH}(b)),
which are then collapsed to 
$\mathcal{P}_{p,p}(\bH_{\bThe})$, $\mathcal{P}_{p,q}(\bH_{\bThe\bPsi})$, 
and $\mathcal{P}_{q,q}(\bH_{\bPsi})$ (Figure \ref{fig:collapseH}(c)).
The first collapse of $\mathcal{P}_{pq,pq}(\bW\otimes\bW)$ to $\bm{M}_{\bThe}$ and $\bm{M}_{\bPsi}$
is equivalent to applying the same collapse for the gradient in Figure \ref{fig:collapseW}(b)  
at the level of cells in Figure \ref{fig:collapseH}(a)
each with size $pq\times pq$, instead of at the level of matrix elements.
The second collapse from $\bm{M}_{\bThe}$ and $\bm{M}_{\bPsi}$ to 
$\mathcal{P}_{p,p}(\bH_{\bThe})$, $\mathcal{P}_{p,q}(\bH_{\bThe\bPsi})$, 
and $\mathcal{P}_{q,q}(\bH_{\bPsi})$ 
is equivalent to applying the same collapse again
on each $pq\times pq$ cell of $\bm{M}_{\bThe}$ and $\bm{M}_{\bPsi}$ 
in Figure \ref{fig:collapseH}(b) at the level of matrix elements.
Notice that not all elements of $\mathcal{P}_{pq,pq}(\bW\otimes\bW)$ contribute to 
$\mathcal{P}_{p,p}(\bH_{\bThe})$, $\mathcal{P}_{p,q}(\bH_{\bThe\bPsi})$, 
and $\mathcal{P}_{q,q}(\bH_{\bPsi})$:
Only the elements of $\mathcal{P}_{pq,pq}(\bW\otimes\bW)$ corresponding to the non-zero
elements of the inflated Kronecker-sum mask $(\One_{p} \oplus \One_{q}) \otimes \One_{pq}$
and the elements of $\bm{M}_{\bThe}$ and $\bm{M}_{\bPsi}$, 
each masked with $\One_{p}\otimes(\One_{p} \oplus \One_{q})$ and $\One_{q}\otimes(\One_{p} \oplus \One_{q})$,
are used to compute 
$\mathcal{P}_{p,p}(\bH_{\bThe})$, $\mathcal{P}_{p,q}(\bH_{\bThe\bPsi})$,
and $\mathcal{P}_{q,q}(\bH_{\bPsi})$.

\begin{figure}[t!]
    \centering
    \includegraphics[width=\linewidth]{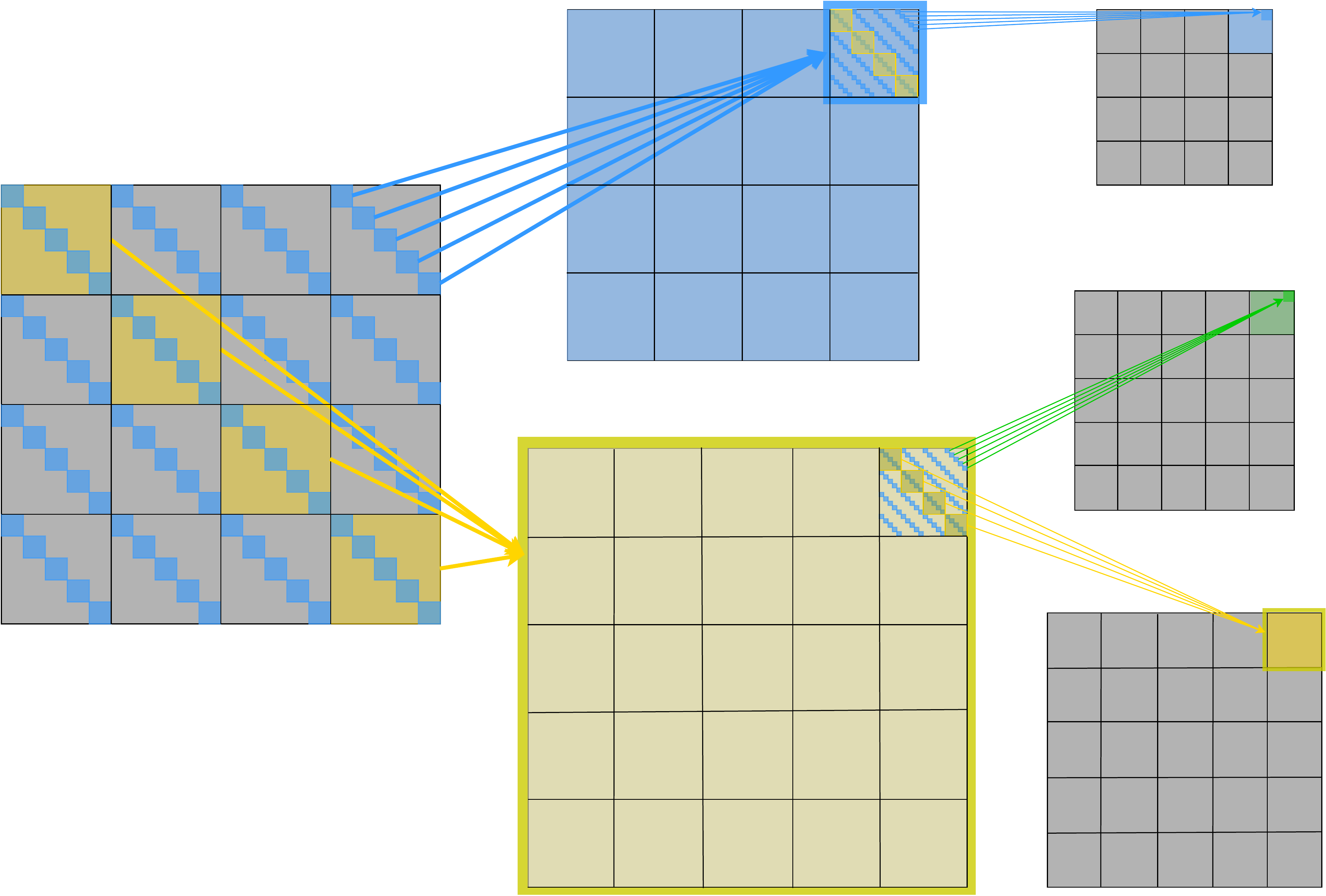}
    \caption*{$\;\;\qquad$(a)$\quad\qquad\qquad\qquad\qquad\qquad\qquad$(b)$\;\quad\qquad\qquad\qquad\qquad\qquad\qquad$(c)}
    \caption{Illustration of the Hessian computation in EiGLasso.
    A two-stage collapse of  $\mathcal{P}_{pq,pq}(\bW\otimes\bW)$ 
    into $\mathcal{P}_{p,p}(\bH_{\bThe})$, $\mathcal{P}_{p,q}(\bH_{\bThe\bPsi})$, and $\mathcal{P}_{q,q}(\bH_{\bPsi})$
    is shown for the same $\bThe$ and $\bPsi$ with $p=4$ and $q=5$ in Figure \ref{fig:collapseW}.
    (a) $\mathcal{P}_{pq,pq}(\bW\otimes\bW)$ of size $p^2q^2\times p^2q^2$
    with a $pq\times pq$ matrix in each cell.
    Only the cells with non-zero elements in the Kronecker-sum-like mask
    $(\One_{p}\oplus\One_{q})\otimes\One_{pq}=(\One_{p}\otimes \bI_{q}+\bI_{p}\otimes\One_{q})\otimes\One_{pq}$
    contribute to the Hessian. The two components of this mask, $(\One_{p}\otimes \bI_{q})\otimes\One_{pq}$ 
    and $(\bI_{p}\otimes\One_{q})\otimes\One_{pq}$, are shown with blue and yellow, respectively.
    (b) Two intermediate matrices after the first collapse, 
    each with size $p^2q\times p^2q$ (top) and size $pq^2\times pq^2$ (bottom), 
    with a $pq\times pq$ matrix in each block.
    Only the elements of these matrices that correspond to non-zero elements in the Kronecker-sum-like
    mask $\One_{p} \otimes (\One_{p}\oplus\One_{q})$ (top)  and
    $\One_{q} \otimes (\One_{p}\oplus\One_{q})$ (bottom) contribute to the Hessian.  
    This is equivalent to applying the mask $\One_{p}\oplus\One_{q}$ to each block. 
    (c) $\mathcal{P}_{p,p}(\bH_{\bThe})$ of size $p^2 \times p^2$ with a $p\times p$ matrix in each block (top),
    $\mathcal{P}_{p,q}(\bH_{\bThe\bPsi})$ of size $p^2 \times q^2$ with a $p\times q$ matrix in each block (middle),
    and $\mathcal{P}_{q,q}(\bH_{\bPsi})$ of size $q^2 \times q^2$, with a $q\times q$ matrix in each block (bottom).
    The two collapses from Panel (a) to Panel (b) and from Panel (b) to Panel (c) are shown with arrows
    that indicate the sum of the cells in the larger matrix into the cell in the smaller matrix.
    }
    \label{fig:collapseH}
\end{figure}

\section*{Appendix B. Computing the Newton Direction}
We provide the details on the coordinate descent updates for $D_{\bThe}$. Updating $D_{\bPsi}$ can be done in a similar manner.

\subsection*{B1. Exact Hessian}
Solving Eq. \eqref{eq:newton} for $[D_{\bThe}]_{ij}$ assuming that all the other elements
of $D_{\bThe}$ are fixed
amounts to solving the following optimization problem:
\begin{align*}
    \argmin_\mu\quad &\mu\Big(q[\bm{S}]_{ij} - [\bW_{\bThe}]_{ij} + \sum_{k=1}^q \bm{v}_{k,i}^T D_{\bThe} \bm{v}_{k,j} + \sum_{l=1}^p\sum_{k=1}^q \lambda_{\bW,lk}^2 [\bm{Q}_{\bThe}]_{il}[\bm{Q}_{\bThe}]_{jl}\bm{q}_{\bPsi, k}^TD_{\bPsi}\bm{q}_{\bPsi, k}\Big) \\
    &+ \frac{\mu^2}{2}\bigg(\sum_{k=1}^q [\bm{V}_{\bThe, k}]_{ij}^2 + [\bm{V}_{\bThe, k}]_{ii}[\bm{V}_{\bThe,k}]_{jj} \bigg) 
    + q\gamma_{\bThe}\left|[\bThe]_{ij} + [D_{\bThe}]_{ij} + \mu\right|,
\end{align*}
where $\bm{v}_{k,i}$ is the $i$th column of $\bm{V}_{\bThe,k} = \bm{Q}_{\bThe}\bXi_{\bThe,k}\bm{Q}_{\bThe}^T$.
This problem has a closed-form solution
\begin{equation*}
    \mu = -c+\mathcal{S}\left(c-\frac{b}{a},\frac{q\gamma_{\bThe}}{a}\right),
\end{equation*}
where
\begin{align*}
    a &= \sum_{k=1}^q [\bm{V}_{\bThe, k}]_{ij}^2 + [\bm{V}_{\bThe, k}]_{ii}[\bm{V}_{\bThe,k}]_{jj},\\
    b &= q[\bm{S}]_{ij} - [\bW_{\bThe}]_{ij} + \sum_{k=1}^q \bm{v}_{k,i}^T D_{\bThe} \bm{v}_{k,j} + \sum_{l=1}^p\sum_{k=1}^q \lambda_{\bW,lk}^2 [\bm{Q}_{\bThe}]_{il}[\bm{Q}_{\bThe}]_{jl}\bm{q}_{\bPsi, k}^TD_{\bPsi}\bm{q}_{\bPsi, k},\\
    c &= [\bThe]_{ij} + [D_{\bThe}]_{ij},
\end{align*}
and $\mathcal{S}(z,r)=\text{sign}(z)\max\{|z|-r,0\}$ is the soft-thresholding function.

\subsection*{B2. Approximate Hessian}
Similarly, solving Eq. \eqref{eq:newton} 
with $\hat{\bH}$
instead of $\bH$
for $[D_{\bThe}]_{ij}$
while fixing all the other elements of $D_{\bThe}$ 
amounts to solving the following optimization problem:
\begin{align*}
    \argmin_\mu\quad &\mu\Big(q[\bm{S}]_{ij} - [\bW_{\bThe}]_{ij} + \sum_{k=1}^K \bm{v}_{k,i}^T D_{\bThe} \bm{v}_{k,j} + (q-K)\bm{v}_{K,i}^T D_{\bThe} \bm{v}_{K,j} \Big) \\
    &+ \frac{\mu^2}{2}\bigg(\sum_{k=1}^K [\bm{V}_{\bThe, k}]_{ij}^2 + [\bm{V}_{\bThe, k}]_{ii}[\bm{V}_{\bThe,k}]_{jj} + (q-K)\left([\bm{V}_{\bThe, K}]_{ij}^2 + [\bm{V}_{\bThe, K}]_{ii}[\bm{V}_{\bThe, K}]_{jj}\right) \bigg) \\
    &+ q\gamma_{\bThe}\left|[\bThe]_{ij} + [D_{\bThe}]_{ij} + \mu\right|,
\end{align*}
with a closed-form solution
\begin{equation*}
    \mu = -c+\mathcal{S}\left(c-\frac{b}{a},\frac{q\gamma_{\bThe}}{a}\right),
\end{equation*}
where
\begin{align*}
    a &= \sum_{k=1}^K [\bm{V}_{\bThe, k}]_{ij}^2 + [\bm{V}_{\bThe, k}]_{ii}[\bm{V}_{\bThe,k}]_{jj} + (q-K)\left([\bm{V}_{\bThe, K}]_{ij}^2 + [\bm{V}_{\bThe, K}]_{ii}[\bm{V}_{\bThe, K}]_{jj}\right),\\
    b &= q[\bm{S}]_{ij} - [\bW_{\bThe}]_{ij} + \sum_{k=1}^K \bm{v}_{k,i}^T D_{\bThe} \bm{v}_{k,j} + (q-K)\bm{v}_{K,i}^T D_{\bThe} \bm{v}_{K,j},\\
    c &= [\bThe]_{ij} + [D_{\bThe}]_{ij}.
\end{align*}



\end{document}